\documentclass{article}
\usepackage[preprint]{log_2024}			

\usepackage{booktabs}						
\usepackage{multirow}						
\usepackage{amsfonts}						
\usepackage{graphicx}						
\usepackage{duckuments}						
\usepackage[hidelinks,backref=page]{hyperref}            

\usepackage[numbers,compress,sort]{natbib}	

\usepackage{amsmath}
\usepackage{amssymb}
\usepackage{mathtools}
\usepackage{amsthm}
\usepackage{bbm}
\usepackage{thmtools}
\usepackage{thm-restate}
\usepackage{caption}
\usepackage{subcaption}
\usepackage{wrapfig}
\usepackage{enumitem}
\usepackage{xfrac}

\usepackage{graphicx}
\usepackage{cleveref}
\usepackage{multicol}
\usepackage{multirow}
\usepackage{makecell}
\usepackage{xspace}
\usepackage{algorithm}
\usepackage{algpseudocode}

\usepackage{amsmath,amsfonts,bm,dsfont,diagbox,amsthm}

\newcommand{\ourdata}{PediaTypes\xspace}
\newcommand{\ourotherdata}{WikiTopics\xspace}

\newcommand{\ourinv}{double invariant\xspace}

\newcommand{\Ourequiv}{Double equivariant\xspace}
\newcommand{\ourequiv}{double equivariant\xspace}

\newcommand{\Ourgraphease}{Knowledge graph\xspace}
\newcommand{\ourgraphease}{knowledge graph\xspace}
\newcommand{\ourgrapheases}{knowledge graphs\xspace}

\newcommand{\ourgraph}{knowledge graph\xspace}
\newcommand{\ourgraphs}{knowledge graphs\xspace}

\newcommand{\OurTask}{Fully Inductive Link Prediction\xspace}
\newcommand{\Ourtask}{Fully inductive link prediction\xspace}
\newcommand{\ourtask}{fully inductive link prediction\xspace}

\newcommand{\OurModel}{ISDEA\xspace}
\newcommand{\OurNewModel}{ISDEA+\xspace}
\newcommand{\OurOtherModel}{DEq-InGram\xspace}

\newcommand{\tr}{\text{train}\xspace}

\newcommand{\te}{\text{inf}\xspace}

\newcommand{\suptr}{^{(\tr)}}
\newcommand{\supte}{^{(\te)}}

\newcommand{\trifunc}[1]{#1_\text{triplet}}
\newcommand{\grafunc}[1]{#1_\text{graph}}

\DeclareMathOperator{\bigldcbrace}{\big\{\mskip-7mu\big\{}
\DeclareMathOperator{\bigrdcbrace}{\big\}\mskip-7mu\big\}}

\makeatletter
\let\save@mathaccent\mathaccent
\newcommand*\if@single[3]{%
    \setbox0\hbox{${\mathaccent"0362{#1}}^H$}%
    \setbox2\hbox{${\mathaccent"0362{\kern0pt#1}}^H$}%
    \ifdim\ht0=\ht2 #3\else #2\fi
    }
\newcommand*\rel@kern[1]{\kern#1\dimexpr\macc@kerna}
\newcommand*\widebar[1]{{\@ifnextchar^{{\wide@bar{#1}{0}}}{\wide@bar{#1}{1}}}}
\newcommand*\wide@bar[2]{\if@single{#1}{\wide@bar@{#1}{#2}{1}}{\wide@bar@{#1}{#2}{2}}}
\newcommand*\wide@bar@[3]{%
\begingroup
\def\mathaccent##1##2{%
    \let\mathaccent\save@mathaccent
    \if#32 \let\macc@nucleus\first@char \fi
    \setbox\z@\hbox{$\macc@style{\macc@nucleus}_{}$}%
    \setbox\tw@\hbox{$\macc@style{\macc@nucleus}{}_{}$}%
    \dimen@\wd\tw@
    \advance\dimen@-\wd\z@
    \divide\dimen@ 3
    \@tempdima\wd\tw@
    \advance\@tempdima-\scriptspace
    \divide\@tempdima 10
    \advance\dimen@-\@tempdima
    \ifdim\dimen@>\z@ \dimen@0pt\fi
    \rel@kern{0.6}\kern-\dimen@
    \if#31
        \overline{\rel@kern{-0.6}\kern\dimen@\macc@nucleus\rel@kern{0.4}\kern\dimen@}%
        \advance\dimen@0.4\dimexpr\macc@kerna
        \let\final@kern#2%
        \ifdim\dimen@<\z@ \let\final@kern1\fi
        \if\final@kern1 \kern-\dimen@\fi
    \else
        \overline{\rel@kern{-0.6}\kern\dimen@#1}%
    \fi
}%
\macc@depth\@ne
\let\math@bgroup\@empty \let\math@egroup\macc@set@skewchar
\mathsurround\z@ \frozen@everymath{\mathgroup\macc@group\relax}%
\macc@set@skewchar\relax
\let\mathaccentV\macc@nested@a
\if#31
    \macc@nested@a\relax111{#1}%
\else
    \def\gobble@till@marker##1\endmarker{}%
    \futurelet\first@char\gobble@till@marker#1\endmarker
    \ifcat\noexpand\first@char A\else
        \def\first@char{}%
    \fi
    \macc@nested@a\relax111{\first@char}%
\fi
\endgroup
}
\makeatother

\def\adj{\mathbf{A}} %
\def\adjwild{\mathbf{A}^{(*)}}

\def\relsizewild{{R^{(*)}}}
\def\nodesize{{N}}
\def\relsize{{R}}

\def\eqref#1{equation~\ref{#1}}

\def\1{\bm{1}}

\def\vz{{\bm{z}}}

\def\mR{{\bm{R}}}

\def\mV{{\bm{V}}}

\DeclareMathAlphabet{\mathsfit}{\encodingdefault}{\sfdefault}{m}{sl}
\SetMathAlphabet{\mathsfit}{bold}{\encodingdefault}{\sfdefault}{bx}{n}
\newcommand{\tens}[1]{\bm{\mathsfit{#1}}}

\def\tB{{\tens{B}}}

\def\tZ{{\tens{Z}}}

\def\cB{{\mathcal{B}}}

\def\cG{{\mathcal{G}}}

\def\cN{{\mathcal{N}}}

\def\cR{{\mathcal{R}}}
\def\cS{{\mathcal{S}}}

\def\cV{{\mathcal{V}}}

\def\sA{{\mathbb{A}}}

\def\sN{{\mathbb{N}}}

\def\sR{{\mathbb{R}}}
\def\sS{{\mathbb{S}}}

\def\L2L{L^2(\mathcal{J}) \rightarrow L^2(\mathcal{J})}

\renewcommand{\L}{\mathcal{L}} 
\def\gG{{\mathcal{G}}}

\def\gR{{\mathcal{R}}}

\def\gV{{\mathcal{V}}}

\def\adj{{\mathbf{A}}} 
\newcommand{\rel}{_{\rm{rel}}}
\newcommand{\trip}{^{(\rm{triplet})}}

\theoremstyle{plain}
\newtheorem{theorem}{Theorem}[section]

\theoremstyle{definition}
\newtheorem{definition}[theorem]{Definition}

\theoremstyle{remark}

\newcommand{\bruno}[1]{}
\newcommand{\jincheng}[1]{}
\newcommand{\yucheng}[1]{}

\newcommand{\update}[1]{#1}

\newcommand{\red}{\color{red}}  %

\title[Double Equivariance for Inductive Link Prediction for Both New Nodes and New Relation Types]{Double Equivariance for Inductive Link Prediction for Both New Nodes and New Relation Types}

\author{%
  Jincheng Zhou \\
  Purdue University \\
  \email{zhou791@purdue.edu} \\
  \And
  Yucheng Zhang \\
  Purdue University \\
  \email{zhan4332@purdue.edu} \\
  \And
  Jianfei Gao\thanks{Equal contribution} \\
  Purdue University \\
  \email{gao462@purdue.edu} \\
  \And
  Yangze Zhou\footnotemark[1] \\
  Purdue University \\
  \email{zhou950@purdue.edu} \\
  \And
  Bruno Ribeiro \\
  Purdue University \\
  \email{ribeiro@cs.purdue.edu} \\
}

\begin{document}

\maketitle

\begin{abstract}

The task of fully inductive link prediction in knowledge graphs has gained significant attention, with various graph neural networks being proposed to address it. 
This task presents greater challenges than traditional inductive link prediction tasks with only new nodes, as models must be capable of zero-shot generalization to both unseen nodes and unseen relation types in the inference graph. 
Despite the development of novel models, a unifying theoretical understanding of their success remains elusive, and the limitations of these methods are not well-studied. 
In this work, we introduce the concept of double permutation-equivariant representations and demonstrate its necessity for effective performance 
\update{in this task.}
We show that many existing models, despite their diverse architectural designs, conform to this framework. 
However, we also identify inherent limitations in double permutation-equivariant representations, which restrict these models’ ability to learn effectively on datasets with varying characteristics. 
Our findings suggest that while double equivariance is necessary for meta-learning across knowledge graphs from different domains, it is not sufficient. 
There remains a fundamental gap between double permutation-equivariant models and the concept of {\em foundation models} designed to learn patterns across all domains.

\end{abstract}

\section{Introduction}
\label{sec:intro}
Knowledge graphs (KGs) are often domain-specific, such as those used in health care~\citep{chandak2023building}, material science~\citep{venugopal2022largest,statt2023materials}, e-commerce~\citep{dong2018challenges}, etc..
For robust learning across multiple domains, it is desirable to pretrain knowledge graph models on multiple domains and then perform zero-shot predictions on new (unseen) domains (\Cref{fig:example-link}).
This work lays the foundations for what we term {\em zero-shot fully inductive link prediction}, which aims to predict missing factual triplets in {\ourgraph}s at inference time, even when both nodes and relations are unseen during training~\citep{ingram,galkin2023towards}.

{\Ourtask} differs from conventional inductive link prediction task~\citep{schlichtkrull2018modeling,teru2020inductive,galkin2021nodepiece,zhu2021neural,chenrefactor}, where only unseen nodes are found during inference, but the vocabulary of relation types remains the same. 
But while there are graph neural networks (GNNs) for zero-shot {\ourtask}, such as InGram~\citep{ingram} and Ultra~\citep{galkin2023towards} (sometimes described as knowledge graph foundation models~\citep{galkin2023towards,mao2024graph}), there is still no unified theoretical framework that explains why these approaches should work or any theory to guide the design of future architectures.

{\bf Contributions.} 
This work introduces a general theoretical framework, named {\em double permutation-equivariant graph representations}, which we demonstrate as the underlying principle behind models like InGram~\citep{ingram} and Ultra~\citep{galkin2023towards}. We demonstrate how this framework provides a general blueprint for designing future models by constructing a simple model following the theory and improving an existing model (InGram).
Most importantly, we highlight the limitations of {\bf all} such models in scenarios where negative transfer~\citep{zhang2022survey,wang2022mitigating} happens. Specifically:
\begin{enumerate}[leftmargin=*]
    \item 
    We formally define the notion of {\em double equivariant structural representations} and {\em distributionally double equivariant positional embeddings}. We demonstrate the generality of this theoretical framework by proving
    that Ultra~\citep{galkin2023towards} produces double equivariant structural representations, while InGram~\citep{ingram} generates distributionally double equivariant positional embeddings. 
    \item 
    Based on the these insights, we propose a simple yet effective variant of InGram, named {\em \OurOtherModel}, which enhances the robustness and stability of the original InGram model.
    Additionally, we introduce a straightforward modeling framework called {\em \OurNewModel}, based on the principle of double equivariant structural representations. 
    \update{This framework can transform {\em any} GNNs designed for homogeneous graphs into double equivariant models suitable for {\ourgraph}s.}
    \item We empirically demonstrate the effectiveness and improvements brought by our framework on two new {\ourgraph} datasets: \ourdata and \ourotherdata. Despite the positive results, we show that {\bf all} existing double equivariant models face significant limitations in their ability to jointly learn across multiple domains. Specifically, these models suffer from negative transfer when certain {\ourgraph}s are combined during training. This finding highlights a key challenge on the path toward achieving ideal graph foundation models~\citep{galkin2023towards,mao2024graph}.
\end{enumerate}

\begin{figure}
\begin{subfigure}{0.45\textwidth} \label{fig:example-link-a}
\centering
\includegraphics[width=\linewidth]{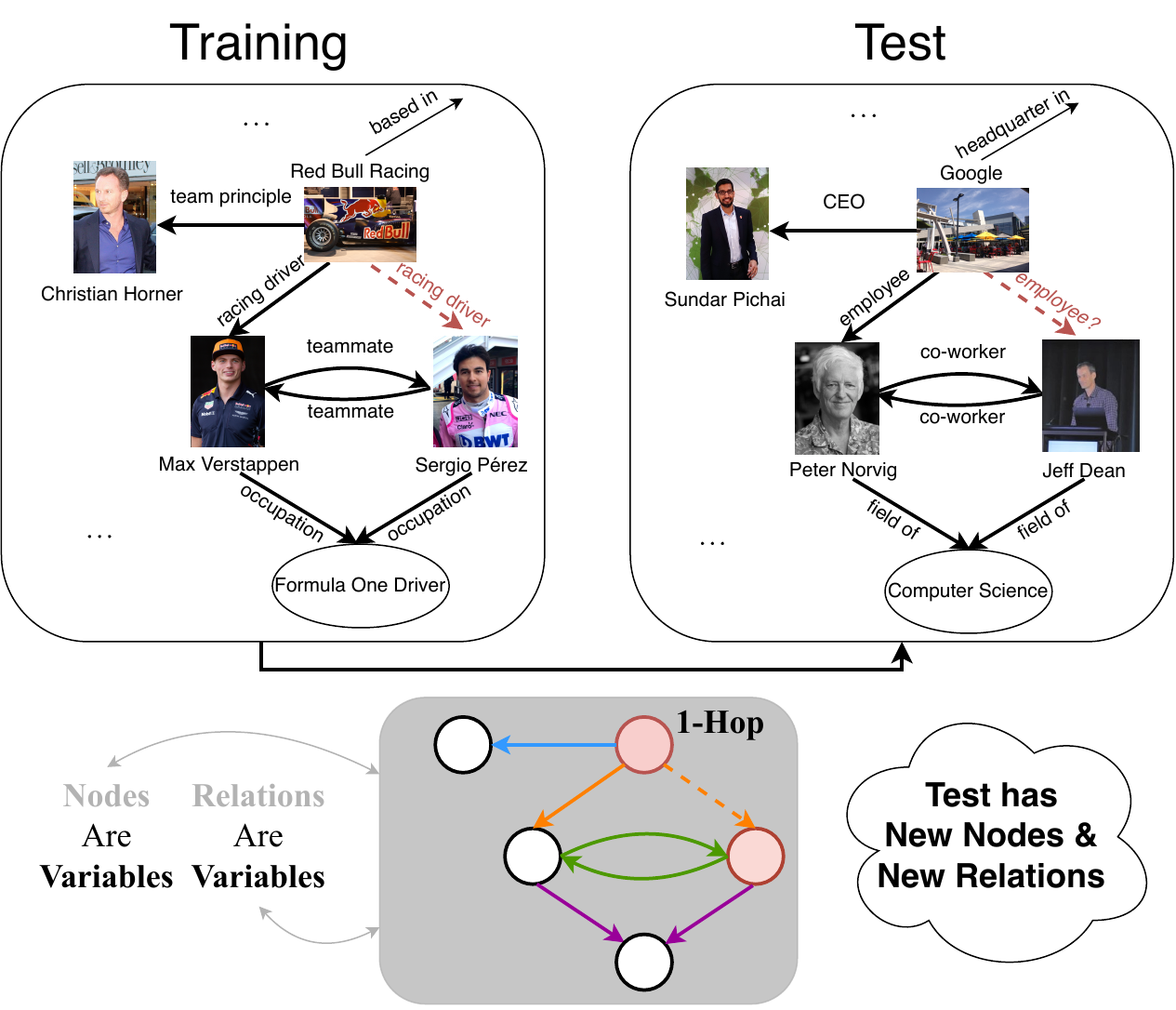} 
\subcaption{{\bf (Our task) {\ourtask}}}
\end{subfigure}
\hfill
\begin{subfigure}{0.45\textwidth} \label{fig:example-link-b}
\centering
\includegraphics[width=\linewidth]{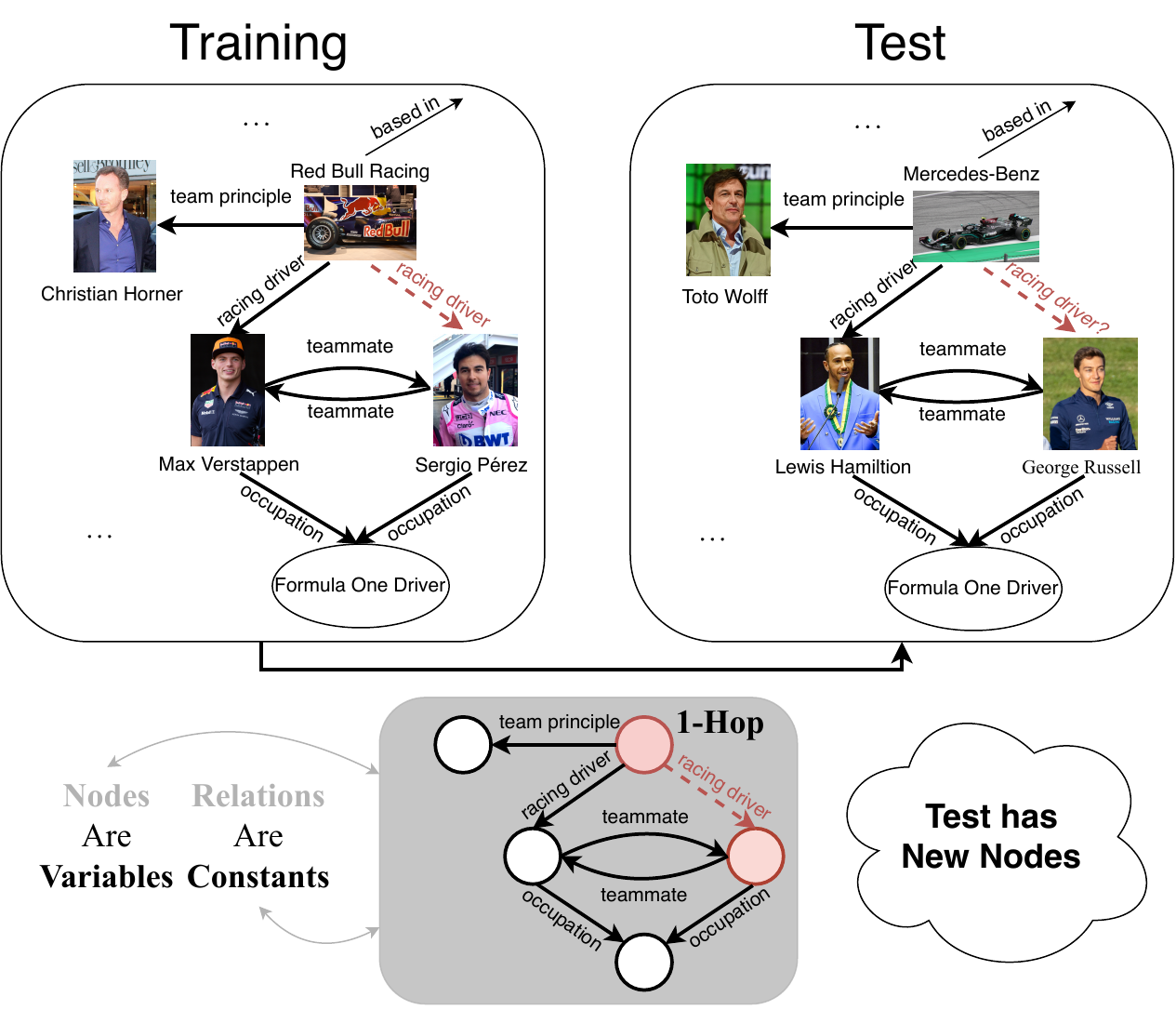} 

\subcaption{{\bf Node-only inductive link prediction}}
\end{subfigure}
\caption{\footnotesize  (a) {\bf {\Ourtask}:} Model learns to inductively predict new facts over an inference-time KG from a different domain, with 
new nodes and new relation types.
(b) {\bf Traditional node-only inductive link prediction:} 
\update{Inference-time KG has new nodes but relation types are all seen in training.}
}
\label{fig:example-link}
\vspace{-15pt}
\end{figure}


\section{Double Equivariance Enables Fully Inductive Link Prediction}\label{sec:theory}
In what follows, we first introduce the \ourtask task. We then proceed to theoretically describe the task in a general setting and propose our double equivariant modeling framework to handle the task using structural representations and positional embeddings.

\subsection{Formalizing the Fully Inductive Link Prediction Task}
\label{subsec:taskdef}
We now introduce notations and definitions used throughout this paper.
First, we formally define our inductive link prediction task for both new nodes and new relation types, i.e., \ourtask, over \ourgraph{s}. We denote $[n] := \{ 1, \dots, n \}$ for any $n\in \sN$.
Let $\cG\suptr = (\cV\suptr, \cR\suptr, \adj\suptr)$ be the training \ourgraph, where $\cV\suptr$ is the set of $N\suptr$ training nodes, $\cR\suptr$ is the set of $R\suptr$ training relation types.
We also define two associated bijective mappings $v\suptr_{\cdot}:[N\suptr]\rightarrow \cV\suptr, r\suptr_{\cdot}:[R\suptr]\rightarrow \cR\suptr$ that enumerate the nodes and relation types in training.
The tensor $\adj\suptr \in \{0, 1\}^{N\suptr \times R\suptr \times N\suptr}$ defines the adjacency  of the training graph such that $\forall (i,k,j)\in [N\suptr]\times [R\suptr]\times [N\suptr], \adj\suptr_{i, k, j} = 1$ indicates that the triplet $(v\suptr_i, r\suptr_k, v\suptr_j)$ is present in the data (%
we denote $(i,k,i)$ as the $k$-th attribute of node $i$). To simplify notation, we further refer to the collection of all {\ourgraph}s of any sizes as $\sA \coloneqq \cup_{N=1}^\infty\cup_{R=2}^\infty \{0, 1\}^{N \times R \times N}$. 

\begin{definition}[\textbf{\Ourtask task}]\label{def:task-double}
The task of \ourtask learns a model on a set of $K$ training graphs $\{\cG\suptr_1,\ldots,\cG\suptr_K\}$ (often in a self-supervised fashion by masking links) and inductively applies it to predict missing links in an inference graph $\cG\supte = (\cV\supte, \cR\supte, \adj\supte)$ with new nodes and new relation types, i.e., $\cV\supte \not\subset \cup^{K}_{i=1}\cV_{i}\suptr, \cR\supte \not\subset \cup^{K}_{i=1}\cR_{i}\suptr$, without extra context given to the model.
\end{definition}
In what follows, for simplicity we ignore the superscript $\text{train}$ and $\text{test}$.
And since there are bijections $v_\cdot \text{ and } r_\cdot$ between indices and nodes and relation types, we represent the triplet $(v_i,r_k,v_j)\in \cV\times \cR\times \cV$ with indices $(i,k,j)\in[N]\times[R]\times[N]$, and mainly use $\adj$ to denote the \ourgraph.  %

\update{
\Cref{fig:example-link}(a) shows a concrete example of a fully inductive link prediction task. Here, the training KG capturing the topic of F1 racing and the test KG concerning the organizational structure of a company have {\em non-overlapping} set of entities and relation types. This is in contrast to existing node-only inductive task as depicted in \Cref{fig:example-link}(b) as well as some of the relation-inductive settings studied by existing work~\citep{geng2023relational,zhao2020attention,MaKEr,csr2022,jin2022inductive}. Notably, the latter generally can handle unseen relation types {\em only when the test graph is a superset of the training graph}. In this case, the existence of an edge with an unseen relation type could be inferred from its neighborhood that only contains training relation types. These methods thus cannot be applied to \Cref{fig:example-link}(a) - here the neighborhood around the target edge in the test graph, (Google, employee, Jeff Dean), only consists of other unseen relation types. 
We also note that the test graph does not carry additional context such as textual description embeddings for unseen relations or graph ontology. \Cref{sec:related,sec:appd-rel} have more thorough discussion.

Because of the generality of \Cref{def:task-double}, the \ourtask forces the model to not rely on potential overlaps or additional contextual information, and must learn to differentiate nodes and relations based only on their structural relationships in $\adj$, rather than their labels in $\cV \text{ and } \cR$. This is achieved by what we call the {\em double equivariant representations}, which we elaborate next.
}

%
%

%
%
%

%
%
%
%
%
%
%
 %
\subsection{Double Equivariant Representations for Knowledge Graphs}
\label{sec:def}
In what follows, we provide definitions and theoretical statements of our proposed double equivariant \ourgraphease representations in the main paper while referring all proofs to \Cref{sec:proof}.
The proposal starts with defining the permutation actions on {\ourgraphs} as:%
\begin{definition}[\textbf{Node and relation permutation actions on {\ourgrapheases}}]
\label{def:rel-graph-perm}
For any \ourgraphease $\adj \in \sA$ with number of nodes and relations $\nodesize, \relsize$, a node permutation $\phi \in \sS_\nodesize$ is an element of the symmetric group $\sS_\nodesize$, a relation permutation $\tau \in \sS_\relsize$ is an element of the symmetric group $\sS_\relsize$, and the operation $\phi \circ \tau \circ \adj$ is the action of $\phi$ and $\tau$ on $\adj$, defined as $\forall (i, k, j) \in [\nodesize]\times[\relsize]\times[\nodesize], (\phi \circ\tau\circ \adj)_{\phi \circ i, \tau\circ k, \phi \circ j} = \adj_{i,k,j}$ where $\phi \circ i = \phi_i$ and $\tau\circ k = \tau_k$. The node and relation permutation actions on $\adj$ are commutative, i.e., $\phi\circ\tau\circ \adj = \tau\circ\phi\circ \adj$.
\end{definition}
To learn structural representation for both nodes and relations, %
we first design triplet representations that are invariant to the two permutation actions on nodes and relations, as shown below.
\begin{definition}[\textbf{Double invariant triplet representations}]
\label{def:trip-repr}
For any \ourgraphease $\adj \in \sA$ with number of nodes and relations $\nodesize,\relsize$, a double invariant triplet representation %
is a function $\trifunc{\Gamma}: \cup_{N=1}^\infty\cup_{R=2}^\infty([N]\times[R]\times[N])\times \sA \rightarrow \sR^{d}, d \geq 1$, such that $\forall (i, k, j) \in [\nodesize]\times[\relsize]\times[\nodesize], \forall \phi \in \sS_\nodesize, \forall \tau \in \sS_\relsize$, $\trifunc{\Gamma}((i, k, j), \adj) = \trifunc{\Gamma}((\phi \circ i, \tau \circ k, \phi \circ j), \phi \circ \tau \circ \adj)$.
\end{definition}
To understand the property of our double invariant triplet representations, we first introduce the notion of \ourgraph isomorphism and triplet double isomorphism.
\begin{definition}[\textbf{\Ourgraphease isomorphism and Triplet isomorphism}]
\label{def:rel-graph-iso}
We say two {\ourgrapheases} $\adj^{(G)}, \adj^{(G^\prime)} \in \sA$ with number of nodes and relations $N^{(G)},R^{(G)}$
and $N^{(G^\prime)}, R^{(G^\prime)}$ respectively, are isomorphic (denoted as ``$\adj^{(G)} \simeq_\text{graph} \adj^{(G^\prime)}$'') if and only if $\exists \phi\in\sS_{N^{(G)}}, \exists\tau\in \sS_{R^{(G)}}$, such that $\phi \circ \tau \circ \adj^{(G)} = \adj^{(G^\prime)}$. %
And we say two triplets $( i^{(G)}, k^{(G)}, j^{(G)} )\in [N^{(G)}]\times[R^{(G)}]\times[N^{(G)}]$, $( i^{(G^\prime)}, k^{(G^\prime)}, j^{(G^\prime)} ) \in [N^{(G^\prime)}]\times[R^{(G^\prime)}]\times[N^{(G^\prime)}]$ are isomorphic triplets (denoted as ``$(( i^{(G)}, k^{(G)}, j^{(G)} ), \adj^{(G)}) \simeq_\text{triplet} ((  i^{(G^\prime)}, k^{(G^\prime)}, j^{(G^\prime)}), \adj^{(G^\prime)})$'') if and only if $\exists \phi\in\sS_{N^{(G)}}, \exists\tau\in \sS_{R^{(G)}}$, such that $\phi \circ \tau \circ \adj^{(G)} = \adj^{(G^\prime)}$ and $( i^{(G^\prime)}, k^{(G^\prime)}, j^{(G^\prime)} ) = ( \phi \circ i^{(G)}, \tau \circ k^{(G)}, \phi \circ j^{(G)} )$.
\end{definition}

\update{
As an example, the training and test graphs in \Cref{fig:example-link}(a) are isomorphic KGs. Moreover, the triplet (Red Bull Racing, racing driver, Sergio Pérez) in the training graph is isomorphic to the triplet (Google, employee, Jeff Dean) in the test graph, as they are the counterparts of each other in the respective isomorphic KGs.
}
It is clear that our double invariant triplet representations are able to output the same representations for these isomorphic triplets, enabling \ourtask, where the model trained to predict the missing (Red Bull Racing, racing driver, Sergio Pérez) in the training graph is able to predict the missing (Google, employee, Jeff Dean) in test. 

\update{
\Cref{def:trip-repr,def:rel-graph-iso} generalize the traditional graph isomorphism defined on homogeneous graphs~\citep{xu2018powerful,srinivasan2020on} to knowledge graphs, by considering not only permutations actions on the nodes, which is those concerned by traditional graph isomorphisms, but also permutation actions on the relation types. 
}
The connection between \Cref{def:trip-repr} and logical reasoning can be found in \Cref{appx:logic}.
Next, we define the structure \ourequiv representations for the whole \ourgraphease $\adj$ (akin to how GNNs provide representations for a whole graph). %
\begin{definition}[\textbf{\Ourequiv \ \ourgraphease structural representations}]
\label{def:rel-graph-repr}
For any \ourgraphease $\adj \in \sA$ with  number of nodes and relations $\nodesize, \relsize$, a function $\grafunc{\Gamma}: \sA \rightarrow \cup_{N=1}^\infty\cup_{R=2}^\infty \sR^{N \times R \times N \times d}, d \geq 1$ is \ourequiv w.r.t.\ arbitrary node $\phi \in \sS_\nodesize$ and relation $\tau \in \sS_\relsize$ permutations, if $\grafunc{\Gamma}(\phi \circ \tau \circ \adj) = \phi \circ \tau \circ \grafunc{\Gamma}(\adj)$.
Moreover, valid mappings of $\Gamma_\text{graph}$ must map a domain element to an image element with the same number of nodes and relations.
\end{definition}
Finally, we connect \Cref{def:trip-repr,def:rel-graph-repr} by showing how to build double equivariant graph representations from double invariant triplet representations in \Cref{thm:inv-and-equiv}, and vice-versa.
\begin{restatable}[\textbf{From double invariant triplet representations to double equivariant graph representations}]{theorem}{thmone}
\label{thm:inv-and-equiv}
For all $\adj \in \sA$ with number of nodes and relations $\nodesize,\relsize$, given a \ourinv triplet representation $\trifunc{\Gamma}$, we can construct a \ourequiv graph representation as $\left( \grafunc{\Gamma}(\adj) \right)_{i,k,j} := \trifunc{\Gamma}((i,k,j), \adj)$, $\forall (i,k,j) \in [\nodesize]\times [\relsize]\times [\nodesize]$, and vice-versa.
\end{restatable}
Next, we consider positional graph embeddings that are equivariant in distribution.

\subsection{Distributionally Double Equivariant Positional Graph Embeddings}
\label{sec:positional}
InGram~\citep{ingram} is an existing work capable of performing our \ourtask task~(\Cref{def:task-double}), but it does so with what we now define as {\em distributionally double equivariant positional embeddings}, which are permutation sensitive, as we will show in \Cref{subsec:de-ingram}:

\begin{definition}[\textbf{Distributionally double equivariant positional embeddings}]
\label{def:pos-trip-repr}
For any \ourgraphease $\adj \in \sA$ with number of nodes and relations $\nodesize,\relsize$, the distributionally double equivariant positional embeddings of $\adj$ are defined as joint samples of random variables $\tZ|\adj \sim p(\tZ|\adj)$,  where the tensor $\tZ $ is defined as $\tZ_{i,k,j}\in\sR^d, d\geq 1, \forall (i,k,j) \in [\nodesize]\times[\relsize]\times[\nodesize]$, where we say $p(\tZ|\adj)$ is a double equivariant probability distribution on $\adj$ defined as $\forall \phi\in \sS_\nodesize,\forall \tau\in \sS_\relsize,  p(\tZ |\adj) = p(\phi \circ \tau \circ  \tZ| \phi \circ \tau \circ \adj)$.
\end{definition}
Prior work on (standard) link prediction tasks has shown the advantages of equivariant representations over positional embeddings~\citep{zhang2021labeling}.
Others have established the equivalence between positional embeddings and structural representations for simple graphs by proving that representations based on an expectation of the positional embeddings are equivariant to node permutations~\citep{srinivasan2020on}.
In what follows, we extend this result to the double equivariant setting:

\begin{restatable}[\textbf{From distributional double equivariant positional embeddings to double equivariant representations}]{theorem}{thmthree}
\label{thm:equivalence} For any \ourgraphease $\adj \in \sA$, %
the average $\mathbb{E}_{p(\tZ|\adj)}[\tZ|\adj]$ is a double equivariant \ourgraphease representation (\Cref{def:rel-graph-repr}) for any distributional double equivariant positional embeddings $\tZ|\adj$ (\Cref{def:pos-trip-repr}). %
\end{restatable}
\update{
Intuitively, a distributionally double equivariant model will produce different embeddings for triplet (Red Bull Racing, racing driver, Sergio Pérez) and (Google, employee, Jeff Dean) in \Cref{fig:example-link}(a). However, if the model is run multiple times (say each time with a different random seed), then the averaged embeddings for the two triplets averaged across runs will be very similar.
}
Later in \Cref{subsec:de-ingram}, we use the result in \Cref{thm:equivalence} to introduce DEq-InGram, a double equivariant representation that builds upon  InGram's distributionally double equivariant positional embeddings (\Cref{def:pos-trip-repr}) that is shown to significantly outperforms the original InGram in \Cref{sec:exp}.


\section{
Double Equivariance to Guide Model Design}
\label{sec:model}

Ultra~\citep{galkin2023towards} and InGram~\citep{ingram} are two recently proposed methods designed to address the {\ourtask} task. In this section, we demonstrate that both methods can be understood within the framework of double equivariance: Ultra generates double equivariant structural representations, while InGram produces distributionally double equivariant positional embeddings. Leveraging insight from~\Cref{thm:equivalence}, we present a straightforward enhancement to InGram, which we call \OurOtherModel. Next, we describe \OurNewModel, a general framework for converting any equivariant GNNs designed for homogeneous graphs into double equivariant models.

\subsection{Ultra Generates Double Equivariant Structural Representations}
\label{subsec:ultra}
%
\update{Ultra was proposed by~\citet{galkin2023towards} to address the {\ourtask} task.}
It builds upon NBFNet~\citep{zhu2021neural}, a path-based neural model originally designed for link prediction in homogeneous graphs. The key architectural innovation of Ultra is the construction of an {\em unweighted, directed relation graph} $\cG_{\text{rel}}$. In this relation graph, nodes represent the relations $\cR$ of the original graph $\cG$, and directed edge capture interaction between these relations with four distinct types: head-to-head, head-to-tail, tail-to-head, and tail-to-tail. Specifically, for two relations $r_1$ and $r_2$ in $\cR$, if there exists a node $v$ in $\cV$ that serves as the head for both $r_1$ and $r_2$ in the original graph $\cG$, then $r_1$ and $r_2$ are connected by a head-to-head edge in $\cG_{\text{rel}}$. The same logic applies to the other three types of edges, effectively capturing various ways relations interact within the {\ourgraph}.

During the forward propagation, Ultra first runs an instance of NBFNet on the relation graph $\cG_{\text{rel}}$. The resulting node representations from $\cG_{\text{rel}}$ correspond to representations of the relations of $\cR$ and are used as edge representations in the original graph $\cG$. Subsequently, a second instance of NBFNet is applied to produce the final node representation in $\cG$.

Our main result concerning Ultra is that it operates as a double equivariant graph representation model(\Cref{def:rel-graph-repr}). The key intuition is that the relation graph $\cG_{\text{rel}}$ is constructed to be invariant to the permutation of relation identities. Since NBFNet itself is an equivariant GNN, the resulting relation representations are equivariant to permutations of relations. Consequently, the final node representations generated by the second NBFNet instance are invariant under the combined permutations of nodes and relations. For a detailed proof, please refer to~\Cref{appx:extraThmProofs}.
\begin{restatable}{lemma}{lemultra}
    The triplet representations generated by Ultra~\citep{galkin2023towards} are double equivariant structural representations (\Cref{def:rel-graph-repr}).
\end{restatable}

\subsection{InGram Generates Distributionally Double Equivariant Positional Embeddings}
\label{subsec:de-ingram}

InGram, introduced by~\citet{ingram}, is another model designed to address the {\ourtask} task. Like Ultra~\citep{galkin2023towards}, InGram constructs a relation graph to capture the interaction between relations. However, unlike Ultra, InGram's relation graph is undirected but weighted, with edge weights computed heuristically to represent the affinity between relations. Intuitively, a high affinity between two relations indicates that they frequently co-occur in the original graph $\cG$. 

In the forward propagation, Ingram utilizes a GATv2 model~\citep{brody2021attentive} on the relation graph to generate relation embeddings. These relation embeddings are then input into a second GATv2 instance operating on the original graph to produce node embeddings. Finally, InGram applies a variant of DistMult~\citep{yang2015embedding} to compute triplet scores using the generated node and relation embeddings. 

Another key difference from Ultra, as we demonstrate in the following lemma, is that the embeddings produced by InGram are sensitive to permutations of node and relation identities. This sensitivity arises because InGram relies on initial node and relation embeddings before the first neural network layer, which are initialized via Glorot initialization~\citep{glorot2010understanding} and re-initialized at the start of each training epoch and during inference. Since random initialization is inherently sensitive to permutations of identities, the resulting embeddings from InGram are not permutation-equivariant. However, due to the model's architectural design, we show that InGram's embeddings can be considered as distributionally double equivariant positional embeddings (\Cref{def:pos-trip-repr}), meaning they are double equivariant {\em in expectation}. For a detailed proof, please refer to~\Cref{appx:extraThmProofs}.
\begin{restatable}{lemma}{lemfour} \label{lemma:ingram-repr}
The triplet representations generated by InGram~\citep{ingram} output distributionally double equivariant positional embeddings (\Cref{def:pos-trip-repr}). %
\end{restatable}

\subsection{\OurOtherModel: Enhancing Double Equivariance via Monte Carlo Sampling}
%

Building on our insight from~\Cref{lemma:ingram-repr} that InGram~\citep{ingram}'s embeddings are double equivariant {\em in expectation}, we propose a straightforward enhancement, which we name {\OurOtherModel}. As we will later demonstrate in~\Cref{sec:exp}, despite its simplicity, this modification significantly enhances the model's performance, showing the strength of our theoretical framework. 

Specifically, our approach uses the insight from~\Cref{thm:equivalence} by employing Monte Carlo sampling to obtain an estimate of a double equivariant graph representation. Let $\mathbf{Z}_{\text{InGram}}((i, k, j), \adj\supte, \mV^{(0)}, \mR^{(0)})$ denote InGram's triplet score function, where $\adj\supte$ is the inference \ourgraphease, $\mV^{(0)} \in \sR^{N\supte \times d}$ represents an initial node embeddings of dimension $d$, and $\mR^{(0)} \in \sR^{R\supte \times d'}$ represents an initial relation embeddings of dimension $d'$. Our {\OurOtherModel} computes its embeddings as follows:
\vspace{-10pt}
\begin{align}
    \Gamma_{\text{\OurOtherModel}}((i, k, j), \adj\supte) = \frac{1}{M} \sum_{m=1}^M \mathbf{Z}_{\text{InGram}}((i, k, j), \adj\supte, \mV^{(0)}_m, \mR^{(0)}_m)
    \label{eq:de-ingram}
\end{align}
\vspace{-10pt}

where $\{\mV^{(0)}_m\}_{m=1}^M$ and $\{\mR_m^{(0)}\}_{m=1}^M$ are $M$ i.i.d. samples drawn from Glorot initialization~\citep{glorot2010understanding}. Notably, this modification is applied at inference time and does not require retraining the InGram model. Hence, one can use a pre-trained InGram checkpoint directly and achieve improved test accuracy using our modification.

\subsection{Inductive Structural Double Equivariant Architecture Plus (\OurNewModel)}
\label{subsec:isdea}
Finally, we propose \OurNewModel, an alternative approach to Ultra~\citep{galkin2023towards} and InGram~\citep{ingram} which allows us to convert any GNNs designed for homogeneous graphs to a double equivariant graph representation model. Different from Ultra and InGram, we do not explicitly construct a relation graph. Rather, we perform an equivariant set aggregation of node representations over the relations using the DSS aggregation layer~\citep{maron2020learning,bromley1993signature}. This guarantees that the resulting representations are equivariant to permutations of relation identities.

Specifically, given a {\ourgraph} $\cG$ with adjacency matrix $\adj$, let $A^{(k)}$ be the relation-induced subgraph of $\adj$ consisting of only edges with relation $k \in \cR$. Then, $\cG$ can be equivalently expressed as a collection of relation-induced subgraphs, $\adj = \bigldcbrace A^{(1)}, \ldots, A^{(\relsize)} \bigrdcbrace$. Since the actions of the two permutation groups $\sS_\nodesize$ and $\sS_\relsize$ commute, the double equivariance of $\adj$ (\Cref{def:rel-graph-iso}) can be described as two (single) equivariance: A (graph) equivariance $\phi \in \sS_\nodesize$ over each relation-induced subgraph $A^{(k)}, k = 1, \ldots, \relsize$, and a (set) equivariance $\tau \in \sS_\relsize$  (over the set of subgraphs). Our double equivariant \OurNewModel model then make use of DSS layer~\citep{maron2020learning,bromley1993signature} to perform aggregation over set of relations. Specifically, the \OurNewModel layer $L :\sA \rightarrow \cup_{N=1}^\infty\cup_{R=2}^\infty \sR^{N \times R \times N \times d}$ is defined as follows. For each $k = 1, ..., \relsize$:
\vspace{-10pt}
\begin{equation}
\label{eq:dsslin}
\left( L \left( \adj \right) \right)_{:, k} = \text{GNN}_1 \left( A^{(k)} \right) + \text{GNN}_2 \left( \sum_{k' \in [\relsize], k \neq k} \left( A^{(k')} \right) \right),
\end{equation}
where $\text{GNN}_1, \text{GNN}_2: \sA \rightarrow \cup_{N=1}^\infty\sR^{N \times N \times d}$ are {\em arbitrary} node-equivariant GNN layers that produce pairwise representations~\citep{zhang2018link,zhu2021neural,zhang2021labeling}. Here, we abuse the notation $\sum$ to denote any valid set aggregation such as sum, mean, max, DepSets~\citep{zaheer2017deep}, etc.. The following lemma shows that \OurNewModel directly produces double equivariant representations. For a detailed proof, please refer to~\Cref{appx:extraThmProofs}.

\begin{restatable}{lemma}{lemthree}
$\Gamma_\text{\OurNewModel}$ in \Cref{eq:dsslin} is a \ourinv triplet representation as per \Cref{def:trip-repr}. %
\end{restatable}

\section{Related Work}\label{sec:related}
A more comprehensive discussion of related work can be found in \Cref{sec:appd-rel}.


\paragraph{Inductive link prediction over new nodes (but not new relations).}
Rule-induction methods~\citep{yang2015embedding,yang2017differentiable,meilicke2018fine,sadeghian2019drum} are inherently node-independent which aim to extract First-order Logical Horn clauses from the attributed multigraph.
Recently, with the advancement of GNNs, %
various works~\citep{schlichtkrull2018modeling,teru2020inductive,galkin2021nodepiece,zhu2021neural,chenrefactor} have applied the idea of GNN in relational prediction to learn structural node/pairwise representation. Although all these methods can be used to perform {\em inductive link prediction over solely new nodes}, they can not handle new relation types in test.
%

\paragraph{Inductive link prediction over both new nodes and new relations.}
Existing methods for querying triplets involving both new nodes and new relations generally assume access to extra context, such as generating language embedding for textual descriptions of unseen relation types~\citep{qin2020generative,geng2021ontozsl,Zha_Chen_Yan_2022,wang2021relational}, 
a shared background graph connecting seen and unseen relations (e.g., test graph has training relations~\citep{csr2022,chen2021topology,MaKEr}), 
or access to graph ontology~\citep{geng2023relational}. Hence, these methods cannot be directly applied to test graphs that neither contain meaningful descriptive information of the unseen relation types (e.g., url links) nor connection with 
nodes and relation types seen in training.
%

%
To the best of our knowledge, InGram~\citep{ingram} and Ultra~\citep{galkin2023towards} are the only existing methods capable of performing \ourtask without additional context data (just the test graph structure is available during inference). InGram and Ultra introduce relation graphs to capture relational representations based on their interactions. The connection between these methods and our work has been described in \Cref{sec:positional,subsec:de-ingram}. 
%

\section{Experiments on How Simple Double Equivariant Models Can Perform \OurTask}
\label{sec:exp}



%
%
\begin{table}[t]
\centering
\captionof{table}{
{\bf Relation \& Node Hits@10 performance on \OurTask over \ourdata.} We report standard deviations over 5 runs. %
A higher value means better \ourtask performance. 
\update{Models labeled with DEq are double equivariant models, and those labeled with d-DEq are distributionally double equivariant ones.}
The dataset name ``$X$-$Y$'' means training on graph $X$ and testing on graph $Y$. The best values are shown in bold font, while the second-best values are underlined. 
The highest standard deviation within each task is highlighted in red color.
``Rand'' row contains unbiased estimations of the performance from a random predictor.
{\bf Double equivariant models consistently achieve better results than non-double-equivariant models with generally smaller standard deviations.}
N/A*: Not available due to constant crashes.
}
{
\subcaption{{\bf {\scriptsize Relation prediction $(i, ?, j)$ performance in \%. Higher $\uparrow$ is better.}}}
\resizebox{0.99\linewidth}{!}{
\begin{tabular}{l r r r r r r r r}
    Models & EN-FR & FR-EN & EN-DE & DE-EN & DB-WD & WD-DB & DB-YG & YG-DB \\
    \midrule
    
    Rand & 19.60{\scriptsize $\pm 00.00$} & 19.60{\scriptsize $\pm 00.00$} & 19.60{\scriptsize $\pm 00.00$} & 19.60{\scriptsize $\pm 00.00$} & 19.60{\scriptsize $\pm 00.00$} & 19.60{\scriptsize $\pm 00.00$} & 19.60{\scriptsize $\pm 00.00$} & 19.60{\scriptsize $\pm 00.00$} \\
    \midrule
    
    GAT & 18.58{\scriptsize $\pm 00.52$} & 18.93{\scriptsize $\pm 00.33$} & 19.40{\scriptsize $\pm 00.28$} & 18.87{\scriptsize $\pm 00.19$} & 18.78{\scriptsize $\pm 00.28$} & 18.76{\scriptsize $\pm 00.33$} & 19.78{\scriptsize $\pm 01.39$} & 19.15{\scriptsize $\pm 00.35$} \\
    GIN & 19.34{\scriptsize $\pm 00.32$} & 19.34{\scriptsize $\pm 00.29$} & 18.98{\scriptsize $\pm 00.27$} & 18.88{\scriptsize $\pm 00.47$} & 19.30{\scriptsize $\pm 00.52$} & 18.86{\scriptsize $\pm 00.35$} & 18.69{\scriptsize $\pm 00.75$} & 18.92{\scriptsize $\pm 00.68$} \\
    GraphConv & 19.18{\scriptsize $\pm 00.27$} & 19.02{\scriptsize $\pm 00.64$} & 19.19{\scriptsize $\pm 00.24$} & 18.93{\scriptsize $\pm 00.60$} & 19.46{\scriptsize $\pm 00.38$} & 19.13{\scriptsize $\pm 00.54$} & 19.13{\scriptsize $\pm 01.24$} & 18.89{\scriptsize $\pm 00.57$} \\
    NBFNet & 21.93{\scriptsize $\pm 02.53$} & 22.20{\scriptsize $\pm 02.92$} & 18.98{\scriptsize $\pm 02.75$} &~~7.01{\scriptsize $\pm 01.43$} & 23.51{\scriptsize $\pm\red 07.06$} & 23.05{\scriptsize $\pm 03.55$} & 31.50{\scriptsize $\pm 04.82$} & 35.17{\scriptsize $\pm 05.13$} \\
    RMPI & 27.91{\scriptsize $\pm 06.48$} & 28.62{\scriptsize $\pm 03.75$} & 27.51{\scriptsize $\pm 06.48$}& 25.59{\scriptsize $\pm 06.48$} & N/A* & 16.76{\scriptsize $\pm 04.03$}& 39.03{\scriptsize $\pm 20.28$} & 11.77{\scriptsize $\pm 07.07$}\\
    \midrule
    InGram \update{(d-DEq)}  & 78.74{\scriptsize $\pm\red 07.48$} & 62.11{\scriptsize $\pm\red 13.60$} & 48.72{\scriptsize $\pm\red 08.94$} & 65.60{\scriptsize $\pm\red 14.42$} & 77.75{\scriptsize $\pm 06.60$} & 63.32{\scriptsize $\pm 02.78$} & 67.98{\scriptsize $\pm\red 25.45$} & 64.98{\scriptsize $\pm\red 26.69$} \\
    Ultra \update{(DEq)}
    &  \textbf{99.96}{\scriptsize $\pm 00.01$} 
    &  \textbf{99.85}{\scriptsize $\pm 00.04$} 
    &  \textbf{99.92}{\scriptsize $\pm 00.02$} 
    &  \textbf{99.69}{\scriptsize $\pm 00.16$} 
    &  \textbf{99.54}{\scriptsize $\pm 00.12$} 
    &  \underline{97.42}{\scriptsize $\pm 00.62$} 
    &  \textbf{95.72}{\scriptsize $\pm 01.67$} 
    &  \textbf{98.15}{\scriptsize $\pm 00.10$} 
    \\
    %
    
    \OurOtherModel \update{~(DEq)} & 87.94{\scriptsize $\pm 05.68$} & 80.47{\scriptsize $\pm 09.90$} & 68.89{\scriptsize $\pm 05.45$} & 80.79{\scriptsize $\pm 10.51$} & 91.47{\scriptsize $\pm 01.53$} & 77.03{\scriptsize $\pm\red 04.09$} & 77.72{\scriptsize $\pm 21.92$} & 89.30{\scriptsize $\pm 05.53$} \\

    \OurNewModel \update{~(DEq)} & \underline{99.12}{\scriptsize $\pm 00.24$} & \underline{98.84}{\scriptsize $\pm 00.06$} & \underline{99.20}{\scriptsize $\pm 00.13$} & \underline{98.99}{\scriptsize $\pm 00.12$} & \underline{98.56}{\scriptsize $\pm 00.12$} & \textbf{98.03}{\scriptsize $\pm 00.17$} & \underline{88.78}{\scriptsize $\pm 03.23$} & \underline{96.45}{\scriptsize $\pm 00.24$} \\
\end{tabular}
\label{tab:pediatypes-relation}
}
\vspace{5pt} %
{
\subcaption{{\bf {\scriptsize Node prediction $(i, k, ?)$ performance in \%. Higher $\uparrow$ is better.}}}
\resizebox{0.99\linewidth}{!}{
\begin{tabular}{l r r r r r r r r}
    Models & EN-FR & FR-EN & EN-DE & DE-EN & DB-WD & WD-DB & DB-YG & YG-DB \\
    \midrule
    
    Rand & 19.60{\scriptsize $\pm 00.00$} & 19.60{\scriptsize $\pm 00.00$} & 19.60{\scriptsize $\pm 00.00$} & 19.60{\scriptsize $\pm 00.00$} & 19.60{\scriptsize $\pm 00.00$} & 19.60{\scriptsize $\pm 00.00$} & 19.60{\scriptsize $\pm 00.00$} & 19.60{\scriptsize $\pm 00.00$} \\
    \midrule
    
    GAT & 89.77{\scriptsize $\pm 00.41$} & 86.83{\scriptsize $\pm 00.41$} & 66.24{\scriptsize $\pm 02.81$} & 69.08{\scriptsize $\pm 00.66$} & 31.08{\scriptsize $\pm 01.07$} & 77.05{\scriptsize $\pm 00.36$} & 53.51{\scriptsize $\pm 00.29$} & 64.13{\scriptsize $\pm 00.31$} \\
    GIN & 90.10{\scriptsize $\pm 00.61$} & 85.32{\scriptsize $\pm 01.18$} & 73.32{\scriptsize $\pm 03.35$}& 75.66{\scriptsize $\pm 04.85$} & 34.87{\scriptsize $\pm 09.12$} & 78.67{\scriptsize $\pm 02.46$} & 56.87{\scriptsize $\pm 00.44$} & 65.27{\scriptsize $\pm 01.14$} \\
    GraphConv & {92.97}{\scriptsize $\pm 00.11$} & \underline{90.56}{\scriptsize $\pm 00.04$} & 83.58{\scriptsize $\pm 00.68$} & {82.64}{\scriptsize $\pm 00.65$} & 40.59{\scriptsize $\pm 01.72$} & {79.28}{\scriptsize $\pm 01.29$} & 68.91{\scriptsize $\pm 00.51$} & 76.50{\scriptsize $\pm 00.14$} \\
    NBFNet & 87.64{\scriptsize $\pm 01.81$}& 89.77{\scriptsize $\pm 00.80$} & {85.56}{\scriptsize $\pm 02.07$} & 59.78{\scriptsize $\pm 03.73$}& {63.23}{\scriptsize $\pm 03.65$} & 78.24{\scriptsize $\pm 00.90$} & 49.97{\scriptsize $\pm 01.44$} & 66.36{\scriptsize $\pm 02.64$} \\
    RMPI & 89.59{\scriptsize $\pm\red 06.61$} & 81.79{\scriptsize $\pm 02.17$} & 82.93{\scriptsize $\pm\red 03.56$}& 81.38{\scriptsize $\pm\red 06.19$} & N/A* & 65.76{\scriptsize $\pm\red 07.45$}& 55.67{\scriptsize $\pm 06.61$} & 71.03{\scriptsize $\pm 02.12$}\\
    \midrule
    InGram \update{(d-DEq)} & {92.32}{\scriptsize $\pm 01.00$} & 83.71{\scriptsize $\pm\red 03.53$} & {90.82}{\scriptsize $\pm 01.84$} & {92.15}{\scriptsize $\pm 00.90$} & {61.44}{\scriptsize $\pm\red 09.84$} & {87.60}{\scriptsize $\pm 01.21$} & 54.79{\scriptsize $\pm\red 08.81$} & 67.84{\scriptsize $\pm 06.38$} \\
    Ultra \update{(DEq)}
    &  \textbf{98.67}{\scriptsize $\pm 00.05$}
    &  \textbf{98.22}{\scriptsize $\pm 00.10$} 
    &  \textbf{98.34}{\scriptsize $\pm 00.07$} 
    &  \textbf{97.00}{\scriptsize $\pm 00.12$} 
    &  \textbf{92.00}{\scriptsize $\pm 00.22$} 
    &  \textbf{95.92}{\scriptsize $\pm 00.13$} 
    &  \textbf{77.22}{\scriptsize $\pm 01.97$} 
    &  \textbf{85.48}{\scriptsize $\pm 01.15$} 
    \\
    %
    
    \OurOtherModel \update{~(DEq)} & 94.47{\scriptsize $\pm 00.60$} & 88.90{\scriptsize $\pm 02.06$} & 93.85{\scriptsize $\pm 00.36$} & 94.02{\scriptsize $\pm 00.74$} & 71.94{\scriptsize $\pm 07.37$} & \underline{91.47}{\scriptsize $\pm 00.62$} & \underline{71.53}{\scriptsize $\pm 04.78$} & \underline{80.53}{\scriptsize $\pm\red 07.96$} \\

    \OurNewModel \update{~(DEq)} & \underline{95.39}{\scriptsize $\pm 00.30$} & {81.57}{\scriptsize $\pm 03.17$} & \underline{97.66}{\scriptsize $\pm 00.19$} & \underline{95.03}{\scriptsize $\pm 00.44$} & \underline{86.60}{\scriptsize $\pm 00.59$} & 90.93{\scriptsize $\pm 00.24$} & 69.62{\scriptsize $\pm 01.10$} & 73.16{\scriptsize $\pm 00.82$} \\
\end{tabular}
\label{tab:pediatypes-node}
}
}
\vspace{-15pt}  %
}
\end{table}%

 In this section, we aim to answer one question: Can the general blueprint proposed in this work guide graph models to perform \ourtask over {\ourgrapheases} accurately? %
Specifically, we want to know {\bf Q1}: Can the proposed \OurNewModel framework successfully transform GNNs designed for homogeneous graphs into double equivariant models to perform \ourtask? {\bf Q2}: Can the proposed \OurOtherModel outperforms InGram on \ourtask in terms of both accuracy and robustness?

{\bf Dataset.}
To fully test the model's capability for \ourtask, we create a new dataset \textbf{\ourdata} (details in \Cref{appx:exp-pediatype}) by sampling from the OpenEA library~\citep{OpenEA}. The OpenEA library~\citep{OpenEA} contains multiple {\ourgrapheases} of relational databases from different domains on similar topics, such as DBPedia~\citep{lehmann2015dbpedia} in different languages (English, French and German), YAGO~\citep{rebele2016yago} and Wikidata~\citep{vrandevcic2014wikidata}. \ourdata includes pairs of {\ourgrapheases} such as English-to-French DBPedia (denoted as EN-FR), DBPedia-to-YAGO (denoted as DB-YG), etc.. 
In each graph, triplets are randomly divided into 80\% training, 10\% validation, and 10\% test. 
We then train and validate the model on one of the graphs (e.g., EN) and directly apply it to another graph (e.g., DE), which has completely new nodes and new relation types.


%
%

{\bf Baselines.} 
To the best of our knowledge, InGram~\citep{ingram} and Ultra~\citep{galkin2023towards} are the only two works capable of performing \ourtask without needing significant modification to the model. 
We also run RMPI~\citep{geng2023relational}, which is capable of performing \ourtask but requires extra context at test time.
%
%
In addition, we consider a variant of NBFNet~\citep{zhu2021neural} in which we modify its architecture following the approach in~\citep{ingram} to enable it on \ourtask. We also compare our models with message-passing GNNs, including GAT~\citep{velivckovic2017graph}, GIN~\citep{xu2018powerful}, GraphConv~\citep{Morris2019WL}, which treats the graph as a homogeneous graph by ignoring the relation types. 
%
For fair comparisons, we add distance features as in \Cref{eq:dssrepr} to increase the expressiveness of these GNNs. 
Additional baseline details are in \Cref{sec:appd-supexp}.

{\bf Evaluation.} We report the Hits@10 performances over 5 runs of different random seeds for all models on both the relation prediction task of $(i, ?, j)$ and the more traditional node prediction task of $(i, k, ?)$. For each task, we sample 50 negative triplets for each ground-truth positive target triplet during test evaluation by corrupting the relation type or the tail node respectively. %
{\em Further experiment details on baseline implementations, ablation studies, and other metrics (e.g., MRR, Hits@1) can be found in \Cref{sec:appd-supexp}.} 
%
%
%

%
{\bf Experiment results.}
\Cref{tab:pediatypes-relation} shows the results on the relation prediction task, and \Cref{tab:pediatypes-node} shows the node prediction task on \ourdata. In all scenarios across both tasks, our model \OurNewModel obtains competitive performance matching the best model, Ultra. The success of \OurNewModel, despite its simplicity, validates the effectiveness of our general blueprint in transforming arbitrary homogeneous GNNs into double equivariant models. 
Furthermore, \OurOtherModel, with its straightforward modification guided by the theory of double equivariance, significantly outperforms the original InGram across all datasets.
This directly corroborates our theoretical predictions in~\Cref{sec:theory} that a model directly producing double equivariant representations is more stable than positional embeddings, which are only double equivariant in expectation. 
\section{Experiment on Negative Transfers in Meta-Learning}  \label{sec:negative-transfer}

As we have empirically demonstrated the effectiveness of the double equivariance framework in~\Cref{sec:exp}, we now turn to a crucial question: {\em are existing double equivariant models ready to be knowledge graph foundation models~\cite{galkin2023towards,mao2024graph}?} Foundation models, as coined by~\citep{bommasani2021opportunities}, are models pre-trained on massive-scale data and are capable of adaption to a wide range of downstream applications.
A fundamental requirement of such models is the ability to effectively learn from pre-training data spanning multiple domains, while overcoming the challenge of negative transfer~\cite{zhang2022survey,wang2022mitigating}. Negative transfer occurs when the inclusion of data from certain domains in the training set somehow impairs the model's transfer learning performance. Another critical requirement is adherence to strong data scaling law~\citep{kaplan2020scaling,chen2024uncovering,huang2024prodigy}, meaning that the model's performance should consistently improve as the training data scale increases. Hence, we investigate the following questions {\bf Q1:} Are there specific {\ourgraph} that consistently cause negative transfer in existing double equivariance models? {\bf Q2:} How well do existing double equivariant models conform to the data scaling law?

\paragraph{Dataset and experiment setup} To this end, we introduce another dataset {\bf \ourotherdata}, and design a meta-learning experiment. \ourotherdata is constructed from WikiData-5M~\citep{wang2021kepler} by grouping the relations into 11 different topics, or domains, such as infrastructure, science, sport, locations, etc.. 
Each group of relations induces a knowledge graph containing only the relations from that group, allowing each knowledge graph to be viewed as representing a specific domain. 
For the experiment, we select a subset of these knowledge graphs for training, while holding out others for testing. Specifically, we train multiple model instances with an increasing number of knowledge graphs mixed into the training process to investigate whether the models adhere to data scaling laws. Additionally, we experiment with different combinations of training knowledge graphs to determine if a certain knowledge graph would consistently trigger negative transfer. \Cref{sec:appd-supexp} contains further dataset and experiment details.

\begin{table}[t]
    \centering
    \caption{Average training and zero-shot Hits@1 performance (in \%) of Ultra, \OurOtherModel, and \OurNewModel on a meta-learning scenario in \ourotherdata. (1) {\bf\em Loc} is one KG that we identified to consistently trigger negative transfer effects among all models. {\bf Mixing {\em Loc} in training data consistently diminishes both the training and test performance across all double equivariant models.} We show in parenthesis the relative performance difference comparing the training combination with {\bf\em Loc} and a similar combination without {\bf\em Loc}. (2) {\bf All models exhibit limited performance improvement as the number of training KG domains increases.}}
    \label{tab:negative-transfer-node}
    {
    \subcaption{Relation prediction $(i, ?, j)$}
    \resizebox{0.99\linewidth}{!}{
    \begin{tabular}{l l l l l l l l}
        \toprule
        \multirow{2.5}{*}{Meta-Learning Training KGs} & \multirow{2.5}{*}{Neg. Transfer?} & \multicolumn{3}{c}{Training Hits@1} & \multicolumn{3}{c}{Average Zero-shot Hits@1} \\
        \cmidrule(lr){3-5} \cmidrule(lr){6-8}
        & & Ultra & \OurOtherModel & \OurNewModel & Ultra & \OurOtherModel & \OurNewModel \\
        \midrule
        Infra & & 91.82 & 73.94 & 75.15 & 77.22 & 62.07 & 80.70 \\
        {\bf Loc} & \checkmark & 28.31 {\scriptsize$(-69.17\%)$} & 18.53 {\scriptsize$(-74.94\%)$} & 81.10 {\scriptsize $(+7.92\%)$} & 13.06 {\scriptsize $(-83.09\%)$} & 19.64 {\scriptsize $(-68.36\%)$} & 85.16 {\scriptsize $(+5.53\%)$}	 \\
        \midrule
        Infra + Sci & & 92.79 & 58.44 & 83.98 & 80.07 & 60.19 & 85.87 \\
        Infra + {\bf Loc} & \checkmark & 82.34 {\scriptsize$(-11.26\%)$} & 59.64 {\scriptsize$(+2.05\%)$} & 73.17 {\scriptsize$(-12.87\%)$} & 48.69 {\scriptsize $(-39.19\%)$} & 47.36 {\scriptsize $(-21.32\%)$} & 74.27 {\scriptsize $(-13.42\%)$} \\
        \midrule
        Infra + Sci + Sport & & 92.03 & 87.03 & 86.62 & 79.28 & 67.48 & 80.88 \\
        Infra + Sci + {\bf Loc} & \checkmark & 83.65 {\scriptsize $(-9.11\%)$} & 36.56 {\scriptsize $(-57.99\%)$} & 84.21 {\scriptsize $(-2.78\%)$} & 22.42 {\scriptsize $(-71.72\%)$} & 12.28 {\scriptsize $(-81.80\%)$} & 75.47 {\scriptsize $(-6.69\%)$} \\
        \midrule
        Infra + Sci + Sport + Tax & & 91.50 & 57.98 & 95.63 & 76.67 & 30.25 & 80.70 \\
        Infra + Sci + Sport + {\bf Loc} & \checkmark & 91.84 {\scriptsize $(+0.37\%)$} & 31.81 {\scriptsize $(-45.14\%)$} & 91.72 {\scriptsize $(-4.08\%)$} & 69.89 {\scriptsize $(-8.84\%)$} & 20.25 {\scriptsize $(-33.06\%)$} & 79.53 {\scriptsize $(-1.44\%$} \\
        \bottomrule
    \end{tabular}
    }
    }
    \vspace{5pt} %
    {
    \subcaption{Node prediction $(i, r, ?)$}
    \resizebox{0.99\linewidth}{!}{
    \begin{tabular}{l l l l l l l l}
        \toprule
        \multirow{2.5}{*}{Meta-Learning Training KGs} & \multirow{2.5}{*}{Neg. Transfer?} & \multicolumn{3}{c}{Training Hits@1} & \multicolumn{3}{c}{Average Zero-shot Hits@1} \\
        \cmidrule(lr){3-5} \cmidrule(lr){6-8}
        & & Ultra & \OurOtherModel & \OurNewModel & Ultra & \OurOtherModel & \OurNewModel \\
        \midrule
        Infra & & 89.93 & 82.65  & 73.53 & 67.19 & 47.22 & 39.79 \\
        {\bf Loc} & \checkmark & 36.92 {\scriptsize $(-58.95\%)$} & 84.93 {\scriptsize $(+2.76\%)$} & 42.31 {\scriptsize $(-42.46\%)$} & 43.04 {\scriptsize $(-35.94\%)$} & 38.47 {\scriptsize $(-18.53\%)$} & 27.15 {\scriptsize $(-31.77\%)$} \\
        \midrule
        Infra + Sci & & 87.44 & 79.80 & 81.73 & 67.49 & 47.67 & 41.92\\
        Infra + {\bf Loc} & \checkmark & 64.19 {\scriptsize $(-26.59\%)$} & 88.35 {\scriptsize $(+10.71\%)$} & 59.60 {\scriptsize $(-27.08\%)$} & 62.98{\scriptsize $(-36.23\%)$} & 42.18 {\scriptsize $(-11.52\%)$} & 38.50{ \scriptsize $(-8.16\%)$} \\
        \midrule
        Infra + Sci + Sport & & 68.40 & 66.82 & 66.23  & 66.97 & 53.92 & 41.34  \\
        Infra + Sci + {\bf Loc} & \checkmark & 64.09 {\scriptsize $(-6.30\%)$} & 87.14 {\scriptsize $(+30.41\%)$} & 66.32 {\scriptsize $(+0.14\%)$} & 63.08 {\scriptsize $(-5.81\%)$} & 41.20 {\scriptsize $(-23.59\%)$} & 37.63 {\scriptsize $(-8.97\%)$}	 \\
        \midrule
        Infra + Sci + Sport + Tax & & 70.26 & 43.87 & 71.51 & 66.90 & 34.16 & 42.45 \\
        Infra + Sci + Sport + {\bf Loc} & \checkmark & 68.71 {\scriptsize $(-2.21\%)$} & 74.28 {\scriptsize $(+69.32\%)$} & 64.55 {\scriptsize $(-9.73\%)$} & 65.88 {\scriptsize $(-1.52\%)$} & 44.33 {\scriptsize $(+29.77\%)$} & 41.34 {\scriptsize $(-2.61\%)$} \\
        \bottomrule
    \end{tabular}
    }}

\end{table}

\paragraph{Experiment results} \Cref{tab:negative-transfer-node} presents the main findings of our experiment. Here we trained the models on a progressively increasing combination of training KGs from the domains of Infrastructure, Science, Sport, Taxonomies, and Locations (abbreviated as Infra, Sci, Sport, Tax, \& Loc respectively). The models were then evaluated on four held-out test domains: art, award, education, and health care, with results averaged across these test domains. Each model was trained for 3 different random seeds for each setting, and we report the Hits@1 performance for both node prediction $(i, k, ?)$ and relation prediction $(i, ?, j)$ tasks, using 50 negative samples. Additionally, we report the models' accuracy on both the training graphs and the inference graphs.

We first note that the {\ourgraph} from the domain Loc is indeed one KG that consistently triggers negative transfer across all models. Specifically, whenever Loc is included in the training data mix (e.g. Infra + Sci + Loc), the models demonstrate worse zero-shot test performance compared to training with other combinations of KGs with the same number of domains (e.g. Infra + Sci + Sport). Additionally, in most cases, the models also exhibit a drop in training accuracy when Loc is included. This suggests that the Loc KG may contain patterns that conflict with those in other KGs, causing models with limited expressivity to struggle with jointly learning these conflicting patterns. Therefore, in response to Q1, all evaluated models are indeed susceptible to negative transfer, with Loc serving as a concrete example.

Our second observation concerns the data scaling law. We find that no model consistently exhibits improved zero-shot test performance as the number of training KG domains increases, whether the Loc KG, which causes negative transfer, is included in the training mix or not. Notably, DEq-InGram unexpectedly shows a sharp decline in performance when the number of training domains increases from 3 to 4, while Ultra and ISDEA+ generally display saturated accuracy. Therefore, in response to Q2, our experiments indicate that all existing double equivariant models have limited data scaling capabilities. Addressing this limitation remains an important direction for future research.


\section{Conclusion}
This work formally introduced the concept of {\em double equivariant structural representations} and {\em distributionally double equivariant positional embedding} as a unifying theoretical framework for fully inductive link prediction in knowledge graphs. We demonstrated that this framework underpins the effectiveness of existing models capable of such a task, and provided a solid blueprint for designing future models. Our empirical studies showed that the proposed framework significantly enhances model performance, as seen with DEq-InGram, and offers a systematic approach to constructing double equivariant models, such as with ISDEA+.

Despite these advancements, our experiments also highlighted that the current double equivariant models are hindered by negative transfer effects and poor data scaling behavior when jointly learning from data over multiple domains in a meta-learning setting. Our proposed benchmark and the experiment results thus presents opportunities for future work to address these gaps to advance toward realizing true knowledge graph foundation models.

\nocite{cloudbank}

\bibliographystyle{unsrtnat}
\bibliography{main}

\newpage
\appendix

\section{Detailed Model Design for \OurNewModel}
\label{appx:detailed_design}
Here we provide further details on the architectural design of \OurNewModel.

\subsection{Implementation Details of \OurNewModel}

We use GNN layers for constructing $L_1,L_2$.
Since most-expressive pairwise representations are computationally expensive, we trade-off expressivity in the implementation of \Cref{eq:dsslin} for speed and memory by using node representation GNN layers \citep{xu2018powerful,velivckovic2017graph,Morris2019WL}. Specifically, for a \ourgraphease $\adj$ with 
number of nodes and relations $\nodesize,\relsize$, 
at each iteration $t = 1, ..., T$, for each relation type $k\in [\relsize]$, all nodes $i \in [\nodesize]$ are associated with two learned vectors $h_{i,k}^{(t)} \in \sR^{d_t}, h_{i,\neg k}^{(t)} \in \sR^{d_t}, d_t \geq 1$.
If there are
no node attributes, we initialize $\forall k \in [\relsize],h_{i,k}^{(0)} = h_{i,\neg k}^{(0)} = \mathbbm{1}$. %
Then we recursively compute the update, 
$\forall i \in [\nodesize], \forall k \in [\relsize]$,
\begin{equation*}
\scriptstyle
    \begin{aligned}
    \scriptstyle
         h_{i, k}^{(t + 1)} & \scriptstyle= \text{GNN}^{(t)}_1 \Big( h_{i, k}^{(t)}, \bigldcbrace h_{j, k}^{(t)} \big\vert j \in \mathcal{N}_k(i) \bigrdcbrace \Big) , \quad \quad\quad \ \ h_{i, \neg k}^{(t + 1)} = \text{GNN}^{(t)}_2 \Big( h_{i, \neg k}^{(t)}, \bigldcbrace h_{j, \neg k}^{(t)} \big\vert j \in \bigcup\limits_{k' \neq k}\mathcal{N}_k(i) \bigrdcbrace \Big),  \quad\ \ \ \text{if } t = 0, \\
        \scriptstyle
         h_{i, k}^{(t + 1)} & \scriptstyle= \text{GNN}^{(t)}_1 \Big( h_{i, k}^{(t)}, \bigldcbrace h_{j, k}^{(t)} \big\vert j \in \bigcup\limits_{k' \in [\relsizewild]}\mathcal{N}_{k'}(i) \bigrdcbrace \Big) , h_{i, \neg k}^{(t + 1)} = \text{GNN}^{(t)}_2 \Big( h_{i, \neg k}^{(t)}, \bigldcbrace h_{j, \neg k}^{(t)} \big\vert j \in \bigcup\limits_{k' \in [\relsizewild]}\mathcal{N}_{k'}(i) \bigrdcbrace \Big) , \ \text{if } t > 0,
    \end{aligned}
\end{equation*}%
where $\text{GNN}^{(t)}_1$ and $\text{GNN}^{(t)}_2$ denote two GNN layers %
and $\cN_k(i) \coloneqq \left\{ j \middle| \adj_{j, k, i} = 1 \right\}$ denotes the neighborhood set of node $i$ with relation $k$ in the unattributed graph $A^{(k)}$.
To get the final representation $X_{i,k}$ for the node $i$ with respect to relation $k$. We define $h_{i,k} = h^{(0)}_{i,k}\Big\|h^{(1)}_{i,k}\Big\|\cdots \Big\|h^{(T)}_{i,k}$, $h_{i,\neg k} = h^{(0)}_{i,\neg k}\Big\|h^{(1)}_{i,\neg k}\Big\|\cdots \Big\|h^{(T)}_{i,\neg k}$, and combine the two embeddings as illustrated in \Cref{eq:dsslin},
\begin{equation}
\label{eq:finalMLP}
    X_{i,k} = \text{MLP}_1(h_{i,k}) + \text{MLP}_2(h_{i,\neg k}), \forall i\in[\nodesize], \forall k\in [\relsize],
\end{equation}
where $\text{MLP}_1, \text{MLP}_2$ are two multi-layer perceptrons, $\|$ as the concatenation operation.

As shown by previous studies, structural node representations %
are not most expressive for link prediction in unattributed graphs~\citep{srinivasan2020on,you2019position}.
Hence, we concatenate $i$ and $j$ (\ourequiv) node representations with the shortest distance between $i$ and $j$ in the observed graph as our triplet representations (appending distances is also adopted in the representations of prior work~\citep{teru2020inductive,galkin2021nodepiece}).
Finally, we obtain the triplet representation,
\begin{equation}
\label{eq:dssrepr}
\Gamma_\text{\OurNewModel}((i,k,j), \adjwild) = \Big( X_{i, k} \Big\| X_{j, k} \Big\| d(i, j) \Big\| d(j, i) \Big), \forall (i, k, j) \in [\nodesize]\times[\relsize]\times[\nodesize],
\end{equation}where we denote $d(i, j)$ as the length of shortest path from $i$ to $j$ without considering $(i,k,j)$.
Since our graph is directed, we concatenate them in both directions. 

%
%
%
 
As in %
\citep{yang2015embedding,schlichtkrull2018modeling,zhu2021neural}, we use negative sampling in our training with the difference that we account for both predicting missing nodes and relation types~(\Cref{def:task-double}).
Specifically, for each positive training triplet $(i,k,j)$ such that $\adj\suptr_{i,k,j}=1$, we first randomly corrupt either the head or the tail $n_\text{nd}$ times to generate the negative (node) %
examples $(i,k,j')$.
Additionally, %
we also want 
our model to learn the correct relation type $(i,?,j)$ between a pair of nodes.  %
Thus, we corrupt relation $n_\text{rl}$ times to generate negative (relation) examples $(i,k',j)$.
In our training, $n_\text{nd} = n_\text{rl} = 2$; while in evaluation, $n_\text{nd} = 50, n_\text{rl} = 0$ for node evaluation, and $n_\text{nd} = 0, n_\text{rl} = 50$ for relation evaluation.
Following \citep{schlichtkrull2018modeling}, we use cross-entropy loss 
to encourage the model to
score positive examples higher than corresponding negative examples:
\begin{align}
\label{eq:loss}
\mathcal{L}\! &=\! - \!\!\sum_{(i, k, j) \in \mathcal{S}} \left( \log \left( \trifunc{\Gamma}((i,k,j), \adj\suptr) \right) \right. \nonumber\\
&\left. + \frac{1}{n_\text{nd} + n_\text{rl}} \sum_{p = 1}^{n_\text{nd} + n_\text{rl}} \log \left( 1 -\trifunc{\Gamma}\left( \left(i'_p, k'_p, j'_p \right), \adj\suptr \right) \right) \right),
\end{align}
where $S = \left\{ (i, k, j) \middle| \adj\suptr_{i, k, j} = 1 \right\}$, and $\left( i'_p, k'_p, j'_p \right)$ are the $p$-th negative node or relation example corresponding to $(i, k, j)$.

We choose GCN~\citep{Kipf2016} as our GNN kernel. In the implementation, we also follow SIGN~\citep{sign_icml_grl2020} to not have activation function between GNN layers to make the method faster and more scalable, and only have activation functions in the final MLPs in \Cref{eq:finalMLP}. In the message passing scheme, we only use DSS-GNN in the first layer, while using the whole graph adjacency matrix in following layer updates. This procedure guarantees the double equivariant property of \OurNewModel and increases the expressiveness of \OurNewModel to capture more diverse relation paths. In training, we carefully design the training batch, so that the each gradient step considers only one relation vs all others. Another way to improve the computation complexity is via parallelization.

\textbf{Time Complexity.} 
For each layer of \OurNewModel, it can be treated as running 2 unattributed GNN times on the \ourgraphease, thus time cost is roughly $2$ times of adopted GNN. In our experiment, we use node representation GNNs (e.g., GIN \citep{xu2018powerful}, GAT \citep{velivckovic2017graph}, GraphConv~\citep{Morris2019WL}) as our GNN architecture, thus the complexity is $\mathcal{O}(L(|\cS|d+|\cV|d^2))$ where $L$ is the number of layers, $d$ is the maximum size of hidden layers, $|\cV|$ is number of nodes and $|\cS|$ is number of fact triplets (number of edges) in the \ourgraphease. Besides, for both positive and negative samples $(i,k,j)$, our method requires the shortest distance between any two nodes without considering $(i,k,j)$, which can be achieved from the Dijkstra or Floyd algorithm.


%
\section{Connection to Double Equivariant Logical Reasoning}
\label{appx:logic}

In what follows, we follow the literature and connect link prediction in \ourgraph to logical induction~\citep{teru2020inductive,zhu2022neural,qiu2023logical}. Existing logical induction requires all involved relations to be observed at least once, thus, such logical reasoning can not generalize to new relation types.
We propose the Universally Quantified Entity and Relation (UQER) Horn clause, a double equivariant extension of conventional logical reasoning, which is capable of generalizing to new relation types, and show that the double invariant triplet representation in \Cref{def:rel-graph-iso} is capable of encoding such set of UQER Horn Clauses.

\begin{definition}[Universally Quantified Entity and Relation (UQER) Horn clause]
\label{def:uqer}
An UQER Horn clause involving $M$ nodes and $K$ relations is defined by an indicator tensor $\tB \in \{0,1\}^{M\times K \times M}$:
\begin{equation}
\label{eq:uqer}
\begin{gathered}
    \forall E_1 \in \cV, \left( \forall E_u \in \cV \setminus \left\{ E_1, \ldots, E_{u-1} \right\} \right)_{u=2}^M \,,
    \forall C_1 \in \cR, \left( \forall C_c \in \cR \setminus \left\{ C_1, \ldots, C_{c-1} \right\} \right)_{c=2}^K \,, \\
    \bigwedge_{\substack{u,u' = 1, \ldots, M, c = 1, \ldots, K, \\ \tB_{u,c,u'} = 1 }} (E_u,C_c,E_{u'}) \implies (E_1,C_1,E_h) ,\:
\end{gathered}
\end{equation}
for any node set $\cV$ and relation set $\cR$ with number of nodes and relations $\nodesize,\relsize$ s.t. $\nodesize \geq M, \relsize \geq K$,
$h \in \{1,2\}$ (where $h=1$ indicates a self-loop relation or a relational node attribute), where if $M > h$, $\forall u \in \{h + 1, \ldots, M \}$,
$\sum_{u' = 1}^M \sum_{c = 1}^K \tB_{u,c,u'} + \tB_{u',c,u} \geq 1$, and if $K \geq 2$, $\forall c \in \{2, \ldots, K \}$, $\sum_{u = 1}^M \sum_{u' = 1}^M \tB_{u,c,u'} + \tB_{u',c,u} \geq 1$ (every variable should appear at least once in the formula).
\end{definition}

Note that our definition of UQER Horn clauses (\Cref{def:uqer}) is a generalization of the First Order Logic (FOL) clauses in \cite{yang2017differentiable,meilicke2018fine,sadeghian2019drum,teru2020inductive} such that the relations in the Horn clauses are also universally quantified rather than predefined constants.
UQER can be used to predict new relations in the test \ourgraph with {\em pattern matching}, i.e., if the left-hand-side (condition) of a UQER can be satisfied in the test \ourgraph, then the right-hand-side (implication) triplet should be present.
In \Cref{fig:example-uqer}, we illustrate two examples using UQER to predict new relations at test time.

We now connect our \ourequiv representations (\Cref{def:trip-repr}) with the UQER Horn clauses.
\begin{restatable}{theorem}{thmtwo}
\label{thm:repr-and-logic} 
For any UQER Horn clause defined by $\tB \in \{ 0, 1 \}^{M \times K \times M}$ (\Cref{def:uqer}), there exists a \ourinv triplet predictor $\trifunc{\Gamma}: \cup_{N=1}^\infty\cup_{R=2}^\infty([N]\times[R]\times[N]) \times \sA \rightarrow \{ 0, 1 \} $ (\Cref{def:trip-repr}), such that %
for any set of truth statements $\cS \subseteq \cV \times \cR \times \cV$ and their equivalent tensor representation $\adj \in \sA$ (where $\adj_{i,k,j} = 1 \text{ iff } (v_i,r_k,v_j \in \cS$), it satisfies $\trifunc{\Gamma}((i,k,j), \adj) = 1 \text{ iff } (i,k,j) \in \cS'$, where $\cS' = \big\{ \left( i,k,j \right) \,\big\vert\, \forall \left( i,k,j \right), %
\text{such that} \left( E_1, C_1, E_2 \right) = \left(v_i,r_k,v_j \right) \in \cV \times \cR \times \cV, \ \exists^{M-2} E_3,...,E_M \in \cV \setminus \left\{ E_1, E_2 \right\},  \exists^{K-1} C_2,...,C_K \in \cR \setminus \left\{ C_1 \right\}, \text{ where } \forall (u,c,u')\in [M]\times [K]\times [M], \tB_{u,c,u'} = 1 \Rightarrow \left(E_u, C_c, E_{u'} \right) \in \cS \big\}$ is the set of true statements induced by modus ponens by the truth statements $\cS$ and the UQER Horn clause, where the existential quantifier $\exists^k$ means exists at least $k$ distinct values.
\end{restatable}

\begin{figure}
\begin{minipage}{0.49\textwidth}
\centering
\includegraphics[width=\linewidth]{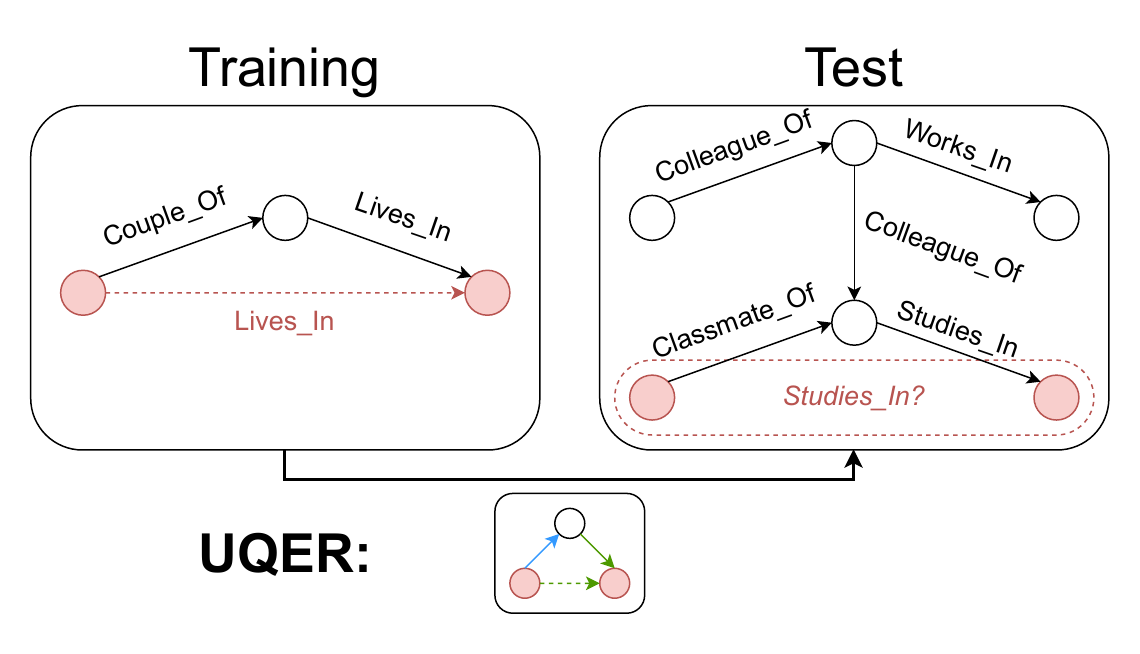}
\subcaption{{\bf A Simple UQER Application}}
\label{fig:uqer-from-cond}
\end{minipage}
\hfill
\begin{minipage}{0.49\textwidth}
\centering
\includegraphics[width=\linewidth]{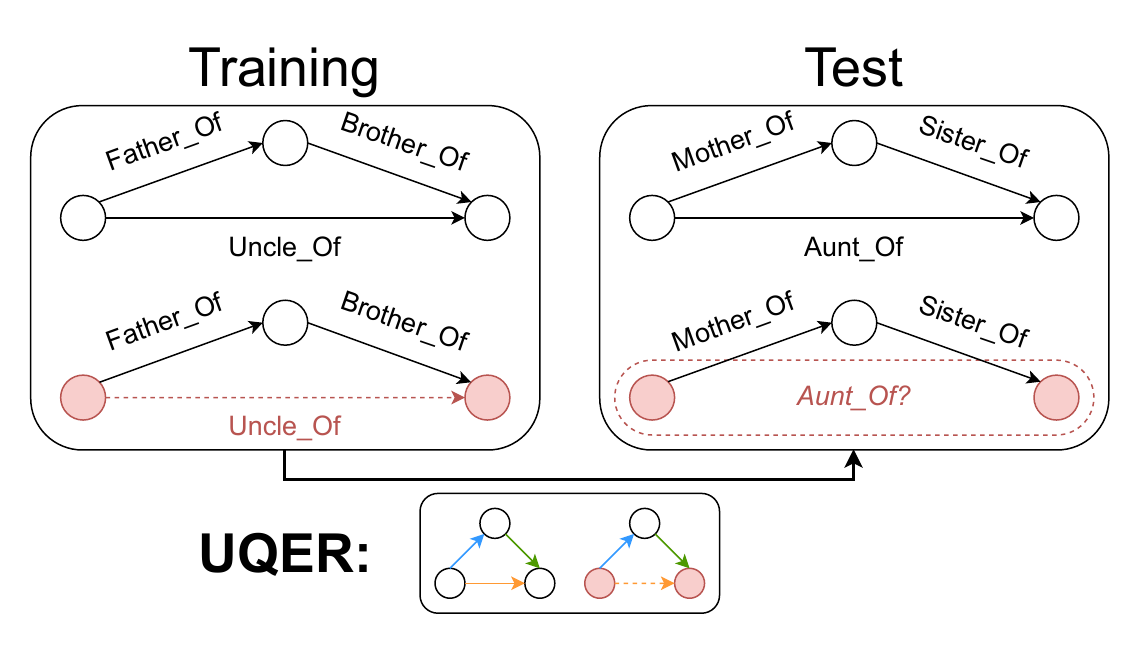}
\subcaption{{\bf A Complex UQER Application}}
\end{minipage}
\caption{
(a) The UQER (bottom) learned from training can be used to predict missing new relation ``Studies{\_}In'' in red since an assignment of left-hand-side of the UQER $\left( E_1, \text{Classmate\_Of}, E_3 \right) \land \left( E_3, \text{Studies\_In}, E_2 \right)$ is satisfied in test.
(b) UQER can contain disconnected components, giving more freedom to its application. For example, the UQER (bottom) can be learned from training to repeat arbitrary logical chain, which makes it possible to deal with new female relations at test time and will predict ``Aunt{\_}Of'' in test just as ``Uncle{\_}Of'' (red) in training.
}
\label{fig:example-uqer}
\end{figure}

The full proof is in \Cref{sec:proof}, showing how the universal quantification in \Cref{def:uqer} is a double invariant predictor.  
\section{Proofs} \label{appx:extraThmProofs}
\label{sec:proof}

\thmone*
\begin{proof}
    ($\Rightarrow$) For any \ourgraph $\adj\in\sA$ with number of nodes and relations $\nodesize, \relsize$, $\Gamma_\text{triplet}:  \cup_{N=1}^\infty\cup_{R=2}^\infty([N]\times[R]\times [N]) \times \sA\rightarrow \mathbb{R}^{d}, d\geq 1$ is a double invariant triplet representation as in \Cref{def:trip-repr}. 
    Using the double invariant triplet representation, we can define a function $\Gamma_\text{graph}:\sA\rightarrow \cup_{N=1}^\infty\cup_{R=2}^\infty\sR^{N\times R\times N\times d}$ such that $\forall (i,k,j)\in [\nodesize]\times[\relsize]\times[\nodesize]$, $(\Gamma_\text{graph}(\adj))_{i,k,j,:} = \Gamma_\text{triplet}((i,k,j),\adj)$. Then $\forall\phi\in \sS_\nodesize,\forall\tau\in\sS_\relsize$, $(\Gamma_\text{graph}(\phi\circ\tau\circ\adj))_{\phi\circ i, \tau\circ k,\phi\circ j,:} = \Gamma_\text{triplet}((\phi\circ i,\tau\circ k,\phi \circ j),\phi\circ\tau\circ\adj)$. We know $\Gamma_\text{triplet}((i,k,j),\adj) = \Gamma_\text{triplet}((\phi\circ i,\tau\circ k,\phi \circ j),\phi\circ\tau\circ\adj)$. Thus we conclude, $\forall \phi\in \sS_\nodesize,\forall \tau\in \sS_\relsize,\forall (i,k,j)\in [\nodesize]\times[\relsize]\times[\nodesize]$, $(\phi\circ\tau\circ \Gamma_\text{graph}(\adj))_{\phi\circ i,\tau\circ k,\phi\circ j,:} = (\Gamma_\text{graph}(\adj))_{i,k,j,:}= \Gamma_\text{triplet}((i,k,j),\adj)=\Gamma_\text{triplet}((\phi\circ i,\tau\circ k,\phi \circ j),\phi\circ\tau\circ\adj)=(\Gamma_\text{graph}(\phi\circ\tau\circ\adj))_{\phi\circ i, \tau\circ k,\phi\circ j,:}$. In conclusion, we show that $\phi\circ\tau\circ\Gamma_\text{graph}(\adj) = \Gamma_\text{graph}(\phi\circ\tau\circ\adj)$, which proves the constructed $\Gamma_\text{graph}$ is a double equivariant representation as in \Cref{def:rel-graph-repr}.

    ($\Leftarrow$) For any \ourgraph $\adj\in\sA$ with number of nodes and relations $\nodesize,\relsize$, assume $\Gamma_\text{graph}:\sA\rightarrow \cup_{N=1}^\infty\cup_{R=2}^\infty \sR^{N\times R\times N\times d}$ is a double equivariant representation as \Cref{def:rel-graph-repr}. Since $\Gamma_\text{graph}(\phi\circ\tau\circ \adj) = \phi\circ\tau\circ \Gamma_\text{graph}(\adj)$, then $\forall (i,k,j)\in [\nodesize]\times[\relsize]\times[\nodesize]$,  $(\Gamma_\text{graph}(\phi\circ\tau\circ \adj))_{\phi\circ i, \tau\circ k, \phi\circ j} = (\phi\circ\tau\circ \Gamma_\text{graph}(\adj))_{\phi\circ i, \tau\circ k, \phi\circ j} = (\Gamma_\text{graph}(\adj))_{i,k,j}$. Then we can define $\Gamma_\text{triplet}:  \cup_{N=1}^\infty\cup_{R=2}^\infty([N]\times[R]\times[N]) \times \sA\rightarrow \mathbb{R}^{d}, d\geq 1$, such that $\forall (i,k,j)\in [\nodesize]\times[\relsize]\times[\nodesize]$,  $\Gamma_\text{triplet}((i,k,j),\adj) = (\Gamma_\text{graph}(\adj))_{i,k,j}$. It is clear that $\Gamma_\text{triplet}((i,k,j),\adj) = (\Gamma_\text{graph}(\adj))_{i,k,j} = (\Gamma_\text{graph}(\phi\circ\tau\circ \adj))_{\phi\circ i, \tau\circ k, \phi\circ j} = \Gamma_\text{triplet}((\phi\circ i,\tau \circ k,\phi \circ j),\phi\circ\tau\circ\adj)$. Thus, we show $\Gamma_\text{triplet}$ is a double invariant triplet representation as in \Cref{def:trip-repr}.
\end{proof}

\thmthree*

\begin{proof}
Based on \Cref{def:pos-trip-repr}, for any \ourgraphease $\adj \in \sA$ with number of nodes and relations $\nodesize,\relsize$, the distributionally double equivariant positional embeddings of $\adj$ are defined as joint samples of random variables $\tZ|\adj \sim p(\tZ|\adj)$,  where the tensor $\tZ $ is defined as $\tZ_{i,k,j}\in\sR^d, d\geq 1, \forall (i,k,j) \in [\nodesize]\times[\relsize]\times[\nodesize]$, where we say $p(\tZ|\adj)$ is a double equivariant probability distribution on $\adj$ defined as $\forall \phi\in \sS_\nodesize,\forall \tau\in \sS_\relsize,  p(\tZ |\adj) = p(\phi \circ \tau \circ  \tZ| \phi \circ \tau \circ \adj)$. 

The tensor $\tZ$ is defined as $\tZ_{i,k,j}\in \sR^d, \forall (i,k,j)\in [\nodesize]\times[\relsize]\times[\nodesize]$, thus $\tZ\in \sR^{\nodesize\times\relsize\times\nodesize\times d}$. So we can consider $\mathbb{E}_{p(\tZ |\adj)}[\tZ |\adj]$ as a function on $\adj$, and output a representation in $\sR^{\nodesize\times\relsize\times\nodesize\times d}$. 
Since $\forall \phi\in \sS_\nodesize,\forall \tau\in \sS_\relsize,  p(\tZ |\adj) = p(\phi \circ \tau \circ  \tZ| \phi \circ \tau \circ \adj)$, it is clear to have $\forall \phi\in \sS_\nodesize,\forall \tau\in \sS_\relsize, \phi \circ \tau \circ\mathbb{E}_{p(\tZ |\adj)}[\tZ |\adj] = \phi \circ \tau \circ \int \vz p(\tZ=\vz|\adj) d\vz =  \int \phi \circ \tau \circ \vz p(\tZ=\vz|\adj) d\vz = \int \phi \circ \tau \circ \vz p(\phi\circ\tau\circ \tZ=\phi \circ \tau \circ \vz|\phi \circ \tau \circ \adj) d(\phi \circ \tau \circ\vz) = \mathbb{E}_{p(\phi \circ \tau \circ  \tZ| \phi \circ \tau \circ \adj)}[\phi \circ \tau \circ  \tZ| \phi \circ \tau \circ \adj]$. Since the permutation $\phi,\tau$ only changes the ordering of the output representation element-wise, we can interchange the permutations with the integral. 

Finally, for any \ourgraphease $\adj \in \sA$ with  number of nodes and relations $\nodesize, \relsize$, we can define $\grafunc{\Gamma}(\adj) : \sA \rightarrow \cup_{N=1}^\infty\cup_{R=2}^\infty \sR^{N \times R \times N \times d}, d \geq 1$ such that $\grafunc{\Gamma}(\adj) := \mathbb{E}_{p(\tZ |\adj)}[\tZ |\adj]$. And we can derive $\phi \circ \tau \circ \grafunc{\Gamma}(\adj) = \phi \circ \tau \circ\mathbb{E}_{p(\tZ |\adj)}[\tZ |\adj] 
 =  \mathbb{E}_{p(\phi \circ \tau \circ\tZ |\phi \circ \tau \circ\adj)}[\phi \circ \tau \circ\tZ |\phi \circ \tau \circ\adj] = \grafunc{\Gamma}(\phi \circ \tau \circ \adj)$. Thus, $\grafunc{\Gamma}(\adj) := \mathbb{E}_{p(\tZ |\adj)}[\tZ |\adj]$ is a \ourequiv \ \ourgraphease representation as per \Cref{def:rel-graph-repr}. 
\end{proof}


\lemultra*
\begin{proof}
Consider an arbitrary knowledge graph $\gG = (\gV, \gR, \gV)$ with the adjacency matrix $\adj$, and its isomorphic graph with adjacency matrix $\adj'$ induced by some arbitrary entity and relation permutations $\phi \in S_N$ and $\tau \in S_R$, i.e., $\adj' = \phi \circ \tau \circ \adj$. Fix a triplet $(h, r, t) \in \gV \times \gR \times \gV$, and let $(h', r', t') = (\phi \circ h, \tau \circ r, \phi \circ t)$.

\begin{enumerate}
\item 
Given any graph $\adj^{*}$, 
an NBFNet produces \textit{pairwise structural node representations} equivariant to node permutations if it initializes the node features by setting all-ones vector to the head node $h^{*}$ of the query triplet $(h^{*}, r^{*}, ?)$, and zero vectors to every other node~\citep{zhu2021neural}. That is, for any query $(h^{*}, r^{*}, ?)$ and any valid node permutation $\phi^{*} \in S_N$, we have 
\begin{align*}
    \Gamma\trip_{\text{NBFNet}} ((h^{*}, r^{*}, t^{*}), \adj^{*}) = \Gamma\trip_{\text{NBFNet}} ((\phi^{*} \circ h^{*}, r^{*}, \phi^{*} \circ t^{*}), \phi^{*} \circ \adj^{*}) , \forall t^{*} \in \gV^{*}.
\end{align*}

\item
The relation graph $\gG\rel$ constructed by ULTRA is invariant to the node permutation $\phi$. This is because changing entity labels does not affect the $\texttt{t2h}, \texttt{h2h}, \texttt{h2t}, \texttt{t2t}$ interactions between relations.

\item
ULTRA invokes the first instance of NBFNet on the relation graph $\gG\rel$. Hence, following 1., we can conclude that the output triplet representation of this NBFNet instance on $\gG\rel$ is invariant to the relation permutation $\tau$, because each node in $\gG\rel$ corresponds to a relation in $\gG$, and so the relation permutation $\tau$ corresponds to some node permutation on $\gG\rel$. Namely, $\forall \hat{r} \in \gR, k \in \{ \texttt{t2h}, \texttt{h2h}, \texttt{h2t}, \texttt{t2t} \}$,
\begin{align*}
    \Gamma\trip_{\text{NBFNet}} ((r, k, \hat{r}), \adj\rel) = \Gamma\trip_{\text{NBFNet}} ((\tau \circ r, k, \tau \circ \hat{r}), \tau \circ \adj\rel) .
\end{align*}
Consequently, the output relation embedding $\mR_r$ is equivariant to the relation permutation:
\begin{align*}
    \mR_r[\hat{r}] = \mR_{\tau \circ r}[\tau \circ \hat{r}], \forall \hat{r} \in \gR .
\end{align*}
Together with 2., we know that $\mR_r$ is \textit{invariant} to node permutations, i.e., $\phi \circ \mR_r = \mR_r$. Hence, $\mR_r$ is double equivariant: $\phi \circ \tau \circ \mR_r [\hat{r}] = \mR_{\tau \circ r} [\tau \circ \hat{r}]$.

\item 
ULTRA then invokes the second instance of NBFNet on the original graph $\gG$, by setting the initial node and edge features from $\mR_r$. Again, following 1., we know that the output triplet representations of ULTRA are invariant to the node permutations:
\begin{align*}
    \Gamma\trip_{\text{ULTRA}} ((h, r, t), \adj) = \Gamma\trip_{\text{ULTRA}} ((\phi \circ h, r, \phi \circ t), \phi \circ \adj) , \forall t \in \gV.
\end{align*}

\item
Then, since the relation representations $\mR_r$ are double equivariant, and they are used as edge features in the original graph, this means that:
\begin{enumerate}
    \item The initial node embeddings set for the head node is the same: $\mR_{r}[r] = \mR_{\tau \circ r}[\tau \circ r] = \mR_{r'}[r']$.
    \item The edge type features are equivariant to $\tau$: $\mR_{r}[\hat{r}] = \mR_{\tau \circ r}[\tau \circ \hat{r}] = \mR_{r'}[\tau \circ \hat{r}]$ for any other relation/edge type $\hat{r} \in \gR$.
\end{enumerate}

\item
Hence, the final triplet representations produced by ULTRA is double invariant:
\begin{align*}
    &\Gamma\trip_{\text{ULTRA}} ((h, r, t), \adj) \\
    & = \Gamma\trip_{\text{ULTRA}} ((\phi \circ h, r, \phi \circ t), \phi \circ \adj) \\
    & = \Gamma\trip_{\text{ULTRA}} ((\phi \circ h, \tau \circ r, \phi \circ t), \tau \circ \phi \circ \adj) \\
    & = \Gamma\trip_{\text{ULTRA}} ((h', r', t'), \adj') .
\end{align*}

\item 
As a result, the graph representations produced by ULTRA is double equivariant. Namely,
\begin{align*}
    \phi \circ \tau \circ \Gamma_{\text{ULTRA}}(\adj) = \Gamma_{\text{ULTRA}}(\adj') = \Gamma_{\text{ULTRA}}(\phi \circ \tau \circ \adj) .
\end{align*}
\end{enumerate}
\end{proof}

\lemfour*
\begin{proof}
To solve \ourtask, InGram~\citep{ingram} first constructs a \textit{relation graph}, in which the relation types are treated as nodes, and the edges between them are weighted by the affinity scores, a measure of co-occurrence between relation types in the original \ourgraphease. It then employs a variant of the GATv2~\citep{velivckovic2017graph,brody2021attentive} on the relation graph to propagate and generate embeddings for the relation types. These relation embeddings, together with another GATv2, are applied to the original \ourgraphease to generate embeddings for the nodes. Finally, a variant of DistMult~\citep{yang2015embedding} is used to compute the scores for individual triplets from the embeddings of the head and tail nodes and the embedding of the relation. 

If the input node and relation embeddings to the InGram model were to be the same across all nodes and across all relation types respectively (such as vectors of all ones), then InGram would have produced double structural representations for the triplets~(\cref{def:trip-repr}). Simply put, this is because the relation graphs of InGram~\citep{ingram} encode only the structural features of the relation types (their mutual structural affinity), which is double equivariant to the permutation of relation type and node indices. Since the same initial embeddings for all nodes and relations are naively double equivariant, and the GATv2~\citep{velivckovic2017graph,brody2021attentive} is a message-passing neural network~\citep{gilmer17a} that also produces equivariant representations, the final relation embeddings would be double equivariant. Same analysis will also show the final node embeddings are double equivariant. 

However, to improve the expressivity of the model, InGram~\citep{ingram} randomly re-initializes the input embeddings for all node and relation types using Glorot initialization~\citep{glorot2010understanding} \textit{for each epoch during training}, a technique inspired by recent studies on the expressive power of GNNs~\citep{abboud2020surprising,sato2021random,murphy2019relational}. Unfortunately, random initial features break the double equivariance of the generated representations, making them sensitive to the permutation of node and relation type indices. However, since the initial node $\mV^{(0)}$ and relation embeddings $\mR^{(0)}$ are randomly initialized, and by design of InGram architecture, we have $\forall (i,k,j)\in [\nodesize]\times[\relsize]\times[\nodesize], \forall \phi\in \sS_\nodesize, \tau\in \sS_\relsize, \mathbf{Z}_{\text{InGram}}((i, k, j), \adj, \mV^{(0)}, \mR^{(0)}) = \mathbf{Z}_{\text{InGram}}((\phi\circ i, \tau\circ k, \phi\circ j), \phi\circ \adj, \mV^{(0)}, \mR^{(0)})$ for any random samples of node and relation embeddings $v^{(0)}, r^{(0)}$. We define $\tZ_{\text{InGram}}|\adj = [\mathbf{Z}_{\text{InGram}}((i,k,j),\adj, \mV^{(0)}, \mR^{(0)}))]_{(i,k,j)\in [\nodesize]\times[\relsize]\times[\nodesize]}$, and $\phi\circ\tau\circ \tZ_{\text{InGram}}|\phi\circ\tau\circ\adj = [\mathbf{Z}_{\text{InGram}}((\phi\circ i,\tau\circ k,\phi\circ j),\phi\circ\tau\circ\adj, \mV^{(0)}, \mR^{(0)}))]_{(\phi\circ i,\tau\circ k,\phi\circ j)\in [\nodesize]\times[\relsize]\times[\nodesize]}$. Since $\mV^{(0)}, \mR^{(0)}$ random variables that do not change with permutations, we can easily derive $p(\phi\circ\tau\circ \tZ_{\text{InGram}}|\phi\circ\tau\circ\adj) = p(\tZ_{\text{InGram}}|\adj)$. Thus, InGram is a distributionally double equivariant positional graph embedding of $\adj$ as per \Cref{def:pos-trip-repr}.
\end{proof}

\lemthree*
\begin{proof}
    From the \OurNewModel model architecture (\Cref{eq:dssrepr}), $\Gamma_\text{\OurNewModel}((i,k,j),\adj) = ( X_{i, k} \mathbin\Vert X_{j, k}\mathbin\Vert d(i, j) \mathbin\Vert d(j, i) )$. Using DSS layers, we can guarantee the node representations $X_{i,k}$ we learn are double invariant under the node and relation permutations, where $X_{i,k}$ in $\adj$ is equal to $X_{\phi\circ i,\tau\circ k}$ in $\phi\circ\tau\circ \adj$. It is also clear that the distance function is invariant to node and relation permutations, i.e. $\forall i,j\in [\nodesize]$, $d(i,j)$ in $\adj$ is the same as $d(\phi\circ i, \phi\circ j)$ in $\phi\circ\tau\circ \adj$. Thus $\Gamma_\text{\OurNewModel}((i,k,j),\adj) = \Gamma_\text{\OurNewModel}((\phi \circ i,\tau\circ k,\phi\circ j),\phi\circ\tau\circ \adj)$ is a double invariant triplet representation as in \Cref{def:trip-repr}. 
\end{proof}

\thmtwo*
\begin{proof} 
Recall that we have two different cases $h = 1$ and $h = 2$ for \Cref{eq:uqer} in \Cref{def:uqer} of UQER.
For the ease of proof, we will focus on the case where $h = 2$ in the following content, and for the case $h = 1$, the proof will be the same.

Given $h = 2$, any UQER is defined by $\tB\in\{0,1\}^{M\times K\times M}$ as 
\begin{equation}
\begin{gathered}
 \forall E_1 \in \cV, \left( \forall E_u \in \cV \setminus \left\{ E_1, \ldots, E_{u-1} \right\} \right)_{u=2}^M \,,
    \forall C_1 \in \cR, \left( \forall C_c \in \cR \setminus \left\{ C_1, \ldots, C_{c-1} \right\} \right)_{c=2}^K \,, \\
    \bigwedge_{\substack{u,u' = 1, \ldots, M, c = 1, \ldots, K, \\ \tB_{u,c,u'} = 1 }} (E_u,C_c,E_{u'}) \implies (E_1,C_1,E_h) ,\:
\end{gathered}
\end{equation}
for any node set $\cV$ and relation set $\cR$ with number of nodes and relations $\nodesize,\relsize$ s.t. $\nodesize \geq M, \relsize \geq K$, where if $M > 2$, $\forall u \in \{3, \ldots, M \}$,
$\sum_{u' = 1}^M \sum_{c = 1}^K \tB_{u,c,u'} + \tB_{u',c,u} \geq 1$, and if $K \geq 2$, $\forall c \in \{2, \ldots, K \}$, $\sum_{u = 1}^M \sum_{u' = 1}^M \tB_{u,c,u'} + \tB_{u',c,u} \geq 1$ (every variable should appear at least once in the formula).

For all sets of truth statements $\forall \cS\subseteq \cup_{N=1}^\infty\cup_{R=2}^\infty \cV\times\cR\times\cV$, it has an equivalent tensor representation $\adj\in \{0,1\}^{\nodesize\times \relsize\times \nodesize}$ such that $\adj_{i,k,j}=1 \iff (v_i,r_k,v_j \in \cS$.
We can then define a triplet representation $\Gamma_\text{triplet}$ based on the given UQER as, $\forall (i,k,j) \in [\nodesize]\times[\relsize]\times[\nodesize] $,
\begin{equation}
\label{eq:myGammatri}
\Gamma_\text{triplet}((i,k,j), \adj) = \begin{cases}
1 & \text{if } (i,k,j)\in \cS'  \\
0 & \text{otherwise}, \\
\end{cases}
\end{equation}
where we define $\cS' = \big\{ \left( i,k,j \right) \,\big\vert\, \forall \left( i,k,j \right) \in [\nodesize]\times [\relsize]\times[\nodesize], 
\text{such that} \left( E_1, C_1, E_2 \right) = \left(v_i,r_k,v_j \right) \in \cV \times \cR \times \cV, \ \exists^{M-2} E_3,...,E_M \in \cV \setminus \left\{ E_1, E_2 \right\},  \exists^{K-1} C_2,...,C_K \in \cR \setminus \left\{ C_1 \right\}, \text{ where } \forall (u,c,u')\in [M]\times [K]\times [M], \tB_{u,c,u'} = 1 \Rightarrow \left(E_u, C_c, E_{u'} \right) \in \cS \big\}$ %
is the set of true statements induced by modus ponens from the truth statements $\cS$ and the UQER Horn Clause, where the existential quantifier $\exists^k$ means exists at least $k$ distinct values.

All we need to show is that \Cref{eq:myGammatri} is a double invariant triplet representation. 
For any node permutation $\phi\in \sS_\nodesize$ and relation permutation $\tau\in\sS_\relsize$ of $\adj$, we define $\phi\circ \tau\circ \cS = \{(v_{\phi\circ i},r_{\tau\circ k},v_{\phi\circ i}|(v_i,r_k,v_j\in \cS\}$ which corresponds to their equivalent tensor representation 
$\phi\circ \tau\circ \adj$, where $(\phi\circ \tau\circ \adj)_{\phi\circ i,\tau\circ k,\phi\circ j}=1 \iff (v_i,r_k,v_j )\in \cS$ otherwise $0$. Similarly, we have
$\phi\circ\tau\circ\cS' = \big\{ \left( \phi\circ i,\tau\circ k,\phi\circ j \right) \,\big\vert\, \forall \left( i,k,j \right) \in [\nodesize]\times [\relsize]\times[\nodesize], 
\text{such that} \left( E_1, C_1, E_2 \right) = \left(v_{\phi\circ i},r_{\tau \circ k},v_{\phi \circ j} \right) \in \cV \times \cR \times \cV, \ \exists^{M-2} E_3,...,E_M \in \cV \setminus \left\{ E_1, E_2 \right\},  \exists^{K-1} C_2,...,C_K \in \cR \setminus \left\{ C_1 \right\}, \text{ where } \forall (u,c,u')\in [M]\times [K]\times [M], \tB_{u,c,u'} = 1 \Rightarrow \left(\phi\circ E_u, \tau\circ C_c, \phi\circ E_{u'} \right) \in \phi\circ\tau\circ \cS \big\}$.

By definition, we have that for any $(i,k,j) \in \cS'$, 
$$\Gamma_\text{triplet}((\phi\circ i,\tau\circ k,\phi\circ j), \phi\circ \tau\circ \adj) = \begin{cases}
1 & \text{if } (\phi\circ i,\tau\circ k,\phi\circ j) \in \phi\circ\tau\circ\cS'  \\
0 & \text{otherwise}, \\
\end{cases}.$$ 
Now we show that $(i,k,j)\in \cS'$ if and only if $(\phi\circ i, \tau\circ k, \phi\circ j)\in \phi\circ\tau\circ\cS'$. If $(i,k,j)\in \cS'$, then $E_1=v_{i},E_2=v_{j}, C_1=r_{k}, \exists^{M-2} E_3,...,E_M \in \cV\setminus \{E_1,E_2\},  \exists^{K-1} C_2,...,C_K \in \cR\setminus \{C_1\},  \text{such that } \tB_{u,c,u'}=1 \implies (E_u,C_c,E_{u'})\in \cS $ . Since $(E_u,C_c,E_{u'})\in \cS$ if and only if $(\phi\circ E_u,\tau\circ C_c,\phi\circ E_{u'})\in \phi\circ\tau\circ \cS$ by definition, we have $(\phi\circ i,\tau\circ k,\phi\circ j)\in \phi\circ\tau\circ\cS'$. Similarly we can prove if  $(\phi\circ i,\tau\circ k,\phi\circ j)\in \phi\circ\tau\circ\cS'$, then $(i,k,j)\in \cS'$ with the same reasoning.

In conclusion, for any $\adj\in \sA$ with number of nodes and relations $\nodesize,\relsize$, since $(i,k,j)\in \cS'$ if and only if $(\phi\circ i, \tau\circ k, \phi\circ j)\in \phi\circ\tau\circ\cS'$, then by definition $\Gamma_\text{triplet}((\phi\circ i,\tau\circ k,\phi\circ j), \phi\circ \tau\circ \adj)=\Gamma_\text{triplet}((i,k,j), \adj)$ holds $\forall (i,k,j) \in %
[\nodesize]\times[\relsize]\times[\nodesize]$, which proves $\Gamma_\text{triplet}$ is a double invariant triplet representation (\Cref{def:trip-repr}).

\end{proof}

\section{Additional Related Work}
\label{sec:appd-rel}

Link prediction in {\ourgraphs}, which are commonly used to represent relational data in a structured way by indicating different types of relations between pairs of nodes in the graph, involves predicting not only the existence of missing edges but also the associated relation types.
\paragraph{Transductive link prediction.}
In transductive link prediction, missing links are predicted over a fixed set of nodes and relation types as in training.
Traditionally, factorization-based methods~\citep{sutskever2009modelling,nickel2011three,bordes2013translating,wang2014knowledge,yang2015embedding,trouillon2016complex,nickel2016holographic,trouillon2017knowledge,dettmers2018convolutional,sun2018rotate} have been proposed to obtain latent embedding of nodes and relation types to capture their relative information in the graph. These models try to score all combinations of nodes and relations with embeddings as factors, similar to tensor factorization. Although excellence in transductive tasks, these positional embeddings~\citep{srinivasan2020on} (a.k.a. permutation-sensitive embeddings) require extensive retraining to perform inductive tasks over new nodes or relations~\citep{teru2020inductive}.
However, in real-world applications, relational data is often evolving, requiring link prediction over new nodes and new relation types, or even entirely new graphs.

\paragraph{Inductive link prediction over new nodes (but not new relations) with GNN-based model.}
In recent years, with the advancement of graph neural networks (GNNs)~\citep{defferrard2016convolutional,Kipf2016,hamilton2017inductive,velivckovic2017graph,bronstein2017geometric,murphy2019relational}, in graph machine learning fields, various works has applied the idea of GNN in relational prediction to ensure the inductive capability of the model, including RGCN~\citep{schlichtkrull2018modeling}, 
CompGCN~\citep{vashishth2019composition}, GraIL~\citep{teru2020inductive}, NodePiece~\citep{galkin2021nodepiece}, NBFNet~\citep{zhu2021neural}, ReFactorGNNs~\citep{chenrefactor} etc.. Specifically, RGCN~\citep{schlichtkrull2018modeling} and CompGCN~\citep{vashishth2019composition} were initially designed for transductive link prediction tasks.  As GNNs are node permutation equivariant~\citep{xu2018powerful,srinivasan2020on}, these models learn structural node/pairwise representation, which can be used to perform {\em inductive link prediction over solely new nodes}, 
while most of the GNN performance are worse than FM-based methods~\citep{Ruffinelli2020You,chenrefactor}. 
Specifically, GraiL~\citep{teru2020inductive} extends the idea from \citep{zhang2018link} to use local subgraph representations for \ourgraph link prediction. Refactor GNNs~\citep{chenrefactor} aims to build the connection between FM and GNNs, where they propose an architecture to cast FMs as GNNs. NodePiece~\citep{galkin2021nodepiece} uses anchor-nodes for parameter-efficient architecture for \ourgraph completion. NBFNet~\citep{zhu2021neural} extends the Bellman-Ford algorithm, which learns pairwise representations by all the path representations between nodes. \citep{barcelo2022weisfeiler} analyzes \ourgraph-GNNs expressiveness by connecting it with the Weisfeiler-Leman test in \ourgraph. %

\paragraph{Inductive link prediction over new nodes (but not new relations) with logical induction.} The relation prediction problem in relational data represented by \ourgraph can also be considered as the problem of learning first-order logical Horn clauses \citep{yang2015embedding,yang2017differentiable,sadeghian2019drum,teru2020inductive} from the relational data, where one aims to extract logical rules on binary predicates. These methods are inherently node-independent and are only able to perform {\em inductive link prediction over new nodes}. \citep{barcelo2020logical} discusses the connection between the expressiveness of GNNs and first-order logical induction, but only on node GNN representation and logical node classifier. \citep{qiu2023logical} further analyzes the logical expressiveness of GNNs for \ourgraphease by showing GNNs are able to capture logical rules from graded modal logic and provides a logical explanation of why pairwise GNNs~\citep{zhang2021labeling,zhu2021neural} can achieve SOTA results. 
In our paper, we try to build the connection between triplet representation and logical Horn clauses. Traditionally, logical rules are learned through statistically enumerating patterns observed in \ourgraph~\citep{lao2010relational,galarraga2013amie}. Neural LP~\citep{yang2017differentiable} and DRUM~\citep{sadeghian2019drum} learn logical rules in an end-to-end differentiable manner using the set of logic paths between two nodes with sequence models. RLogic~\citep{cheng2022rlogic} follows a similar manner, which breaks a big sequential model into small atomic models in a recursive way. \citep{galkin2022inductive} aims to inductively extract logical rules by devising NodePiece~\citep{galkin2021nodepiece} and NBFNet~\citep{zhu2021neural}. However, all these methods are not able to deal with new relation types in test.

\paragraph{Inductive link prediction over both new nodes and new relations (with extra context)} 
Few-shot and zero-shot relational reasoning~\citep{xiong2018one,lv2019adapting,qin2020generative,zhao2020attention,geng2021ontozsl,wang2021relational,csr2022,chen2023generalizing,geng2023relational} aim to query triplets involving unseen relation types with access to few or zero support triplets of these unseen relation types at test time. Recent methods~\citep{qin2020generative,zhao2020attention,csr2022,geng2023relational} can even query over unseen nodes. Yet, they often need extra context in the test graph, such as textual descriptions and/or ontological information of the unseen relation types or a shared background graph between the training and test graph, i.e., the test nodes and relation types are connected to the training ones. %
For instance, zero-shot link prediction methods such as~\citep{qin2020generative} employ a generative adversarial network~\citep{goodfellow2014adv} to utilize the additional textual information to bridge the semantic gap between seen and unseen relations. 
Later, OntoZSL~\citep{geng2021ontozsl} presented an ontology-enhanced zero-shot learning approach that incorporates both ontology structural and textural information. 
Similarly, TACT~\citep{chen2021topology} aims to model the topological correlations between the target relations and their adjacent relations (assumes there are relations that are seen in train) using a relational correlation network to learn more expressive representations of the target relations. 
A recent work is RMPI~\citep{geng2023relational} that extracts enclosing subgraphs around the target triplet, which are assumed to contain triplets of some relation types seen in training and uses graph ontology to bridge the unseen relation types to the seen ones. 
Another one~\citep{zhao2020attention} uses attention-based GNNs and convolutional transition for link prediction over new nodes and new relations assuming a shared background graph between training and test (i.e., new relations in test are connected with existing nodes and relations in training).
MaKEr~\citep{MaKEr} also uses the local graph structure to handle new nodes and new relation types using a meta-learning framework, assuming the test graph has overlapping relations and entities with the training graph. 
On the other hand, few-shot relational reasoning methods learn representations of the unseen relation types from the few support triplets, which are generally assumed to connect to existing nodes and relations seen in training~\citep{xiong2018one,chen2019meta,zhang2020few}. 
For example, \citep{xiong2018one} was the first to solve the one-shot task by proposing to compute matching scores between the new relation types observed in the support set to those training relation types. Later, FSRL~\citep{zhang2020few} extends \citep{xiong2018one} by using an attention-based aggregation to take advantage of information from all support triplets. 
\update{ 
Another line of work is called subgraph reasoning, i.e., predict $(i, r?, j)$ by extracting subgraphs surrounding $i$ and $j$ and subgraphs surrounding another triplet $(i', r, j')$ with the same target relation type $r$ but different entities $i'$ and $j'$. As long as these two set of subgraphs are similar, one can proclaim that the target triplet $(i, r, j)$ should exist. GraphANGEL~\citep{jin2022inductive} connects this subgraph reasoning idea to first-order logical (FOL) reasoning to tackle unseen relation types at test time, by extracting subgraphs whose structural patterns conform to a set of pre-defined FOL rules. Concurrently, CSR~\citep{csr2022} employs a similar method but grounds the subgraph reasoning procedure in the language of statistical hypothesis testing. In comparison to our work, all existing subgraph reasoning methods require that the target triplet $(i, r, j)$ with the unseen relation type $r$ must be surrounded by subgraphs consisting solely of triplets with relation types seen in training. In other words, $(i, r, j)$ must be connected to the training graph. Hence, they are not applicable when the test KG comes from another domain that is disconnected to the training KG, which is the most general inductive scenario studied by this work.
}
Finally, one other line of research is to solve few-shot relational reasoning via meta-learning. For instance, MetaR~\citep{chen2019meta} updates a meta representation over the relation types, and MetaKGR~\citep{lv2019adapting} adopts MAML~\citep{finn2017model} to learn meta parameters for frequently occurring relations, which can then be adapted to few-shot relations.
\update{
Similar to subgraph reasoning methods,
}
all of these few-shot learning methods require that the few-shot triplets are connected to a background graph observed during training in order to learn about the relationship between new relation types and existing ones. 
Hence, all these methods cannot be directly applied to test graphs that neither contain textual descriptions of the unseen relation types nor triplets involving those relation types seen in training. 

\paragraph{Inductive link prediction over both new nodes and new relations (no extra context)} In this paper, we focus on the most general task, i.e., inductive link prediction over both new nodes and new relations on entirely new test graphs without textual descriptions, which we call {\em \ourtask}. To the best of our knowledge, InGram~\citep{ingram} and ULTRA~\citep{galkin2023towards} are the only existing methods capable of performing this task. InGram and ULTRA introduce relation graphs to capture relational representations based on their interactions.
In contrast to InGram and ULTRA that designed a specific architecture, our work proposes a general theoretical framework for designing an entire class of models capable of solving the \ourtask task, which encompasses them as a specific instantiation. Modeling details of InGram and ULTRA have been substantially discussed in the main paper.

\paragraph{Equivariance of graph neural networks}
\update{
A substantial body of literature has focused on restricting neural networks to a class of functions relevant to the task by exploiting the symmetry of problems and enforcing equivariance with respect to a certain symmetry group of transformations. Earlier examples such as translation equivariance in Convolutional Neural Networks have shown the effectiveness of such methodology for computer vision tasks~\citep{Bruna2013SpectralNA,defferrard2016convolutional}. In the field of graph machine learning, Graph Neural Networks possess the equivariance to permutation to node identities~\citep{Kipf2016}. Our work on double equivariance belongs to this category of research and is a generalization of traditional GNN equivariance by extending it to the Cartesian product of permutation groups of both node and relation identity permutations. One parallel line of research is to exploit the geometric symmetry in the input node features on tasks such as point cloud data, molecular structure representation, and particle simulation~\citep{ramakrishnan2014quantum,kipf2018neural,uy2019revisiting,Han2022GeometricallyEG,satorras2021n,liu2021pre,Zhou2023UniMolAU}. In particular, several works have studied how to achieve equivariance to rotations and translations for 3D vector input features, i.e., E(3) equivariance, as well as SE(3) equivariance when reflections are also considered~\citep{Thomas2018TensorFN,fuchs2020se,finzi2020generalizing,kohler2020equivariant}. \citet{satorras2021n} then proposes an E($n$) equivariance architecture that can efficiently generalize to translation, rotation, and reflection in larger than 3 dimensional spaces. Follow-up work further applies these geometry equivariance to the task of learning molecular representations~\citep{liu2021pre,Zhou2023UniMolAU}. In comparison to our work, these studies focus on symmetry in the input features, whereas the theory of double equivariance is concerned with the structure of knowledge graphs. Hence, they are orthogonal to our contribution.

\paragraph{Out-of-Distribution (OOD) generalization of graph neural networks}
Many different types of OOD scenarios on graphs have been studied by the literature and each type of OOD scenario potentially requires a different kind of invariance assumption, among which causality has been found to be one of the best framework to model the invariances\citep{bevilacqua2021size,chen2022learning,zhao2022learning}. For example, \citet{bevilacqua2021size} first proposed to tackle OOD graph classification tasks, where test graphs have larger size than those seen in training, by understanding the problem from a causal perspective and capture the invariance via the structural causal models (SCMs)~\citep{pearl2009causality}. \citet{chen2022learning} then extended this methodology to a wider class of OOD and distribution shift scenarios on homogeneous graphs. \citet{zhao2022learning} similarly proposed to understand the problem of link prediction from a causal perspective and achieve superior performance on a wide range of link prediction tasks, although they did not directly the OOD problems. In connection to our work, the causality and invariance assumptions studied by all these works are largely orthogonal to our invariance assumption of double equivariance. In particular, we posit that double equivariance is a necessary condition for effective solution to the zero-shot fully inductive task on knowledge graphs. Without guaranteeing double equivariance, applying the methods from~\citep{bevilacqua2021size,chen2022learning,zhao2022learning} would not solve our task. Nevertheless, the interesting question remains as to whether a double equivariant link prediction model can benefit from these causality assumptions, and whether doing so can better alleviate the negative transfer problem that we observe in the meta-learning setting (\Cref{sec:negative-transfer}). We leave this interesting and important investigation to future work.
}

%
\section{Experiments}
\label{sec:appd-supexp}

%

%
\subsection{\Ourtask task over both new nodes and new relation types}
In this section, we provide more detailed experiment results and analysis for our method on inductively \ourtask on both new nodes and new relation types. 
%



%
%

\subsubsection{Experiment Setup}

\paragraph{Evaluation Metrics.} We sample $50$ negative triplets for each test positive triplet during test evaluation by corrupting either nodes or relation types (\Cref{eq:loss}), and use Nodes Hits@$k$ and Relation Hits@$k$ separately which counts the ratio of positive triplets ranked at or above the $k$-th place against the $50$ negative samples as evaluation metric over $5$ runs. Specifically, for Node prediction evaluation, we sample without replacement $50$ negative tail (or head) nodes, and for Relation prediction evaluation, we sample with replacement $50$ negative relation types (can also handle cases where the number of test relations is less than $50$). We also report other widely used metrics such as MRR.%

\paragraph{Hyperparameters and Implementation Details.}
For homogeneous GNN methods, NBFNet and \OurNewModel, We follow the same configuration as \citep{teru2020inductive} such that the hidden layers have 32 neurons. We use Adam optimizer with grid search over learning rate $\alpha\in \{0.01,0.001,0.0001\}$, and over weight decay $\beta\in\{0.0005,0\}$. For all datasets, we train these models for 10 epochs with a mini-batch size of 16. %
For the GNN kernel we choose GCN~\citep{Kipf2016} of \OurNewModel. For these models, the number of hops and number of layers are 3 on all datasets to ensure fair comparison.

Since NBFNet is designed to only perform inductive link prediction with solely new nodes and utilizes trained relation embeddings, we use randomly initialized embeddings for the unseen relation types at test time to enable it for performing \ourtask.

To run InGram~\citep{ingram} on \ourdata and \ourotherdata, we conduct hyperparameter search over the configurations of ranking loss margin $\gamma \in \{ 1.0, 2.0 \}$, learning rate $\alpha \in \{ 0.0005, 0.001 \}$, number of entity layers $L \in \{ 2, 3, 4 \}$, and number of entity layers $\hat{L} \in \{ 2, 3, 4 \}$. For other hyperparameters, we use the suggested values from InGram~\citep{ingram} and their codebase, such as the number of bins $B = 10$ and the number of attention heads $K = 8$. We then use the overall best-performing hyperparameters on \ourdata and the best-performing hyperparameters on \ourotherdata to run InGram on all tasks in \ourdata and all tasks in \ourotherdata respectively.


To run \OurOtherModel, we use the same trained checkpoints of InGram. The difference is at inference time, where instead of a single forward pass with one sample of randomly initialized entity and relation embeddings for InGram, we draw 10 samples of initial entity and relation embeddings and run 10 forward passes. This yields 10 Monte Carlo samples of the triplet scores, which we then use to compute the \OurOtherModel triplet scores according to~\Cref{eq:de-ingram}.

For RMPI~\citep{geng2023relational}, we use the provided hyperparameters from the codebase and run the RMPI-NE version of the model with a concatenation-based fusion function, which generally has the best performance reported in RMPI~\citep{geng2023relational}. We note that, since our \ourgraphease does not contain ontological information over the unseen relation types of the test graphs, we instead provide the model with randomly initialized embeddings for the unseen relation types to perform \ourtask. 

Training was performed on NVidia A100s, L4s, GeForce RTX 2080 Ti, and TITAN V GPUs.

\subsubsection{\Ourtask over \ourdata}
\label{appx:exp-pediatype}

As discussed in \Cref{sec:exp}, we create our own \ourtask benchmark dataset \ourdata. Each graph in \ourdata is sampled from a graph in the OpenEA library~\citep{OpenEA} (under GPL-3.0 license).
OpenEA~\citep{OpenEA} library provides multiple pairs of \ourgraphease, each pair of which is a database containing similar topics.
Each node of a graph corresponds to the Universal Resource Identifier (URI) of an entity in the database, e.g., ``{\em http://dbpedia.org/resource/E399772}'' from English DBPedia.
Each relation type of a graph corresponds to the URI of a relation in the database, e.g., ``{\em http://dbpedia.org/ontology/award}'' from English DBPedia.
Moreover, since each pair of graphs describes similar topics, most entities and relations are highly related, e.g., ``{\em http://dbpedia.org/resource/E678522}'' from English and ``{\em http://fr.dbpedia.org/resource/E415873}'' from French are indeed the same thing, except that the labeling is different. These multilingual KGs predominantly use English for relation labels, which causes an overlap in relations. However, in our experimental setup, we treat relations as if they were in different languages and do not leverage this overlapping information during model training. Thus, we would expect a powerful model that is insensitive to node and relation type labelings to be able to learn on one graph of the pair and perform well on the other graph of the same pair.

To control the size under a feasible limitation, we use the same subgraph sampling algorithm as GraIL~\citep{teru2020inductive}, which proposes link prediction benchmarks over solely new nodes. Details are provided in \Cref{alg:sample-deg}. For each pair of graphs from the OpenEA library, e.g., English-to-French DBPedia, we first apply the sampling algorithm as in \Cref{alg:sample-deg} on each graph to reduce the size of each graph. Then we randomly split querying triplets given by the \Cref{alg:sample-deg} into 80\% training, 10\% validation, and 10\% test for each graph. Finally, to construct the task where we learn on English DBPedia but test on French DBPedia (denoted as EN-FR), we pick training and validation triplets from the English graph for model tuning, and only use test triplets from the French graph for model evaluation; Similarly, for task from French to English (FR-EN), we pick training and validation triplets from French graph for model tuning, and only use test triplets from English graph for model evaluation. The {\bf dataset statistics} for \ourdata are summarized in \Cref{fig:stats-pediatype}.

\begin{figure}
\centering
\begin{minipage}{0.24\textwidth}
\centering
\resizebox{\linewidth}{!}{
\begin{tabular}{l|cc}
\hline
& DBPedia & Wikidata \\
\hline
\#Nodes & 4906 & 4948\\
\#Relations & 144 & 102 \\
\#Triplets (Obv.) & 17593 & 23888 \\
\#Triplets (Qry.) & 1666 & 2456 \\
\#Avg. Deg. & 7.85 & 10.65 \\
\hline
\end{tabular}
}
\includegraphics[width=\linewidth]{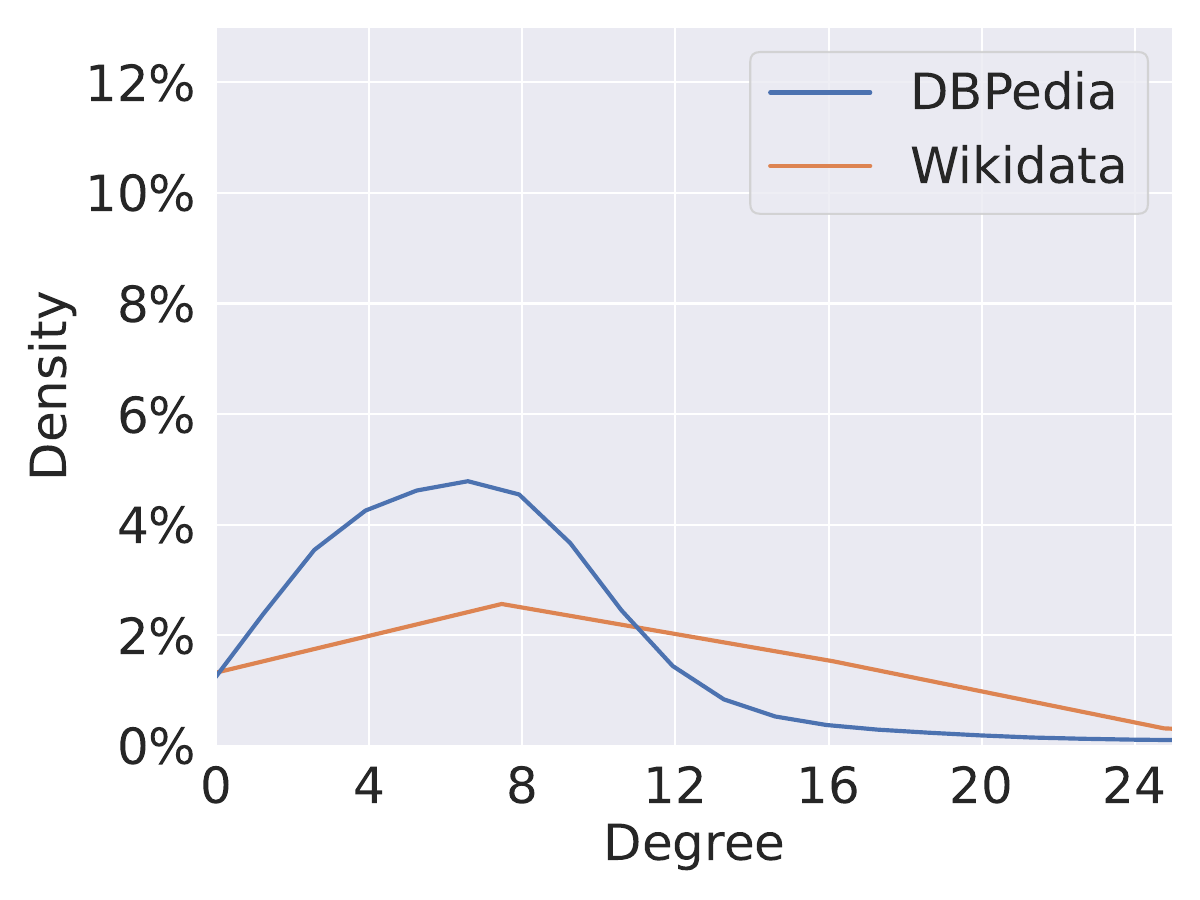}
\subcaption{DB$\longleftrightarrow$WD}
\end{minipage}
\hfill
\begin{minipage}{0.24\textwidth}
\centering
\resizebox{\linewidth}{!}{
\begin{tabular}{l|cc}
\hline
& DBPedia & YAGO \\
\hline
\#Nodes & 4795 & 4751 \\
\#Relations & 64 & 17 \\
\#Triplets (Obv.) & 13248 & 11327 \\
\#Triplets (Qry.) & 1177 & 973 \\
\#Avg. Deg. & 6.02 & 5.18 \\
\hline
\end{tabular}
}
\includegraphics[width=\linewidth]{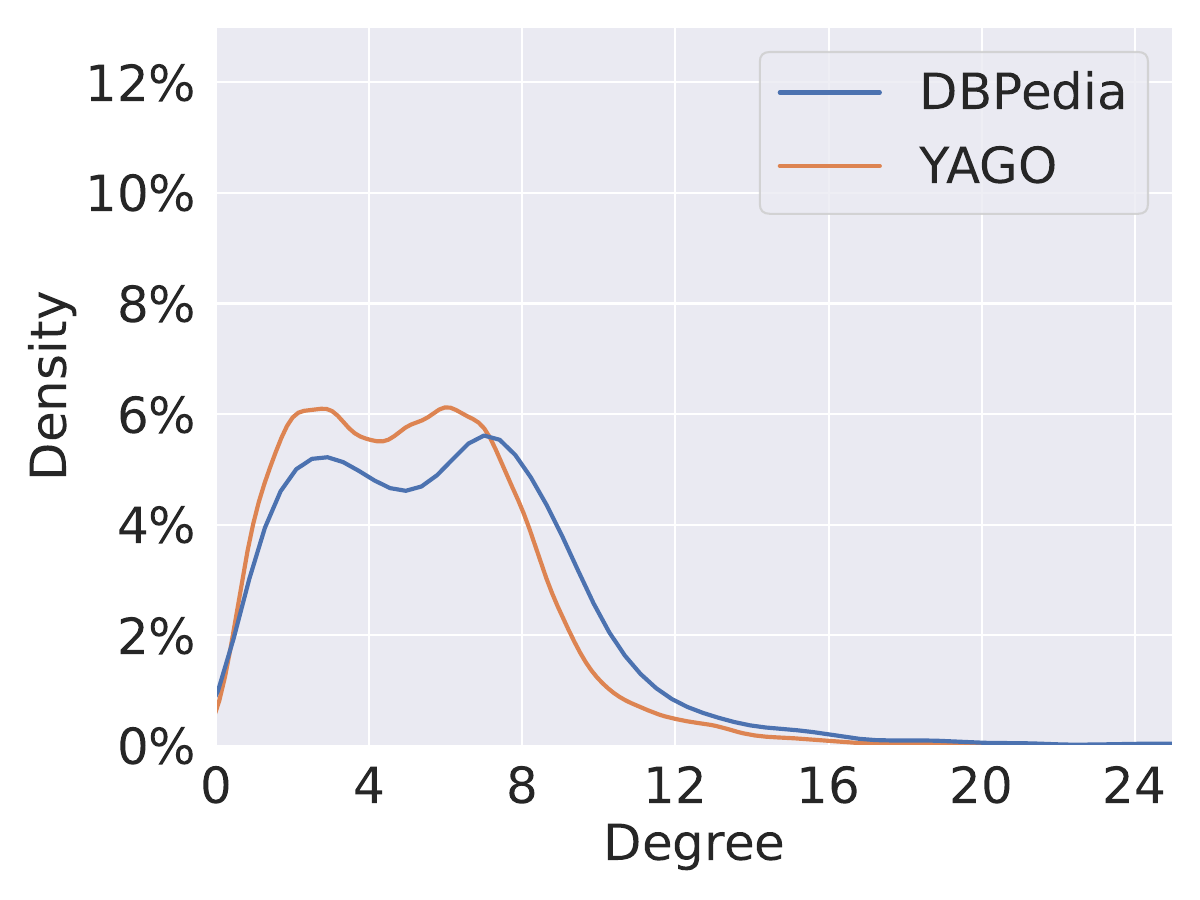}
\subcaption{DB$\longleftrightarrow$YG}
\end{minipage}
\hfill
\begin{minipage}{0.24\textwidth}
\centering
\resizebox{\linewidth}{!}{
\begin{tabular}{l|cc}
\hline
& English & French \\
\hline
\#Nodes & 4962 & 4933 \\
\#Relations & 122 & 101 \\
\#Triplets (Obv.) & 30876 & 24165 \\
\#Triplets (Qry.) & 3326 & 2485 \\
\#Avg. Deg. & 13.79 & 10.81 \\
\hline
\end{tabular}
}
\includegraphics[width=\linewidth]{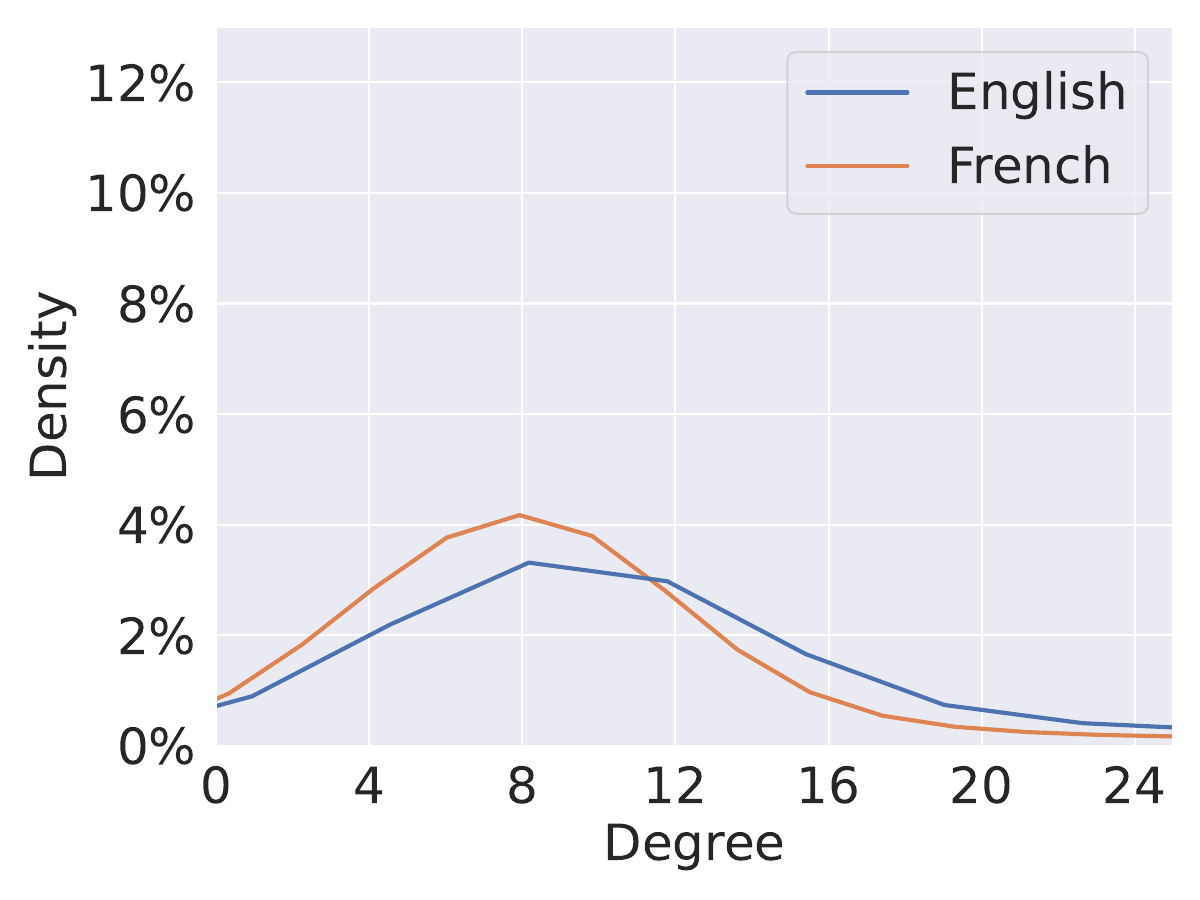}
\subcaption{EN$\longleftrightarrow$FR}
\end{minipage}
\hfill
\begin{minipage}{0.24\textwidth}
\centering
\resizebox{\linewidth}{!}{
\begin{tabular}{l|cc}
\hline
& English & German \\
\hline
\#Nodes & 4890 & 4915 \\
\#Relations & 121 & 67 \\
\#Triplets (Obv.) & 25177 & 29011 \\
\#Triplets (Qry.) & 2626 & 3100 \\
\#Avg. Deg. & 11.37 & 13.07 \\
\hline
\end{tabular}
}
\includegraphics[width=\linewidth]{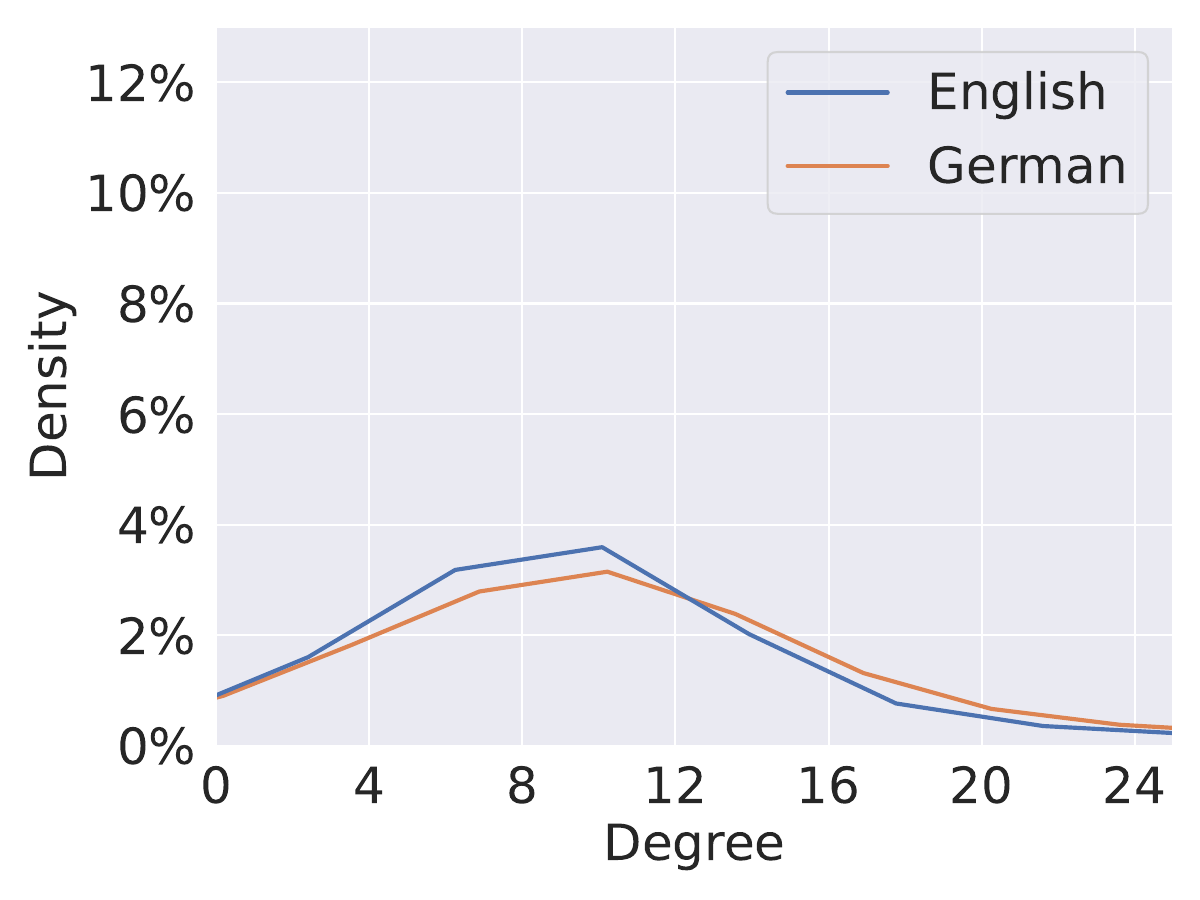}
\subcaption{EN$\longleftrightarrow$DE}
\end{minipage}
\caption{{\bf Statistics of \ourdata:} We report graph statistics including the number of nodes, number of relations, observed (obv.) triplets, querying (qry.) triplets, and average degree for each graph pair, e.g., (a) corresponds to DBPedia-and-Wikidata pair, and will be used to construct DB2WD and WD2DB tasks. We also report (in \& out) degree distribution on each graph at the bottom. We omit tail distribution larger than 25 since they are too small and almost flat.}
\label{fig:stats-pediatype}
\end{figure}

\begin{algorithm}
\caption{{\bf Sampling Algorithm for \ourdata.} This is a subgraph sampling code for a single graph (either training or test). It will reduce the large original graph into a connected graph of the required size.}
\label{alg:sample-deg}
\begin{algorithmic}[1]
\Require Raw graph triplets $\cS^\text{raw}$, Raw graph node set $\cV^\text{raw}$, Raw graph relation set $\cR^\text{raw}$, Maximum number of nodes $N$, Maximum number of edges $M$, Maximum node degree $D$.
\Ensure Subgraph triplets $\cS^\text{sub}$
\State $\cS^\text{sub} \leftarrow \emptyset$
\State $\cV^\text{sub} \leftarrow \emptyset$
\State $\cR^\text{sub} \leftarrow \emptyset$
\State Create an empty queue $Q$.
\State Get the node $v_0$ with the highest degree in the raw graph.
\State $Q\text{.add}(v_0)$
\State $\cV^\text{sub} \leftarrow \cV^\text{sub} \cup \{ v_0 \}$
\While{$|Q| > 0$}
    \State $u \leftarrow Q\text{.pop}()$
    \If{$|\cV^\text{sub}| \geq N$ or $|\cV^\text{sub}| \geq M$}
        \State \textbf{continue}
    \EndIf
    \State $\cB = \{ (v, r, u) | (r, v) \in \cR^\text{raw} \times \cV^\text{raw} \} \cup \{ (u, r, v) | (r, v) \in \cR^\text{raw} \times \cV^\text{raw} \}$
    \If{$|\cB| > D$}
        \State Uniformly select $D$ triplets from $\cB$ as $\cB'$
    \Else
        \State $\cB' \leftarrow \cB$
    \EndIf
    \For{$(i, r, j) \in \cB'$}
        \If{$i = u$}
            \State $Q\text{.add}(j)$
            \State $\cV^\text{sub} \leftarrow \cV^\text{sub} \cup \{ j \}$
        \Else
            \State $Q\text{.add}(i)$
            \State $\cV^\text{sub} \leftarrow \cV^\text{sub} \cup \{ i \}$
        \EndIf
        \State $\cS^\text{sub} \leftarrow \cS^\text{sub} \cup \{ (i, r, j) \}$
    \EndFor
\EndWhile
\end{algorithmic}
\end{algorithm}

\paragraph{Additional Results} We present the Node \& Relation Hits@10 performance in the main paper. We provide more results including MRR, Hits@1, Hits@5 in \Cref{tab:pediatypes-mrr,tab:pediatypes-hit1,tab:pediatypes-hit5}. We can see that our proposed \OurNewModel and \OurOtherModel achieve competent performance with the best baseline in the much harder relation prediction task, showing their power to generalize to both new nodes and new relations. The structural double equivariant model \OurNewModel performs worse on node prediction over some datasets, which might be due to the node GNN implementation of \OurNewModel. These tasks do not care much about the actual relation type as we can see from the superior performance of homogeneous GNNs on node prediction. So the additional equivariance over relations and the training loss over both negative nodes and negative relations might cause the model to focus more on the relation prediction task, while the double equivariant structural representation might hurt the performance of missing node prediction~\citep{srinivasan2020on}. %

But it is important to note that the structural double equivariant models excel on relation prediction and achieves much better results on Hits@1 and Hits@5 as shown in \Cref{tab:pediatypes-hit1,tab:pediatypes-hit5}. We also note that in the Hits@1 and Hits@5~\Cref{tab:pediatypes-hit1,tab:pediatypes-hit5}, there are cases where \OurOtherModel has higher variances than the original InGram while achieving much better average performance. This is because due to the random initialization, InGram performs poorly on the much harder Hits@1 and Hits@5 performance compared to Hits@10. In some seeds of the runs, \OurOtherModel successfully improves the performance of InGram, but there are still seeds of runs that \OurOtherModel still performs similar to InGram. Thus, it results in \OurOtherModel having much better average results while also with higher standard deviations.

\begin{table}[t]
\centering
\captionof{table}{
{\bf Relation \& Node MRR performance on \OurTask over \ourdata.} We report standard deviations over 5 runs. %
A higher value means better \ourtask performance. 
\update{Models labeled with DEq are double equivariant models, and those labeled with d-DEq are distributionally double equivariant ones.}
``Rand'' column contains unbiased estimations of the performance from a random predictor.
\update{\bf Double equivariant models consistently achieve better results than non-double-equivariant models with generally smaller standard deviations.}
N/A*: Not available due to constant crashes.
}
{
\subcaption{{\bf {\scriptsize Relation prediction $(i, ?, j)$ performance in \%. Higher $\uparrow$ is better.}}}
\resizebox{\linewidth}{!}{
\begin{tabular}{l r r r r r r r r}
    Models & EN-FR & FR-EN & EN-DE & DE-EN & DB-WD & WD-DB & DB-YG & YG-DB \\
    \midrule
    
    Rand & 8.86{\scriptsize $\pm 00.00$} & 8.86{\scriptsize $\pm 00.00$} & 8.86{\scriptsize $\pm 00.00$} & 8.86{\scriptsize $\pm 00.00$} & 8.86{\scriptsize $\pm 00.00$} & 8.86{\scriptsize $\pm 00.00$} & 8.86{\scriptsize $\pm 00.00$} & 8.86{\scriptsize $\pm 00.00$}\\
    \midrule
    
    GAT & 8.04{\scriptsize $\pm 00.25$} & 7.93{\scriptsize $\pm 00.04$} & 8.17{\scriptsize $\pm 00.08$} & 8.12{\scriptsize $\pm 00.09$} & 8.06{\scriptsize $\pm 00.15$} & 7.90{\scriptsize $\pm 00.12$} & 8.12{\scriptsize $\pm 00.21$} & 8.17{\scriptsize $\pm 00.16$} \\
    \text{{GIN}} & 8.07{{\scriptsize $\pm 00.09$}} & 8.09{{\scriptsize $\pm 00.05$}} & 8.07{{\scriptsize $\pm 00.13$}} & 8.07{{\scriptsize $\pm 00.11$}} & 8.03{{\scriptsize $\pm 00.20$}} & 7.97{{\scriptsize $\pm 00.30$}} & 7.82{{\scriptsize $\pm 00.27$}} & 7.84{{\scriptsize $\pm 00.14$}} \\
    GraphConv & 7.92{\scriptsize $\pm 00.16$} & 7.97{\scriptsize $\pm 00.12$} & 8.07{\scriptsize $\pm 00.15$} & 8.03{\scriptsize $\pm 00.05$} & 8.14{\scriptsize $\pm 00.04$} & 7.98{\scriptsize $\pm 00.18$} & 8.04{\scriptsize $\pm 00.24$} & 7.84{\scriptsize $\pm 00.13$} \\
    NBFNet &  10.25{\scriptsize $\pm 01.24$} & 9.53{\scriptsize $\pm 00.85$} & 8.15{\scriptsize $\pm 01.21$} & 4.32{\scriptsize $\pm 00.26$} & 10.33{\scriptsize $\pm 02.45$} & 8.97{\scriptsize $\pm 01.24$} & 9.29{\scriptsize $\pm 01.38$} & 14.54{\scriptsize $\pm 04.76$} \\
    RMPI & 12.45{\scriptsize $\pm 01.90$} & 12.10{\scriptsize $\pm 02.71$} & 11.69{\scriptsize $\pm 04.37$} & 10.28{\scriptsize $\pm 01.28$} & N/A* & 8.54{\scriptsize $\pm 02.70$} & 17.89{\scriptsize $\pm 12.22$} & 6.53{\scriptsize $\pm 02.16$} \\
    \midrule
    InGram \update{(d-DEq)} & 50.03{\scriptsize $\pm 05.32$} & 26.31{\scriptsize $\pm 08.27$} & 21.32{\scriptsize $\pm 07.84$} & 29.81{\scriptsize $\pm 14.21$} & 48.70{\scriptsize $\pm 10.06$} & 38.81{\scriptsize $\pm 03.10$} & 29.94{\scriptsize $\pm 13.28$} & 32.26{\scriptsize $\pm 13.97$} \\
    Ultra \update{(DEq)}
    &  \textbf{82.50}{\scriptsize $\pm 02.44$} 
    &  \textbf{72.99}{\scriptsize $\pm 07.41$} 
    &  \textbf{76.16}{\scriptsize $\pm 07.00$} 
    &  \textbf{91.05}{\scriptsize $\pm 01.53$} 
    &  \textbf{88.52}{\scriptsize $\pm 03.82$} 
    &  \textbf{75.30}{\scriptsize $\pm 04.87$} 
    &  \underline{47.42}{\scriptsize $\pm 12.54$} 
    &  \textbf{80.38}{\scriptsize $\pm 01.85$} 
    \\
    %
    
    \OurOtherModel \update{(DEq)} & \underline{73.38}{\scriptsize $\pm 05.77$} & 41.61{\scriptsize $\pm 10.12$} & 46.86{\scriptsize $\pm 09.11$} & 40.56{\scriptsize $\pm 14.80$} & \underline{80.74}{\scriptsize $\pm 04.47$} & 66.06{\scriptsize $\pm 02.91$} & 39.51{\scriptsize $\pm 16.76$} & 49.10{\scriptsize $\pm 05.43$} \\

    \OurNewModel \update{(DEq)} & 72.96{\scriptsize $\pm 00.77$} & \underline{65.73}{\scriptsize $\pm 00.58$} & \underline{59.95}{\scriptsize $\pm 03.91$} & \underline{84.71}{\scriptsize $\pm 01.11$} & 71.47{\scriptsize $\pm 00.31$} & \underline{71.47}{\scriptsize $\pm 00.69$} & \textbf{66.48}{\scriptsize $\pm 06.75$} & \underline{67.36}{\scriptsize $\pm 00.43$}
\end{tabular}
\label{tab:pediatypes-relation-mrr}
}
\vspace{5pt} %
{
\subcaption{{\bf {\scriptsize Node prediction $(i, k, ?)$ performance in \%. Higher $\uparrow$ is better.}}}
\resizebox{\linewidth}{!}{
\begin{tabular}{l r r r r r r r r}
    Models & EN-FR & FR-EN & EN-DE & DE-EN & DB-WD & WD-DB & DB-YG & YG-DB \\
    \midrule
    
    Rand & 8.86{\scriptsize $\pm 00.00$} & 8.86{\scriptsize $\pm 00.00$} & 8.86{\scriptsize $\pm 00.00$} & 8.86{\scriptsize $\pm 00.00$} & 8.86{\scriptsize $\pm 00.00$} & 8.86{\scriptsize $\pm 00.00$} & 8.86{\scriptsize $\pm 00.00$} & 8.86{\scriptsize $\pm 00.00$}\\
    \midrule
    
    GAT & 51.43{\scriptsize $\pm 00.25$} & 49.48{\scriptsize $\pm 01.51$} & 26.22{\scriptsize $\pm 00.44$} & 25.45{\scriptsize $\pm 01.23$} & 16.87{\scriptsize $\pm 00.59$} & 34.66{\scriptsize $\pm 00.33$} & 37.22{\scriptsize $\pm 00.29$} & 45.96{\scriptsize $\pm 00.29$} \\
    GIN & 53.72{\scriptsize $\pm 03.45$} & 52.03{\scriptsize $\pm 03.38$} & 34.60{\scriptsize $\pm 07.43$} & 37.27{\scriptsize $\pm 09.42$} & 20.75{\scriptsize $\pm 07.22$} & 40.37{\scriptsize $\pm 08.20$} & 35.80{\scriptsize $\pm 01.36$} & 44.77{\scriptsize $\pm 00.92$} \\
    GraphConv & 63.72{\scriptsize $\pm 01.76$} & 57.77{\scriptsize $\pm 01.09$} & 48.18{\scriptsize $\pm 00.96$} & 45.18{\scriptsize $\pm 00.15$} & 22.49{\scriptsize $\pm 00.76$} & 50.30{\scriptsize $\pm 02.80$} & 38.71{\scriptsize $\pm 00.55$} & 50.54{\scriptsize $\pm 00.42$} \\
    NBFNet & 69.22{\scriptsize $\pm 02.44$} & \underline{74.01}{\scriptsize $\pm 01.41$} & {63.49}{\scriptsize $\pm 02.44$} & 38.86{\scriptsize $\pm 02.55$} & 41.26{\scriptsize $\pm 02.58$} & {64.02}{\scriptsize $\pm 01.25$} & 38.13{\scriptsize $\pm 01.11$} & 52.30{\scriptsize $\pm 02.09$}  \\
    RMPI & 63.02{\scriptsize $\pm 02.94$} & 43.72{\scriptsize $\pm 05.65$} & 44.82{\scriptsize $\pm 02.93$} & 46.84{\scriptsize $\pm 05.36$} & N/A* & 46.33{\scriptsize $\pm 08.76$} & \underline{43.00}{\scriptsize $\pm 03.70$} & 53.72{\scriptsize $\pm 01.84$}  \\
    \midrule
    InGram \update{(d-DEq)} & 71.23{\scriptsize $\pm 01.73$} & 55.67{\scriptsize $\pm 05.65$} & 55.94{\scriptsize $\pm 02.76$} & {61.15}{\scriptsize $\pm 01.42$} & 34.50{\scriptsize $\pm 08.47$} & 57.05{\scriptsize $\pm 03.73$} & 26.36{\scriptsize $\pm 04.73$} & \underline{56.23}{\scriptsize $\pm 01.56$} \\
    Ultra \update{(DEq)}
    &  \textbf{84.97}{\scriptsize $\pm 00.09$} 
    &  \textbf{88.20}{\scriptsize $\pm 00.25$} 
    &  \textbf{81.03}{\scriptsize $\pm 00.10$} 
    &  \textbf{72.80}{\scriptsize $\pm 00.26$} 
    &  \textbf{69.04}{\scriptsize $\pm 00.18$} 
    &  \textbf{81.43}{\scriptsize $\pm 00.16$} 
    &  \textbf{49.19}{\scriptsize $\pm 01.35$} 
    &  \textbf{60.24}{\scriptsize $\pm 00.74$} 
    \\
    %
    
    \OurOtherModel \update{(DEq)} & \underline{78.45}{\scriptsize $\pm 00.89$} & 68.59{\scriptsize $\pm 04.30$} & 66.13{\scriptsize $\pm 01.48$} & \underline{70.32}{\scriptsize $\pm 01.58$} & 44.71{\scriptsize $\pm 08.98$} & 69.23{\scriptsize $\pm 02.53$} & 35.67{\scriptsize $\pm 03.92$} & 48.07{\scriptsize $\pm 08.76$} \\

    \OurNewModel \update{(DEq)} & 74.95{\scriptsize $\pm 01.56$} & 57.17{\scriptsize $\pm 01.70$} & \underline{74.38}{\scriptsize $\pm 00.66$} & 62.62{\scriptsize $\pm 00.22$} & \underline{63.21}{\scriptsize $\pm 00.69$} & \underline{70.58}{\scriptsize $\pm 00.42$} & 34.79{\scriptsize $\pm 00.49$} & 40.71{\scriptsize $\pm 01.75$}
\end{tabular}
\label{tab:pediatypes-node-mrr}
}
}
}
\label{tab:pediatypes-mrr}
\end{table}

\begin{table}[t]
\centering
\captionof{table}{
{\bf Relation \& Node Hits@1 performance on \OurTask over \ourdata.} We report standard deviations over 5 runs. %
A higher value means better \ourtask performance.
\update{Models labeled with DEq are double equivariant models, and those labeled with d-DEq are distributionally double equivariant ones.}
``Rand'' column contains unbiased estimations of the performance from a random predictor.
\update{\bf Double equivariant models consistently achieve better results than non-double-equivariant models with generally smaller standard deviations.}
N/A*: Not available due to constant crashes.
}
{
\subcaption{{\bf {\scriptsize Relation prediction $(i, ?, j)$ performance in \%. Higher $\uparrow$ is better.}}}
\resizebox{\linewidth}{!}{
\begin{tabular}{l r r r r r r r r}
    Models & EN-FR & FR-EN & EN-DE & DE-EN & DB-WD & WD-DB & DB-YG & YG-DB \\
    \midrule
    
    Rand & 1.96{\scriptsize $\pm 00.00$} & 1.96{\scriptsize $\pm 00.00$} & 1.96{\scriptsize $\pm 00.00$} & 1.96{\scriptsize $\pm 00.00$} & 1.96{\scriptsize $\pm 00.00$} & 1.96{\scriptsize $\pm 00.00$} & 1.96{\scriptsize $\pm 00.00$} & 1.96{\scriptsize $\pm 00.00$}\\
    \midrule
    
    GAT & 1.07{\scriptsize $\pm 00.14$} & 1.01{\scriptsize $\pm 00.01$} & 1.03{\scriptsize $\pm 00.03$} & 1.11{\scriptsize $\pm 00.09$} & 1.07{\scriptsize $\pm 00.14$} & 0.99{\scriptsize $\pm 00.21$} & 0.96{\scriptsize $\pm 00.16$} & 1.09{\scriptsize $\pm 00.25$} \\
    GIN & 1.01{\scriptsize $\pm 00.03$} & 0.95{\scriptsize $\pm 00.08$} & 1.03{\scriptsize $\pm 00.06$} & 1.10{\scriptsize $\pm 00.06$} & 0.96{\scriptsize $\pm 00.15$} & 1.00{\scriptsize $\pm 00.15$} & 0.92{\scriptsize $\pm 00.15$} & 0.83{\scriptsize $\pm 00.17$} \\
    GraphConv & 0.91{\scriptsize $\pm 00.03$} & 0.97{\scriptsize $\pm 00.06$} & 1.05{\scriptsize $\pm 00.14$} & 1.01{\scriptsize $\pm 00.03$} & 1.09{\scriptsize $\pm 00.07$} & 0.91{\scriptsize $\pm 00.04$} & 0.94{\scriptsize $\pm 00.22$} & 0.88{\scriptsize $\pm 00.20$} \\
    NBFNet & 4.43{\scriptsize $\pm 01.24$} & 3.62{\scriptsize $\pm 01.01$} & 2.49{\scriptsize $\pm 01.23$} & 0.51{\scriptsize $\pm 00.18$} & 4.18{\scriptsize $\pm 02.17$} & 2.80{\scriptsize $\pm 00.83$} & 1.63{\scriptsize $\pm 00.89$} & 7.30{\scriptsize $\pm 05.01$} \\
    RMPI & 3.92{\scriptsize $\pm 02.08$} & 4.04{\scriptsize $\pm 01.83$} & 3.37{\scriptsize $\pm 02.20$} & 2.13{\scriptsize $\pm 00.79$} & N/A* & 2.39{\scriptsize $\pm 02.35$} & 7.36{\scriptsize $\pm 09.03$} & 0.91{\scriptsize $\pm 00.92$} \\
    \midrule
    InGram \update{(d-DEq)} & 35.19{\scriptsize $\pm 07.73$} & 12.40{\scriptsize $\pm 07.55$} & 8.45{\scriptsize $\pm 06.57$} & 16.46{\scriptsize $\pm 16.33$} & 33.66{\scriptsize $\pm 12.09$} & 25.69{\scriptsize $\pm 03.88$} & 14.24{\scriptsize $\pm 12.00$} & 15.83{\scriptsize $\pm 12.59$} \\
    Ultra \update{(DEq)}
    &  \textbf{78.64}{\scriptsize $\pm 03.16$} 
    &  \textbf{66.58}{\scriptsize $\pm 09.50$} 
    &  \textbf{70.53}{\scriptsize $\pm 08.99$} 
    &  \textbf{89.25}{\scriptsize $\pm 01.90$} 
    &  \textbf{86.07}{\scriptsize $\pm 04.78$} 
    &  \textbf{70.47}{\scriptsize $\pm 05.95$} 
    &  \underline{35.73}{\scriptsize $\pm 15.06$} 
    &  \textbf{76.42}{\scriptsize $\pm 02.37$} 
    \\
    %
    
    \OurOtherModel \update{(DEq)} & \underline{65.26}{\scriptsize $\pm 10.23$} & 26.90{\scriptsize $\pm 12.97$} & 36.80{\scriptsize $\pm 11.16$} & 25.34{\scriptsize $\pm 18.48$} & \underline{75.00}{\scriptsize $\pm 06.42$} & \underline{60.35}{\scriptsize $\pm 02.56$} & 24.28{\scriptsize $\pm 14.29$} & 30.82{\scriptsize $\pm 10.43$} \\

    \OurNewModel \update{(DEq)} & 58.43{\scriptsize $\pm 01.29$} & \underline{48.68}{\scriptsize $\pm 00.96$} & \underline{37.29}{\scriptsize $\pm 05.11$} & \underline{75.08}{\scriptsize $\pm 01.99$} & 57.05{\scriptsize $\pm 00.92$} & 56.00{\scriptsize $\pm 01.17$} & \textbf{59.36}{\scriptsize $\pm 07.96$} & \underline{49.41}{\scriptsize $\pm 00.85$}
\end{tabular}
\label{tab:pediatypes-relation-hit1}
}
\vspace{5pt} %
{
\subcaption{{\bf {\scriptsize Node prediction $(i, k, ?)$ performance in \%. Higher $\uparrow$ is better.}}}
\resizebox{\linewidth}{!}{
\begin{tabular}{l r r r r r r r r}
    Models & EN-FR & FR-EN & EN-DE & DE-EN & DB-WD & WD-DB & DB-YG & YG-DB \\
    \midrule
    
   Rand & 1.96{\scriptsize $\pm 00.00$} & 1.96{\scriptsize $\pm 00.00$} & 1.96{\scriptsize $\pm 00.00$} & 1.96{\scriptsize $\pm 00.00$} & 1.96{\scriptsize $\pm 00.00$} & 1.96{\scriptsize $\pm 00.00$} & 1.96{\scriptsize $\pm 00.00$} & 1.96{\scriptsize $\pm 00.00$}\\
    \midrule
    
    GAT & 31.80{\scriptsize $\pm 00.64$} & 30.19{\scriptsize $\pm 02.30$} & 10.23{\scriptsize $\pm 00.96$} & 8.68{\scriptsize $\pm 01.69$} & 7.98{\scriptsize $\pm 00.89$} & 16.26{\scriptsize $\pm 00.34$} & 26.09{\scriptsize $\pm 00.47$} & 33.06{\scriptsize $\pm 00.29$} \\
    GIN & 34.59{\scriptsize $\pm 04.64$} & 34.57{\scriptsize $\pm 05.26$} & 17.69{\scriptsize $\pm 07.91$} & 20.74{\scriptsize $\pm 10.01$} & 12.42{\scriptsize $\pm 06.59$} & 23.10{\scriptsize $\pm 09.67$} & 23.72{\scriptsize $\pm 01.62$} & 32.26{\scriptsize $\pm 01.89$} \\
    GraphConv & 47.48{\scriptsize $\pm 02.60$} & 40.37{\scriptsize $\pm 01.52$} & 31.96{\scriptsize $\pm 01.02$} & 28.46{\scriptsize $\pm 00.13$} & 12.53{\scriptsize $\pm 00.34$} & 35.82{\scriptsize $\pm 03.54$} & 24.12{\scriptsize $\pm 00.80$} & 37.05{\scriptsize $\pm 00.51$} \\
    NBFNet & 64.17{\scriptsize $\pm 02.68$} & \underline{69.68}{\scriptsize $\pm 01.63$} & 57.50{\scriptsize $\pm 02.66$} & 32.26{\scriptsize $\pm 02.81$} & 34.56{\scriptsize $\pm 02.54$} & 59.70{\scriptsize $\pm 01.38$} & \underline{33.32}{\scriptsize $\pm 01.11$} & \underline{47.47}{\scriptsize $\pm 02.08$} \\
    RMPI & 48.27{\scriptsize $\pm 03.74$} & 26.92{\scriptsize $\pm 04.87$} & 27.38{\scriptsize $\pm 03.09$} & 29.60{\scriptsize $\pm 04.77$} & N/A* & 34.81{\scriptsize $\pm 08.97$} & {33.29}{\scriptsize $\pm 03.20$} & {42.14}{\scriptsize $\pm 02.87$} \\
    \midrule
    InGram \update{(d-DEq)} & 60.00{\scriptsize $\pm 02.06$} & 41.59{\scriptsize $\pm 06.37$} & 39.05{\scriptsize $\pm 02.99$} & {45.44}{\scriptsize $\pm 01.69$} & 22.06{\scriptsize $\pm 08.10$} & 42.54{\scriptsize $\pm 04.50$} & 13.47{\scriptsize $\pm 03.50$} & 20.09{\scriptsize $\pm 04.96$} \\
    Ultra \update{(DEq)}
    &  \textbf{82.21}{\scriptsize $\pm 00.12$} 
    &  \textbf{86.39}{\scriptsize $\pm 00.30$} 
    &  \textbf{77.61}{\scriptsize $\pm 00.13$} 
    &  \textbf{67.74}{\scriptsize $\pm 00.32$} 
    &  \textbf{63.87}{\scriptsize $\pm 00.19$} 
    &  \textbf{78.43}{\scriptsize $\pm 00.20$} 
    &  \textbf{42.13}{\scriptsize $\pm 01.31$} 
    &  \textbf{54.25}{\scriptsize $\pm 00.70$} 
    \\
    %
    
    \OurOtherModel \update{(DEq)} & \underline{69.46}{\scriptsize $\pm 01.12$} & 57.65{\scriptsize $\pm 05.54$} & {51.93}{\scriptsize $\pm 01.88$} & \underline{57.06}{\scriptsize $\pm 01.96$} & 32.12{\scriptsize $\pm 09.51$} & {57.84}{\scriptsize $\pm 03.28$} & 20.49{\scriptsize $\pm 03.35$} & 33.01{\scriptsize $\pm 08.87$} \\

    \OurNewModel \update{(DEq)} & 62.17{\scriptsize $\pm 02.38$} & 44.67{\scriptsize $\pm 01.92$} & \underline{63.36}{\scriptsize $\pm 00.78$} & {47.78}{\scriptsize $\pm 00.48$} & \underline{51.76}{\scriptsize $\pm 00.86$} & \underline{60.15}{\scriptsize $\pm 00.69$} & 20.74{\scriptsize $\pm 00.28$} & 26.24{\scriptsize $\pm 02.03$}
\end{tabular}
\label{tab:pediatypes-node-hit1}
}
}
}
\label{tab:pediatypes-hit1}
\end{table}

\begin{table}[t]
\centering
\captionof{table}{
{\bf Relation \& Node Hits@5 performance on \OurTask over \ourdata.} We report standard deviations over 5 runs. %
A higher value means better \ourtask performance. 
\update{Models labeled with DEq are double equivariant models, and those labeled with d-DEq are distributionally double equivariant ones.}
``Rand'' column contains unbiased estimations of the performance from a random predictor.
\update{\bf Double equivariant models consistently achieve better results than non-double-equivariant models with generally smaller standard deviations.}
N/A*: Not available due to constant crashes.
}
{
\subcaption{{\bf {\scriptsize Relation prediction $(i, ?, j)$ performance in \%. Higher $\uparrow$ is better.}}}
\resizebox{\linewidth}{!}{
\begin{tabular}{l r r r r r r r r}
    Models & EN-FR & FR-EN & EN-DE & DE-EN & DB-WD & WD-DB & DB-YG & YG-DB \\
    \midrule
    
    Rand & 9.80{\scriptsize $\pm 00.00$} & 9.80{\scriptsize $\pm 00.00$} & 9.80{\scriptsize $\pm 00.00$} & 9.80{\scriptsize $\pm 00.00$} & 9.80{\scriptsize $\pm 00.00$} & 9.80{\scriptsize $\pm 00.00$} & 9.80{\scriptsize $\pm 00.00$} & 9.80{\scriptsize $\pm 00.00$}\\
    \midrule
    
    GAT & 9.08{\scriptsize $\pm 00.39$} & 8.63{\scriptsize $\pm 00.25$} & 9.47{\scriptsize $\pm 00.18$} & 9.20{\scriptsize $\pm 00.24$} & 8.95{\scriptsize $\pm 00.36$} & 8.63{\scriptsize $\pm 00.29$} & 9.58{\scriptsize $\pm 00.50$} & 9.16{\scriptsize $\pm 00.23$} \\
    GIN & 9.09{\scriptsize $\pm 00.16$} & 9.31{\scriptsize $\pm 00.15$} & 9.18{\scriptsize $\pm 00.28$} & 9.23{\scriptsize $\pm 00.34$} & 9.12{\scriptsize $\pm 00.12$} & 8.85{\scriptsize $\pm 00.56$} & 8.53{\scriptsize $\pm 00.66$} & 8.61{\scriptsize $\pm 00.34$} \\
    GraphConv & 8.97{\scriptsize $\pm 00.66$} & 8.74{\scriptsize $\pm 00.26$} & 9.23{\scriptsize $\pm 00.11$} & 8.82{\scriptsize $\pm 00.10$} & 9.17{\scriptsize $\pm 00.29$} & 9.11{\scriptsize $\pm 00.50$} & 9.01{\scriptsize $\pm 00.72$} & 8.73{\scriptsize $\pm 00.15$} \\
    NBFNet &  12.94{\scriptsize $\pm 01.77$} & 12.46{\scriptsize $\pm 01.40$} & 8.56{\scriptsize $\pm 01.67$} & 2.68{\scriptsize $\pm 00.72$} & 13.44{\scriptsize $\pm 04.02$} & 11.74{\scriptsize $\pm 03.02$} & 11.95{\scriptsize $\pm 03.78$} & 20.37{\scriptsize $\pm 05.90$} \\
    RMPI & 16.39{\scriptsize $\pm 04.15$} & 15.76{\scriptsize $\pm 04.58$} & 15.86{\scriptsize $\pm 08.05$} & 12.56{\scriptsize $\pm 02.70$} & N/A* & 8.91{\scriptsize $\pm 03.51$}  & 24.25{\scriptsize $\pm 19.24$} & 4.98{\scriptsize $\pm 03.08$} \\
    \midrule
    InGram \update{(d-DEq)} & 67.15{\scriptsize $\pm 05.04$} & 37.86{\scriptsize $\pm 14.41$} & 30.99{\scriptsize $\pm 11.82$} & 40.00{\scriptsize $\pm 13.02$} & 65.80{\scriptsize $\pm 09.59$} & 51.66{\scriptsize $\pm 03.57$} & 43.27{\scriptsize $\pm 19.30$} & 51.54{\scriptsize $\pm 26.09$}  \\
    Ultra \update{(DEq)}
    &  \textbf{98.32}{\scriptsize $\pm 00.07$} 
    &  \textbf{97.01}{\scriptsize $\pm 01.24$} 
    &  \textbf{98.78}{\scriptsize $\pm 00.35$} 
    &  \textbf{98.41}{\scriptsize $\pm 00.50$} 
    &  \textbf{97.83}{\scriptsize $\pm 00.61$} 
    &  \underline{92.22}{\scriptsize $\pm 01.77$} 
    &  \textbf{75.95}{\scriptsize $\pm 07.70$} 
    &  \textbf{95.85}{\scriptsize $\pm 00.03$} 
    \\
    %
    
    \OurOtherModel \update{(DEq)} & 83.23{\scriptsize $\pm 05.64$} & {59.83}{\scriptsize $\pm 11.57$} & {54.30}{\scriptsize $\pm 08.25$} & {57.65}{\scriptsize $\pm 15.74$} & {87.08}{\scriptsize $\pm 02.55$} & {70.79}{\scriptsize $\pm 03.80$} & {51.45}{\scriptsize $\pm 29.14$} & {75.85}{\scriptsize $\pm 07.26$} \\

    \OurNewModel \update{(DEq)} & \underline{93.79}{\scriptsize $\pm 00.20$} & \underline{91.21}{\scriptsize $\pm 00.25$} & \underline{94.04}{\scriptsize $\pm 01.21$} & \underline{96.64}{\scriptsize $\pm 00.13$} & \underline{90.83}{\scriptsize $\pm 02.29$} & \textbf{93.97}{\scriptsize $\pm 00.45$} & \underline{74.27}{\scriptsize $\pm 06.29$} & \underline{92.75}{\scriptsize $\pm 00.40$}
\end{tabular}
\label{tab:pediatypes-relation-hit5}
}
\vspace{5pt} %
{
\subcaption{{\bf {\scriptsize Node prediction $(i, k, ?)$ performance in \%. Higher $\uparrow$ is better.}}}
\resizebox{\linewidth}{!}{
\begin{tabular}{l r r r r r r r r}
    Models & EN-FR & FR-EN & EN-DE & DE-EN & DB-WD & WD-DB & DB-YG & YG-DB \\
    \midrule
    
    Rand & 9.80{\scriptsize $\pm 00.00$} & 9.80{\scriptsize $\pm 00.00$} & 9.80{\scriptsize $\pm 00.00$} & 9.80{\scriptsize $\pm 00.00$} & 9.80{\scriptsize $\pm 00.00$} & 9.80{\scriptsize $\pm 00.00$} & 9.80{\scriptsize $\pm 00.00$} & 9.80{\scriptsize $\pm 00.00$}\\
    \midrule
    
    GAT & 78.49{\scriptsize $\pm 00.44$} & 74.70{\scriptsize $\pm 00.68$} & 42.17{\scriptsize $\pm 00.91$} & 42.39{\scriptsize $\pm 00.52$} & 20.96{\scriptsize $\pm 00.65$} & 57.26{\scriptsize $\pm 00.89$} & 46.92{\scriptsize $\pm 00.37$} & 59.20{\scriptsize $\pm 00.41$} \\
    GIN & 79.96{\scriptsize $\pm 01.88$} & 74.33{\scriptsize $\pm 01.16$} & 53.97{\scriptsize $\pm 07.61$} & 55.89{\scriptsize $\pm 10.06$} & 25.05{\scriptsize $\pm 09.23$} & 61.94{\scriptsize $\pm 06.71$} & 46.56{\scriptsize $\pm 01.37$} & 57.48{\scriptsize $\pm 00.35$} \\
    GraphConv & 85.21{\scriptsize $\pm 00.63$} & 80.67{\scriptsize $\pm 00.30$} & 67.76{\scriptsize $\pm 01.19$} & 64.97{\scriptsize $\pm 00.43$} & 28.37{\scriptsize $\pm 01.41$} & 67.36{\scriptsize $\pm 02.37$} & \underline{53.79}{\scriptsize $\pm 00.72$} & 64.13{\scriptsize $\pm 00.23$} \\
    NBFNet & 81.48{\scriptsize $\pm 02.24$} & \underline{85.15}{\scriptsize $\pm 01.06$} & 77.62{\scriptsize $\pm 02.41$} & 48.73{\scriptsize $\pm 02.59$} & 51.52{\scriptsize $\pm 03.21$} & {72.18}{\scriptsize $\pm 00.90$} & 44.01{\scriptsize $\pm 01.40$} & 60.34{\scriptsize $\pm 02.28$}  \\
    RMPI & 82.47{\scriptsize $\pm 02.25$} & 64.88{\scriptsize $\pm 07.62$} & 67.24{\scriptsize $\pm 04.38$} & 69.47{\scriptsize $\pm 06.60$} & N/A* & 60.11{\scriptsize $\pm 08.77$}  & 51.57{\scriptsize $\pm 05.03$} & \underline{66.67}{\scriptsize $\pm 01.28$} \\
    \midrule
    InGram \update{(d-DEq)} & {85.15}{\scriptsize $\pm 01.74$} & 72.32{\scriptsize $\pm 05.31$} & {78.84}{\scriptsize $\pm 02.86$} & {81.01}{\scriptsize $\pm 00.97$} & 45.96{\scriptsize $\pm 11.09$} & 74.88{\scriptsize $\pm 03.09$} & 37.49{\scriptsize $\pm 06.84$} & 50.66{\scriptsize $\pm 06.76$} \\
    Ultra \update{(DEq)}
    &  \textbf{96.61}{\scriptsize $\pm 00.04$} 
    &  \textbf{96.34}{\scriptsize $\pm 00.08$} 
    &  \textbf{95.07}{\scriptsize $\pm 00.09$} 
    &  \textbf{91.92}{\scriptsize $\pm 00.13$} 
    &  \textbf{83.82}{\scriptsize $\pm 00.32$} 
    &  \textbf{91.41}{\scriptsize $\pm 00.08$}  
    &  \textbf{66.92}{\scriptsize $\pm 01.61$} 
    &  \textbf{76.66}{\scriptsize $\pm 01.38$} 
    \\
    %
    
    \OurOtherModel \update{(DEq)} & {89.62}{\scriptsize $\pm 00.63$} & {81.54}{\scriptsize $\pm 02.82$} & {84.57}{\scriptsize $\pm 00.95$} & \underline{87.16}{\scriptsize $\pm 01.04$} & {57.44}{\scriptsize $\pm 09.14$} & {83.14}{\scriptsize $\pm 01.64$} & 51.77{\scriptsize $\pm 05.14$} & {65.33}{\scriptsize $\pm 09.57$} \\

    \OurNewModel \update{(DEq)} & \underline{92.45}{\scriptsize $\pm 00.73$} & 71.24{\scriptsize $\pm 02.13$} & \underline{89.98}{\scriptsize $\pm 00.96$} & {82.65}{\scriptsize $\pm 00.79$} & \underline{76.12}{\scriptsize $\pm 00.87$} & \underline{83.33}{\scriptsize $\pm 00.46$} & 48.04{\scriptsize $\pm 01.78$} & 57.92{\scriptsize $\pm 01.47$}
\end{tabular}
\label{tab:pediatypes-node-hit5}
}
}
}
\label{tab:pediatypes-hit5}
\end{table}

\subsubsection{\ourotherdata: Testing meta-learning and zero-shot transfer capabilities}
\label{appx:exp-wikitopics}

The WikiTopics dataset is created from the WikiData-5M~\citep{wang2021kepler} (under CC0 1.0 license).
Each node in the graphs of this dataset represents an entity described by an existing Wikipedia page, and each relation type corresponds to a particular relation between the entities, such as ``director of'' or ``designed by''.
The node and relation type indices are codenames that start with the prefix ``Q'' and ``P'' respectively, which are devoid of semantic meaning.
Nevertheless, WikiData-5M~\citep{wang2021kepler} provides aliases for all nodes and relation types that map their indices to textual descriptions, and we use these textual descriptions to group the relation types into 11 different topic groups, or domains (we do not however provide these textual descriptions to the models per the specification of the \ourtask task).
In total, WikiData-5M~\citep{wang2021kepler} contains 822 relation types. 
We create WikiTopics datasets from all 822 relation types, which comprise graphs with as many as 66 relation types. Each graph has a disjoint set of relation types from all other graphs.
\Cref{tab:wikitopics-topic-desc}~shows the 11 topics/domains of the WikiTopics dataset, each corresponding to a distinct KG with distinct relation types.
\begin{table}[h]
    \centering
    \caption{The 11 different topics/domains of the WikiTopics dataset.}
    \resizebox{0.7\linewidth}{!}{
    \begin{tabular}{cll}
        Domain KG index & Abbreviation & Description \\
        \midrule
        T1 & Art    & Art and Media Representation \\
        T2 & Award  & Award Nomination and Achievement \\
        T3 & Edu    & Education and Academia \\
        T4 & Health & Health, Medicine, and Genetics \\
        T5 & Infra  & Infrastructure and Transportation \\
        T6 & Loc    & Location and Administrative Entity \\
        T7 & Org    & Organization and Membership \\
        T8 & People & People and Social Relationship \\
        T9 & Science    & Science, Technology, and Language \\
        T10 & Sport & Sport, and Game Competition \\
        T11 & Tax   & Taxonomy and Biology \\
    \end{tabular}}
    \label{tab:wikitopics-topic-desc}
\end{table}

To control the overall size of the graphs in WikiTopics, we downsample $10,000$ nodes for each domain from the subgraph consisting of only the triplets with the relation types belonging to that domain.
We adopt the Forest Fire sampling procedure with burning probability $p=0.8$~\citep{SamplingGraph} implemented in the Little Ball of Fur Python package~\citep{rozemberczki2020little}.
We then split the downsampled domain KG into 90\% observable triplets and 10\% querying triplets to be predicted by the models.
When splitting, we ensure that the set of nodes in the querying triplets is a subset of those in the observable triplets.
This way, the model is not tasked with the impossible task of predicting relation types between orphaned nodes previously unseen in the observable part of the graph.
This is implemented via an iterative procedure, where we first sample a batch of missing triplets from the downsampled domain graph, then discard those that contain unseen nodes in the rest of the triplets, and repeat this process until the number of sampled triplets reaches 10\% of total triplets.
\Cref{fig:stats-wikitopics} shows the {\bf data statistics} of WikiTopics dataset. %

\begin{figure}
\centering
\begin{minipage}{0.7\textwidth}
\centering
\resizebox{\linewidth}{!}{
\begin{tabular}{l|ccccc}
\hline
& \#Nodes & \# Relations & \#Triplets (Obv.) & \#Triplets (Qry.) & Avg. Deg. \\
\hline
Art & 10000 & 45 & 28023 & 3113 & ~~6.23 \\
Award & 10000 & 10 & 25056 & 2783 & ~~5.57 \\
Edu & 10000 & 15 & 14193 & 1575 & ~~3.15 \\
Health & 10000 & 20 & 15337 & 1703 & ~~3.41 \\
Infra & 10000 & 27 & 21646 & 2405 & ~~4.81 \\
Loc & 10000 & 35 & 80269 & 8918 & 17.84 \\
Org & 10000 & 18 & 30214 & 3357 & ~~6.71 \\
People & 10000 & 25 & 58530 & 6503 & 13.01 \\
Sci & 10000 & 42 & 12516 & 1388 & ~~2.78 \\
Sport & 10000 & 20 & 46717 & 5190 & 10.38 \\
Tax & 10000 & 31 & 19416 & 2157 & ~~4.32 \\
\hline
\end{tabular}
}
\includegraphics[width=\linewidth]{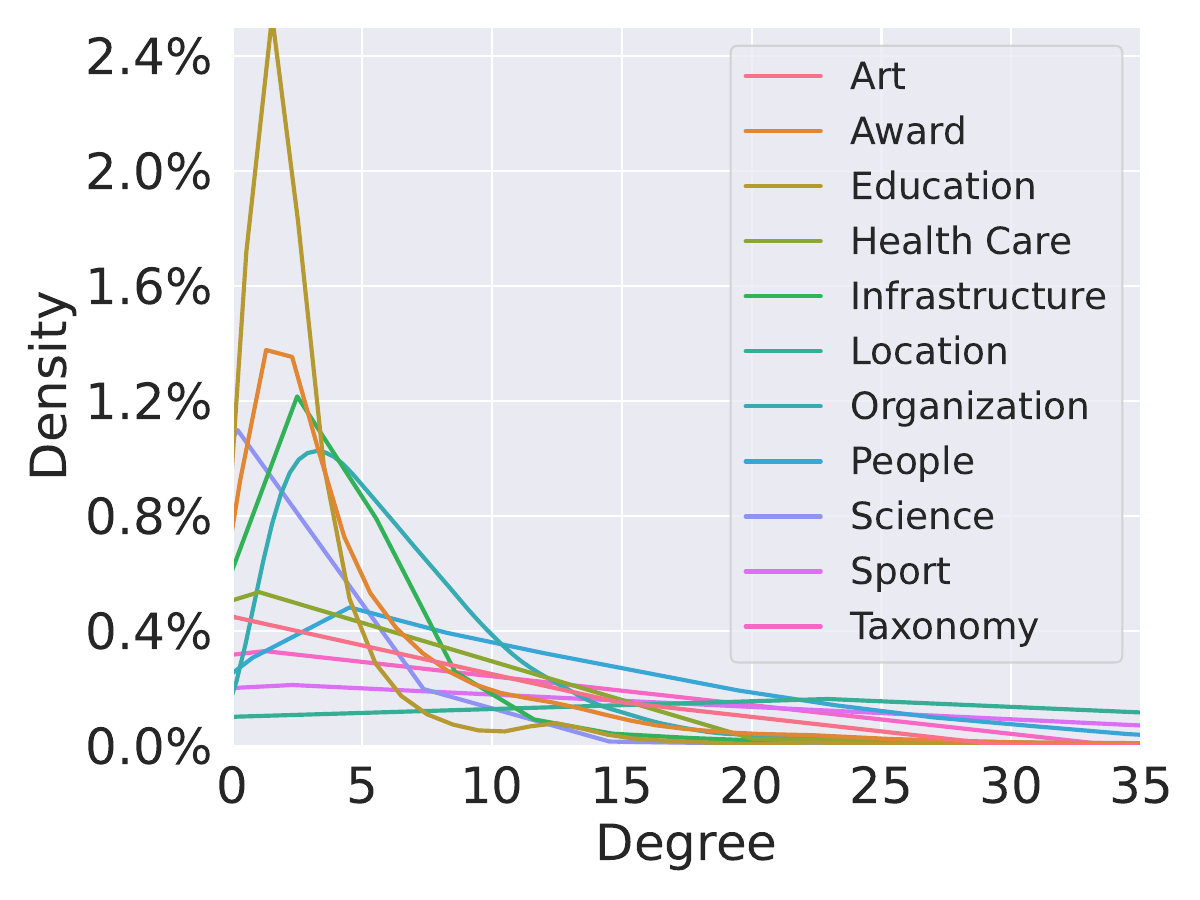}
\end{minipage}
\caption{{\bf Statistics of WikiTopics:} We report graph statistics including the number of nodes, number of relations, observed (obv.) triplets, querying (qry.) triplets, and average degree for each graph. We also report (in \& out) degree distribution on each graph at the bottom. We omit tail distribution larger than 35 since they are fairly small and almost flat.}
\label{fig:stats-wikitopics}
\end{figure}

\update{
\paragraph{Meta-learning on WikiTopics and zero-shot transfer to large-scale KGs}
The experiment described in \Cref{sec:negative-transfer} is conducted by training double equivariance models (Ultra~\citep{galkin2023towards}, DEq-Ingram, ISDEA+) on a mixture of WikiTopics KGs with an increasing number of training domains and test the models' average zero-shot performance on the remaining held-out domains. The results as shown in \Cref{tab:negative-transfer-node} indicate that current double equivariant models suffer from negative transfer (their performance drops when certain ``bad'' KG is included in the training mixture) and they have poor data scaling law (their performance saturates after 3 training domains). To investigate whether this phenomenon is consistent on other KGs that are larger in scale than WikiTopics, we conduct an additional experiment where we test the models' zero-shot performance on 3 commonly used large-scale KGs in the literature: FB15K237~\citep{toutanova2015observed}, NELL995~\citep{Xiong2017DeepPathAR}, and YAGO310~\citep{Mahdisoltani2015YAGO3AK}. \Cref{tab:wikitopics-large-inference-kg-size} shows the sizes of the 3 KGs respectively. In this experiment, we use the same training model checkpoints obtained for \Cref{tab:negative-transfer-node} and evaluate their zero-shot test performance on FB15K237, NELL995, and YAGO310. We report the zero-shot test Hits@1 accuracy for both relation prediction and node prediction tasks averaged across all 3 KGs. Note that this is still the fully inductive setting, since all of the 3 KGs are used for inference purpose only. The set of entities and relation types are distinct from that in the training WikiTOpics KGs.

\Cref{tab:negative-transfer-large-inference} shows the result of this experiment. These new results on large-scale KGs are consistent with what we learned from \Cref{tab:negative-transfer-node} in \Cref{sec:negative-transfer}. To emphasize, our observation is two-fold: First, same as what we observed in previous experiments, the domain {\bf Loc} is one example of a KG that elicits negative transfer effects from all existing double equivariant models. In general, whenever {\bf Loc} is included in the training data mix (E.g. Infra+Sci+Sport vs Infra+Sci+Loc), the model’s zero-shot test performance tends to deteriorate (E.g. a drop of -21.77\% for Ultra). This trend is more prominent on the relation prediction task. Second, all existing models demonstrate relatively poor data scaling law, especially when the number of training data domains increases from 3 (Infra + Sci + Sport) to 4 (Infra + Sci + Sport + Tax). All models show saturated performance, with DEq-InGram in particular having a significant drop in the relation prediction task. 

These new results on large-scale KGs, together with our previous results  from \Cref{tab:negative-transfer-node}, highlight the gap of existing double equivariant models towards achieving a true “graph foundation model (GFM).” As {\bf Loc} appears to be a KG that has different patterns from other domain-specific KGs, an ideal GFM should be able to jointly learn these different patterns and enhance its zero-shot test performance on new domains, instead of being interfered with and forgetting what it has learned, as the current models do. Our experiments here therefore pinpoint a specific scenario in our benchmark in which future work that proposes new GFM architectures can attempt to address and evaluate their methods.

}

\begin{table}[h]
    \centering
    \caption{Sizes of FB15K237~\citep{toutanova2015observed}, NELL995~\citep{Xiong2017DeepPathAR}, and YAGO310~\citep{Mahdisoltani2015YAGO3AK} used as inference graphs in the large-scale WikiTopics Meta-learning experiment. Observable triplets are the number of edges in the input KG to the model. Test triplets are the number of target ground-truth edges to evaluate. The observable triplets are taken from the training split and the test triplets from the test split of the original datasets.}
    \resizebox{0.7\linewidth}{!}{
    \begin{tabular}{crrrr}
        Dataset & Entities & Relations & Observable Triplets & Test Triplets \\
        \midrule
        FB15K237 & 14541  & 237 & 272115 & 20466 \\
        NELL995 & 74536 & 200 & 149678 & 2818 \\
        YAGO310 & 123182 & 37 & 1079040 & 5000 
    \end{tabular}}
    \label{tab:wikitopics-large-inference-kg-size}
\end{table}

\begin{table}
    \centering
    \caption{Average zero-shot Hits@1 performance (in \%) of Ultra, \OurOtherModel, and \OurNewModel by training on increasingly larger mixtures of \ourotherdata KGs and test on FB15K237, NELL995, and YAGO310. (1) {\bf\em Loc} is one \ourotherdata KG that we identified to consistently trigger negative transfer effects among all models. {\bf Mixing {\em Loc} in training data consistently diminishes both the training and test performance across all double equivariant models.} We show in parenthesis the relative performance difference comparing the training combination with {\bf\em Loc} and a similar combination without {\bf\em Loc}. (2) {\bf All models exhibit limited performance improvement as the number of training KG domains increases.}}
    \label{tab:negative-transfer-large-inference}
    {
    \subcaption{Average zero-shot Hits@1 of Relation prediction $(i, ?, j)$ task}
    \resizebox{0.99\linewidth}{!}{
    \begin{tabular}{l l l l l l l l}
        \toprule
        Meta-Learning Training KGs & Neg. Transfer? & Ultra & \OurOtherModel & \OurNewModel \\
        \midrule
        Infra & & 59.64 & 10.30 & 65.73  \\
        {\bf Loc} & \checkmark & 14.90 {\scriptsize$(-75.01\%)$} & 10.79 {\scriptsize$(+4.76\%)$} & 46.02 {\scriptsize $(-29.99\%)$} \\
        \midrule
        Infra + Sci & & 69.08 & 16.06 & 62.47  \\
        Infra + {\bf Loc} & \checkmark & 54.93 {\scriptsize$(-20.48\%)$} & 13.99 {\scriptsize$(-12.89\%)$} & 58.21 {\scriptsize$(-6.82\%)$}  \\
        \midrule
        Infra + Sci + Sport & & 74.55 & 54.17 & 60.82 \\
        Infra + Sci + {\bf Loc} & \checkmark & 58.32 {\scriptsize $(-21.77\%)$} & 19.85 {\scriptsize $(-63.36\%)$} & 56.59 {\scriptsize $(-6.95\%)$}  \\
        \midrule
        Infra + Sci + Sport + Tax & & 65.99 & 25.70 & 61.23 \\
        Infra + Sci + Sport + {\bf Loc} & \checkmark & 49.26 {\scriptsize $(-25.35\%)$} & 19.91 {\scriptsize $(-22.53\%)$} & 60.27 {\scriptsize $(-1.57\%)$} \\
        \bottomrule
    \end{tabular}
    }
    }
    \vspace{5pt} %
    {
    \subcaption{Average zero-shot Hits@1 of Node prediction $(i, r, ?)$ task}
    \resizebox{0.99\linewidth}{!}{
    \begin{tabular}{l l l l l l l l}
        \toprule
        Meta-Learning Training KGs & Neg. Transfer? & Ultra & \OurOtherModel & \OurNewModel \\
        \midrule
        Infra & & 79.31 & 22.12  & 51.84 \\
        {\bf Loc} & \checkmark & 58.45 {\scriptsize $(-26.30\%)$} & 22.64 {\scriptsize $(+2.35\%)$} & 57.24 {\scriptsize $(+10.42\%)$} \\
        \midrule
        Infra + Sci & & 78.08 & 32.24 & 65.52 \\
        Infra + {\bf Loc} & \checkmark & 72.26 {\scriptsize $(-7.45\%)$} & 36.24 {\scriptsize $(+12.41\%)$} & 58.50 {\scriptsize $(-10.71\%)$} \\
        \midrule
        Infra + Sci + Sport & & 77.70 & 39.24 & 65.53  \\
        Infra + Sci + {\bf Loc} & \checkmark & 51.91 {\scriptsize $(-33.19\%)$} & 38.61 {\scriptsize $(-1.61\%)$} & 61.88 {\scriptsize $(-5.57\%)$} \\
        \midrule
        Infra + Sci + Sport + Tax & & 77.90 & 39.41 & 60.04 \\
        Infra + Sci + Sport + {\bf Loc} & \checkmark & 80.22 {\scriptsize $(+2.98\%)$} & 40.26 {\scriptsize $(+2.16\%)$} & 62.53 {\scriptsize $(+4.15\%)$} \\
        \bottomrule
    \end{tabular}
    }}
\end{table}

\paragraph{Cross-domain one-to-one zero-shot transfer over WikiTopics KGs}
\update{
In addition to the meta-learning experiments we reported in \Cref{tab:negative-transfer-node} and \Cref{tab:negative-transfer-large-inference}, we also tested double equivariant model's (Ultra~\citep{galkin2023towards}, DEq-Ingram, ISDEA+) zero-shot performance from each WikiTopic domain to another one domain.
}
In this setting, we train the models on each of the $11$ graphs for 5 random seeds, and for each trained model checkpoint, we cross-test it on all the other 10 graphs, resulting in a total of $550$ statistics. We report the mean results across random seeds in heatmaps.
We present a detailed results (heatmaps with values) of Node and Relation Hits@10, Hits@1, and MRR for WikiTopics in \Cref{fig:wikitopics-full,fig:wikitopics-node-full}. Due to the large number of runs ($11\times 10=110$ different train-test scenarios, each with 5 random seeds, resulting in a total of 550 runs) and the time constraints to run all baseline models, we perform the evaluation over only the three models (\OurNewModel, \OurOtherModel, and InGram) that are designed for our \ourtask task. \Cref{fig:wikitopics-full} shows that for the task of predicting missing relation types $(i, ?, j)$, \OurNewModel and \OurOtherModel are consistently better than InGram across all different metrics. Especially, the structural double equivariant \OurNewModel model exhibits more consistent results across different train-test scenarios than both \OurOtherModel and InGram, and achieves significantly better results in Hits@1 and MRR, showcasing its ability for \ourtask in a much harder evaluation scenario.
For the task of prediction missing nodes $(i, k, ?)$ as shown in~\Cref{fig:wikitopics-node-full}, \OurNewModel, \OurOtherModel, and InGram showcase comparable performance, whereas \OurNewModel exhibits more consistent results across different train-test scenarios than both \OurOtherModel and InGram. We also note that similar to the relation prediction task, \OurNewModel also exhibits the best performance in the Hits@1 metric for the node prediction task.

\begin{figure}[t]
    \centering
    \begin{subfigure}[b]{0.32\textwidth}
        \centering        
        \includegraphics[width=\linewidth,height=0.84\linewidth]{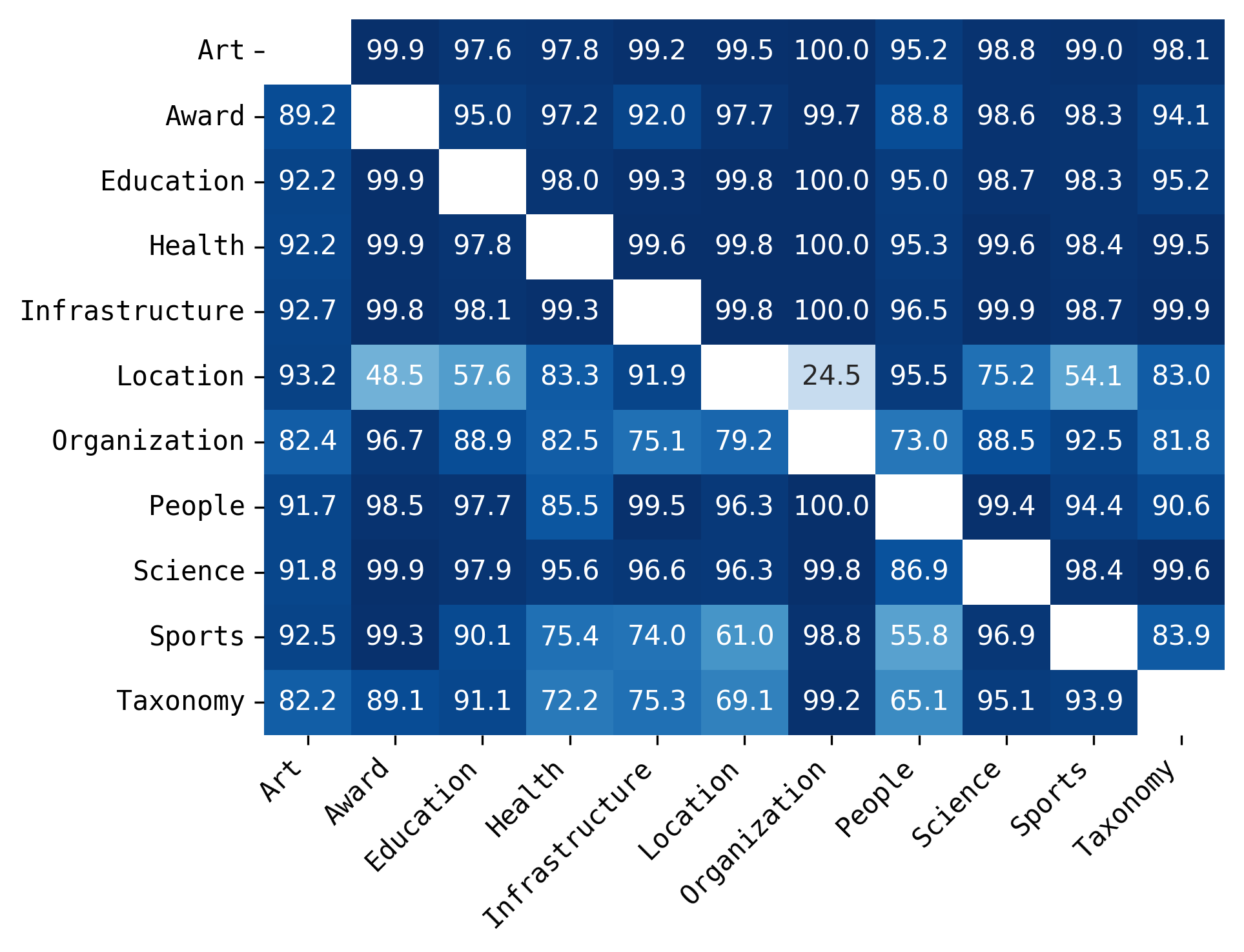}
        %
        \caption{\OurNewModel Hits@10}
    \end{subfigure}
    \hfill
    \begin{subfigure}[b]{0.27\textwidth}
        \centering
        \includegraphics[width=\linewidth,height=\linewidth]{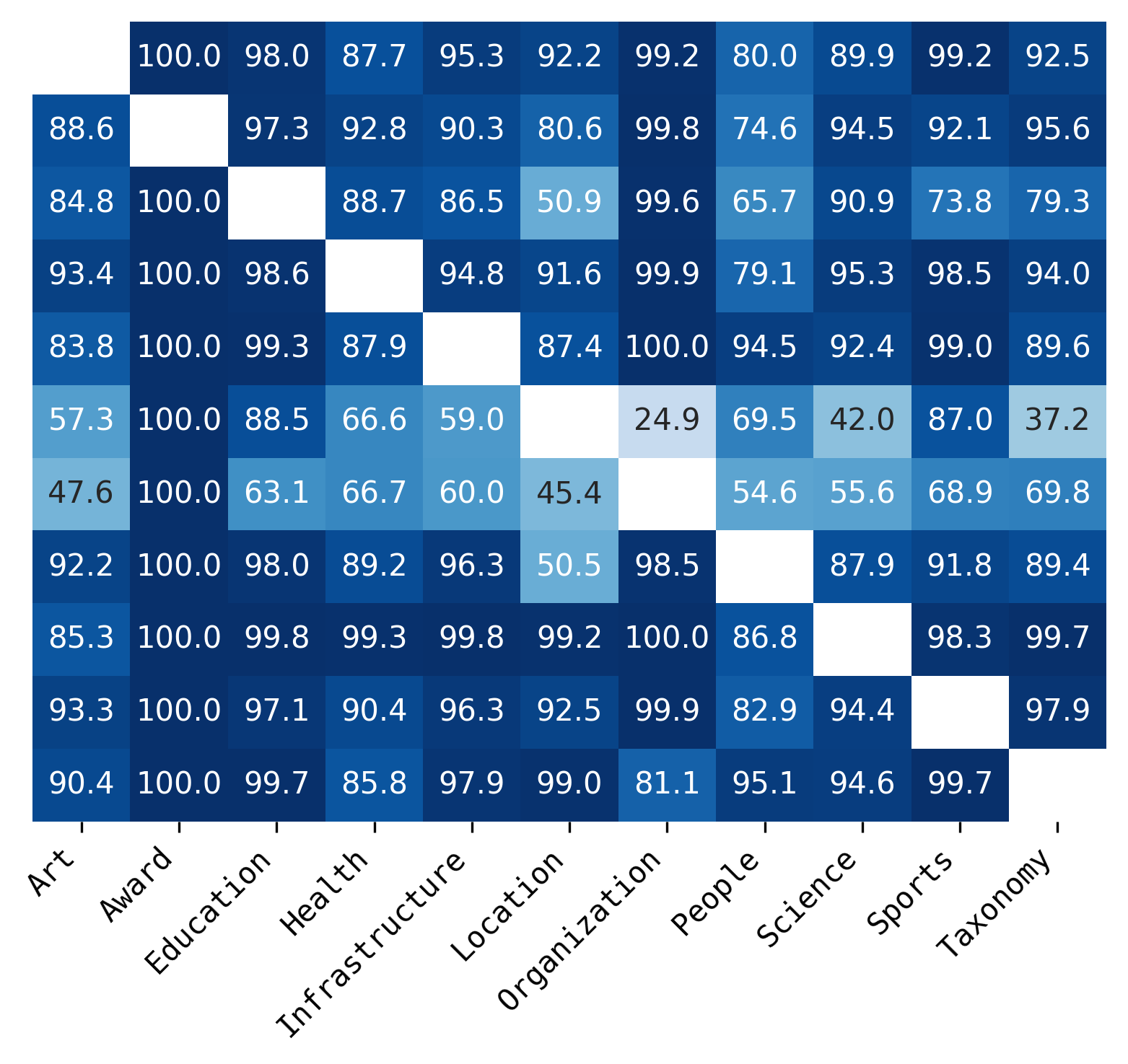}
        \caption{\OurOtherModel Hits@10}
    \end{subfigure}
    \hfill
    \begin{subfigure}[b]{0.31\textwidth}
        \centering
        \includegraphics[width=\linewidth,height=0.88\linewidth]{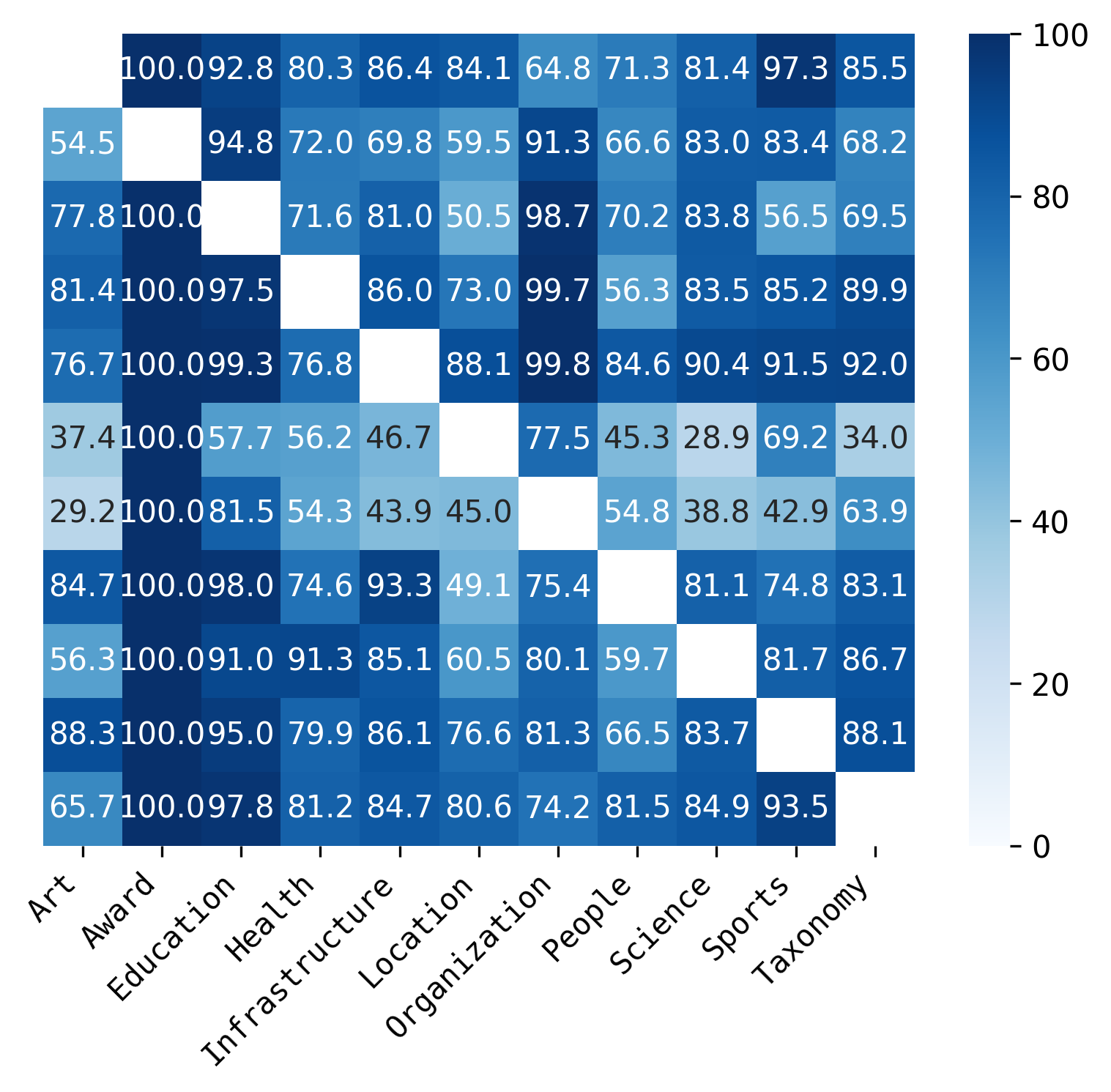}
        \caption{Original InGram Hits@10}
    \end{subfigure}
    \begin{subfigure}[b]{0.32\textwidth}
        \centering        
        \includegraphics[width=\linewidth,height=0.84\linewidth]{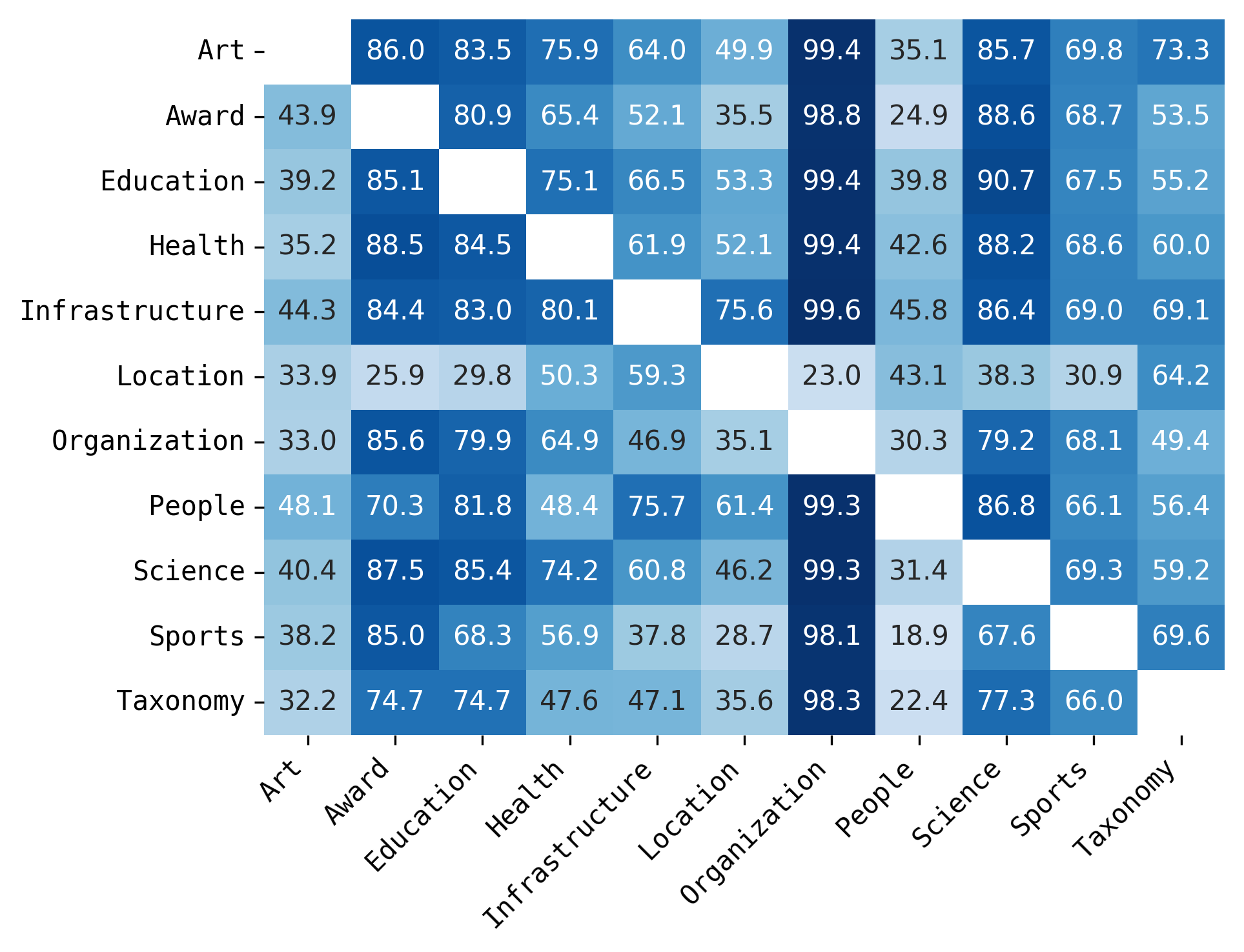}
        %
        \caption{\OurNewModel Hits@1}
    \end{subfigure}
    \hfill
    \begin{subfigure}[b]{0.27\textwidth}
        \centering
        \includegraphics[width=\linewidth,height=\linewidth]{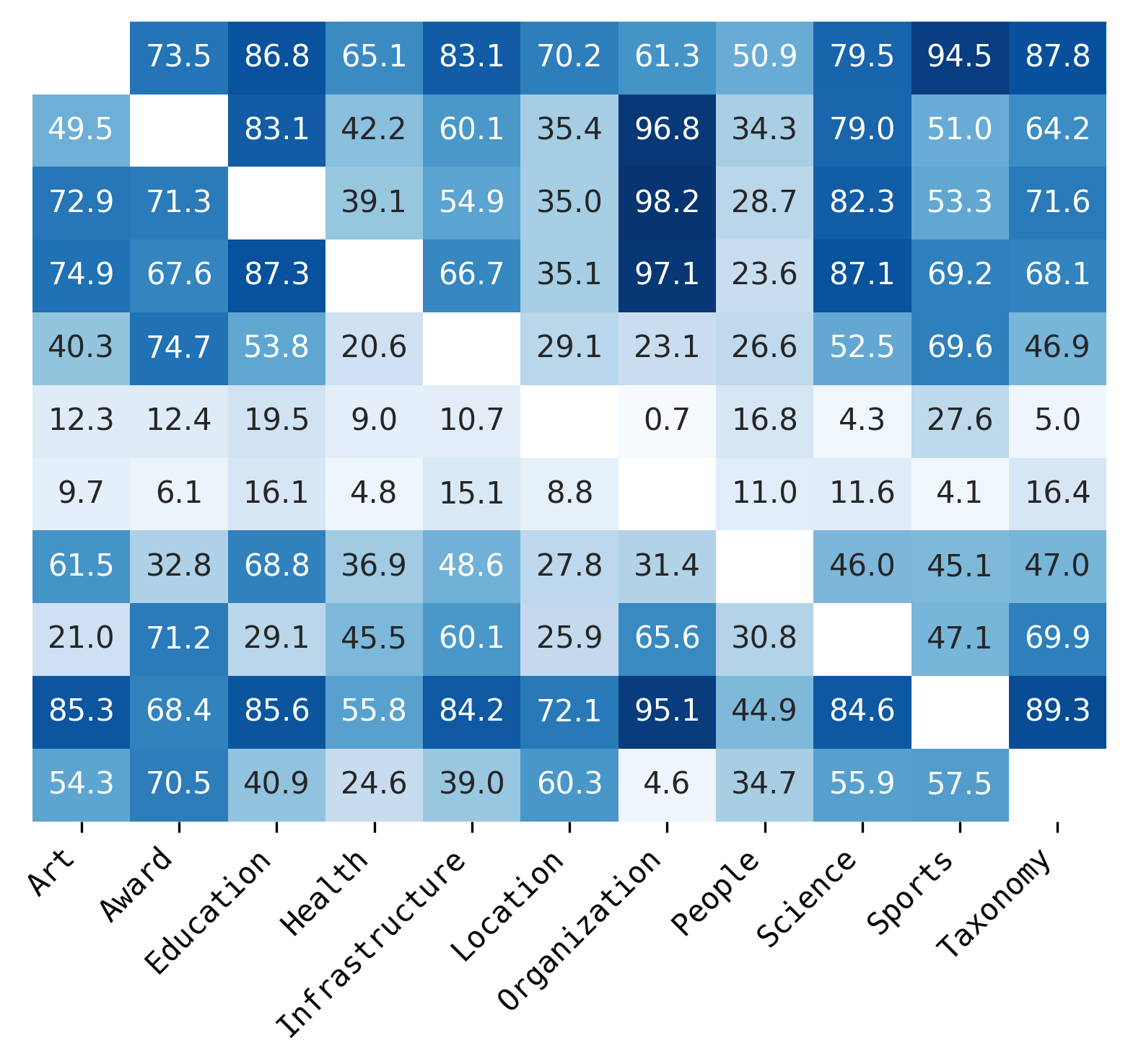}
        \caption{\OurOtherModel Hits@1}
    \end{subfigure}
    \hfill
    \begin{subfigure}[b]{0.31\textwidth}
        \centering
        \includegraphics[width=\linewidth,height=0.88\linewidth]{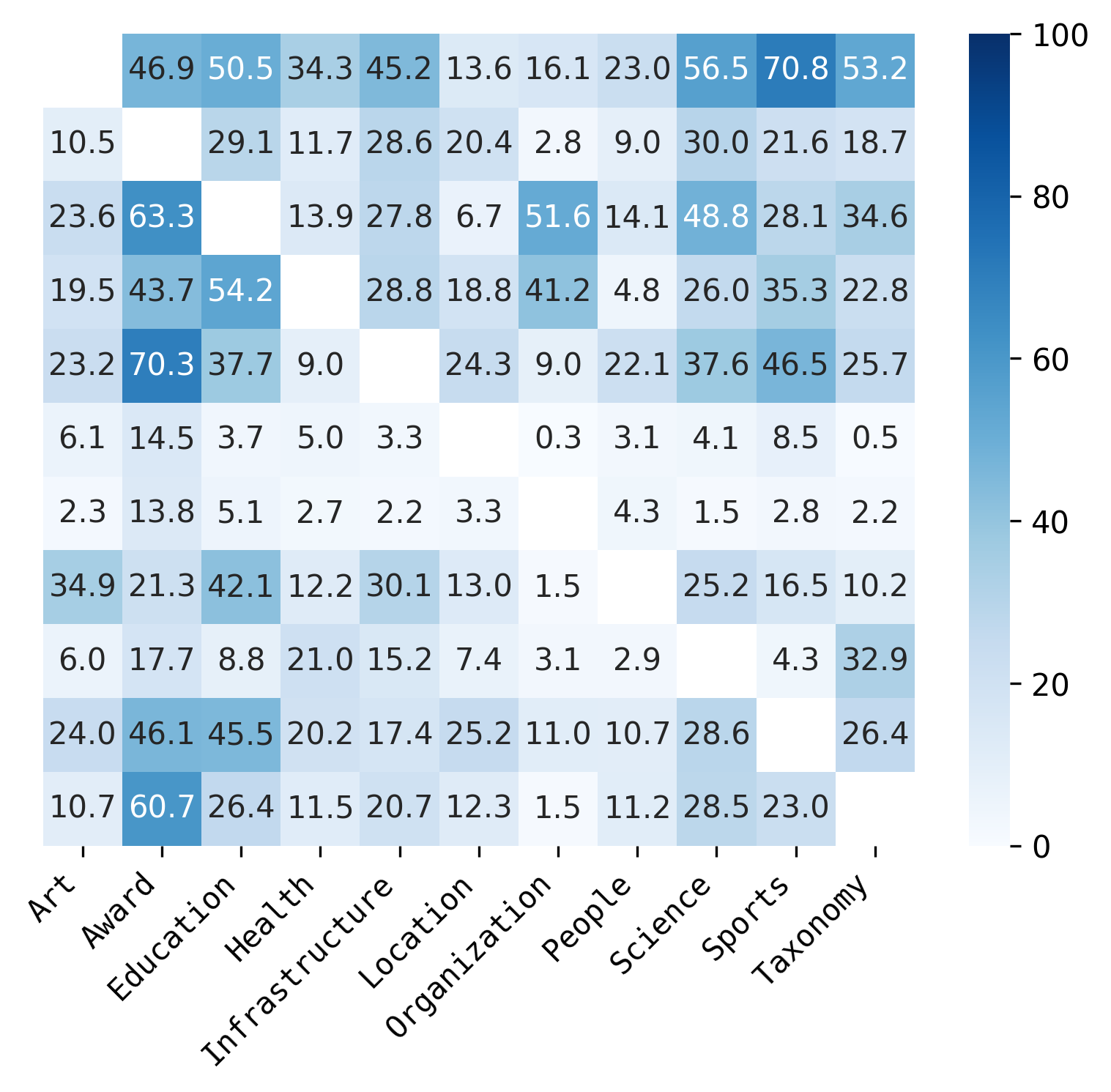}
        \caption{Original InGram Hits@1}
    \end{subfigure}
    \begin{subfigure}[b]{0.32\textwidth}
        \centering        
        \includegraphics[width=\linewidth,height=0.84\linewidth]{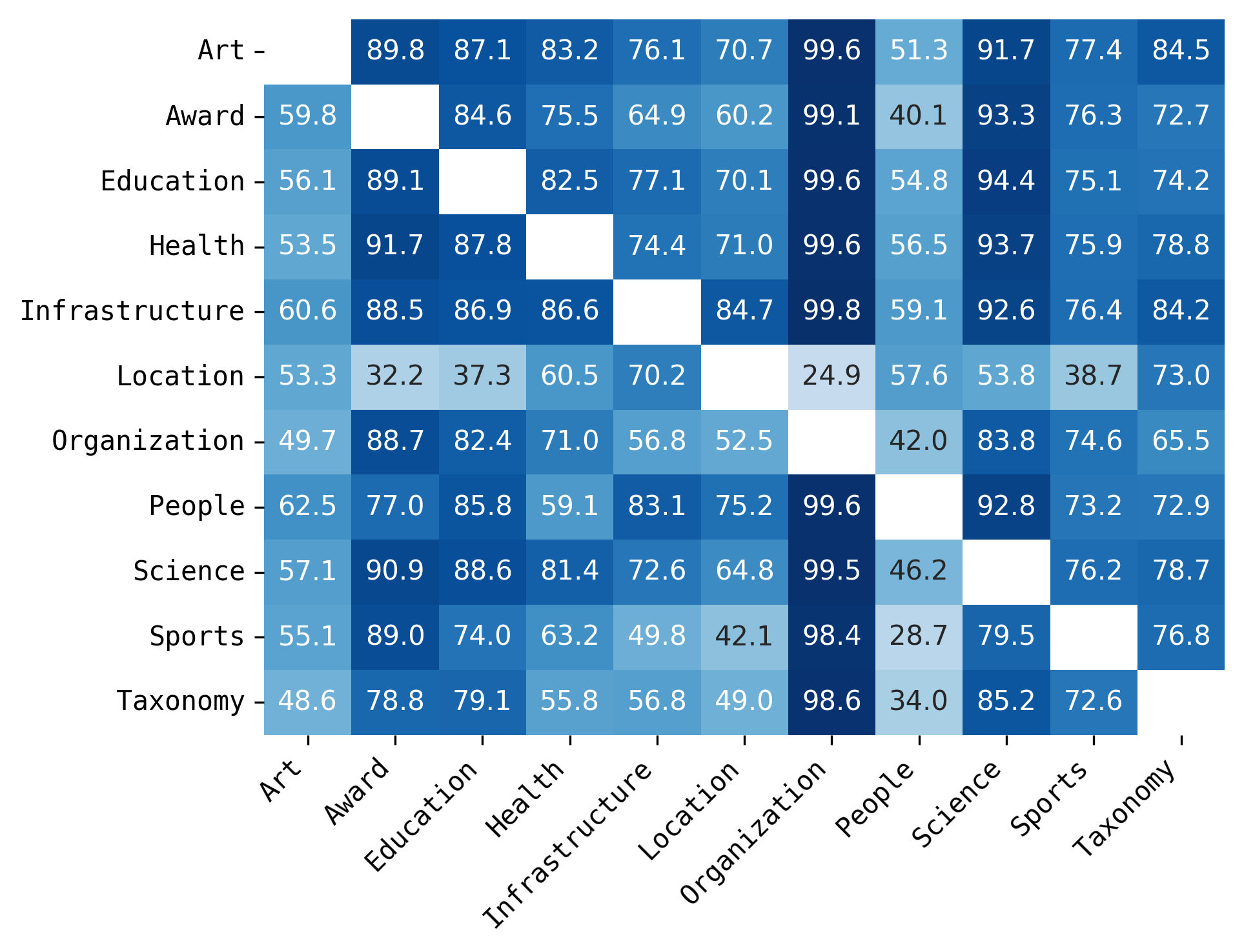}
        %
        \caption{\OurNewModel MRR}
    \end{subfigure}
    \hfill
    \begin{subfigure}[b]{0.27\textwidth}
        \centering
        \includegraphics[width=\linewidth,height=\linewidth]{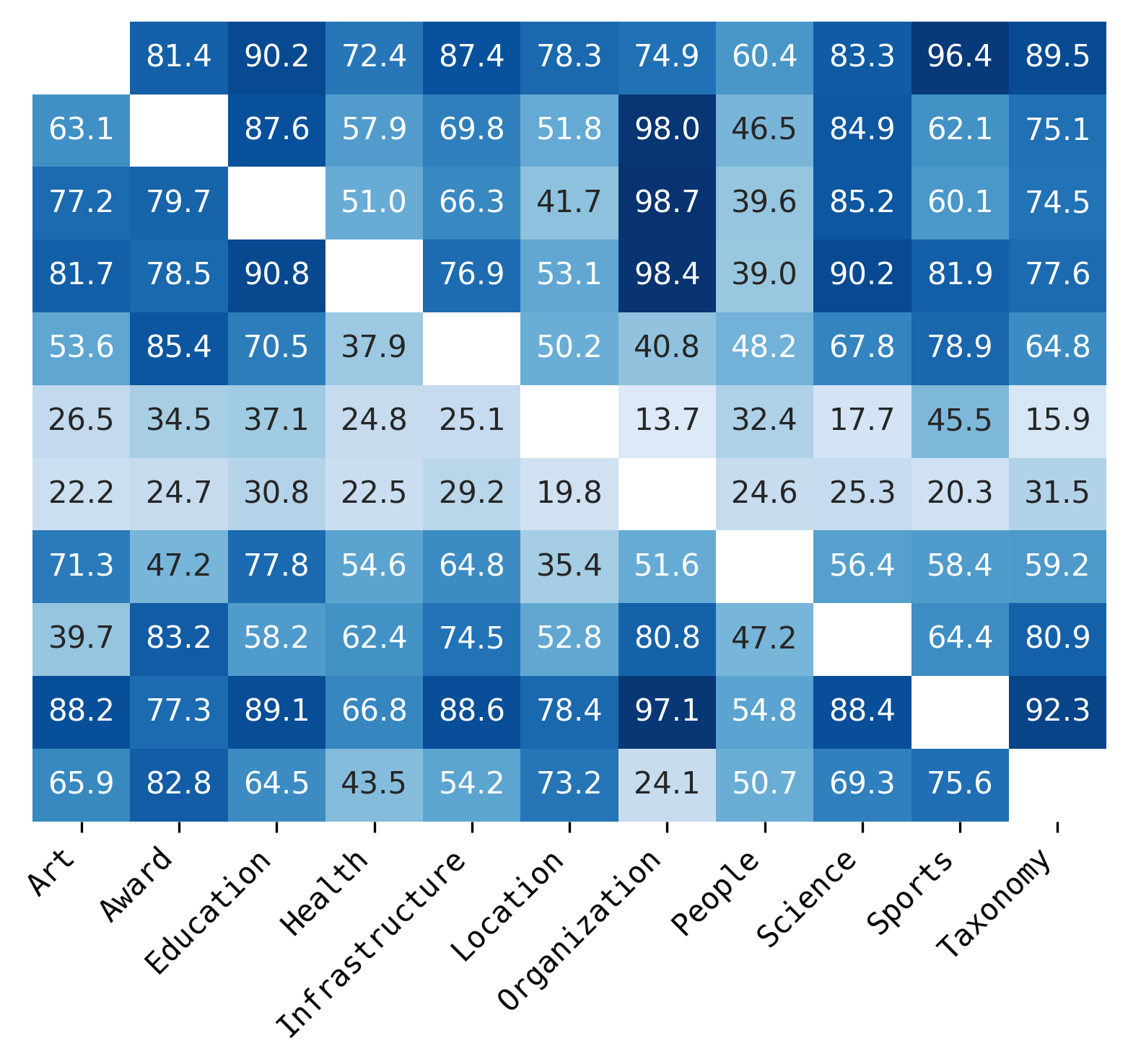}
        \caption{\OurOtherModel MRR}
    \end{subfigure}
    \hfill
    \begin{subfigure}[b]{0.31\textwidth}
        \centering
        \includegraphics[width=\linewidth,height=0.88\linewidth]{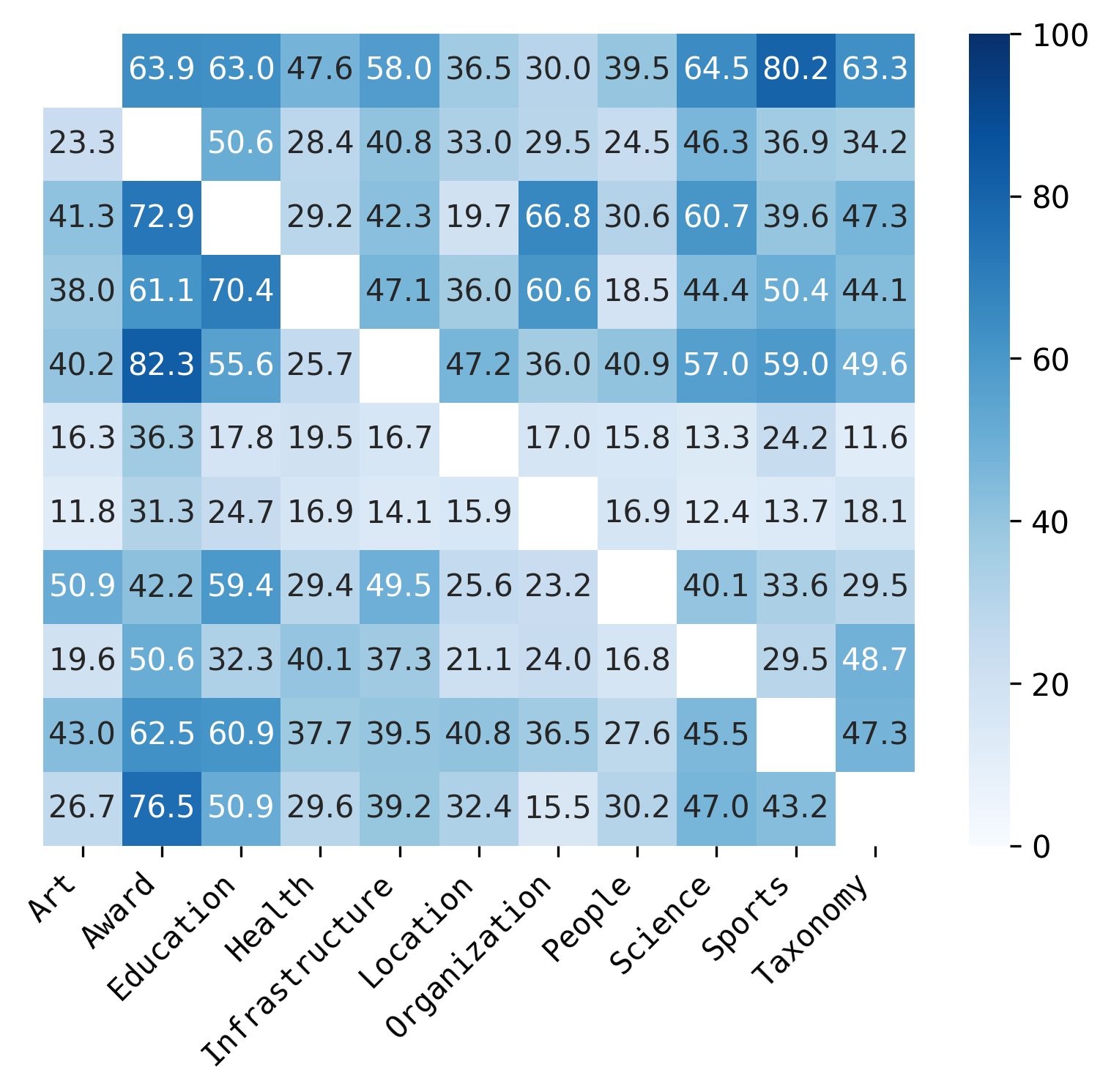}
        \caption{Original InGram MRR}
    \end{subfigure}
    %
    \caption{{\bf Relation prediction $(i, ?, j)$ performance over WikiTopics} for \OurNewModel, \OurOtherModel, and InGram~\citep{ingram}. Each row within each heatmap corresponds to a training graph, and each column within each heatmap corresponds to a test graph. A darker color means better performance. 
    \textbf{Both \OurModel, \OurOtherModel perform significantly better than InGram, especially for Hits@1 and MRR, whereas \OurNewModel exhibits more consistent results across different train-test scenarios than both \OurOtherModel and InGram.}
    }
    \label{fig:wikitopics-full}
\end{figure}

\begin{figure}[t]
    \centering
    \begin{subfigure}[b]{0.32\textwidth}
        \centering        
        \includegraphics[width=\linewidth,height=0.84\linewidth]{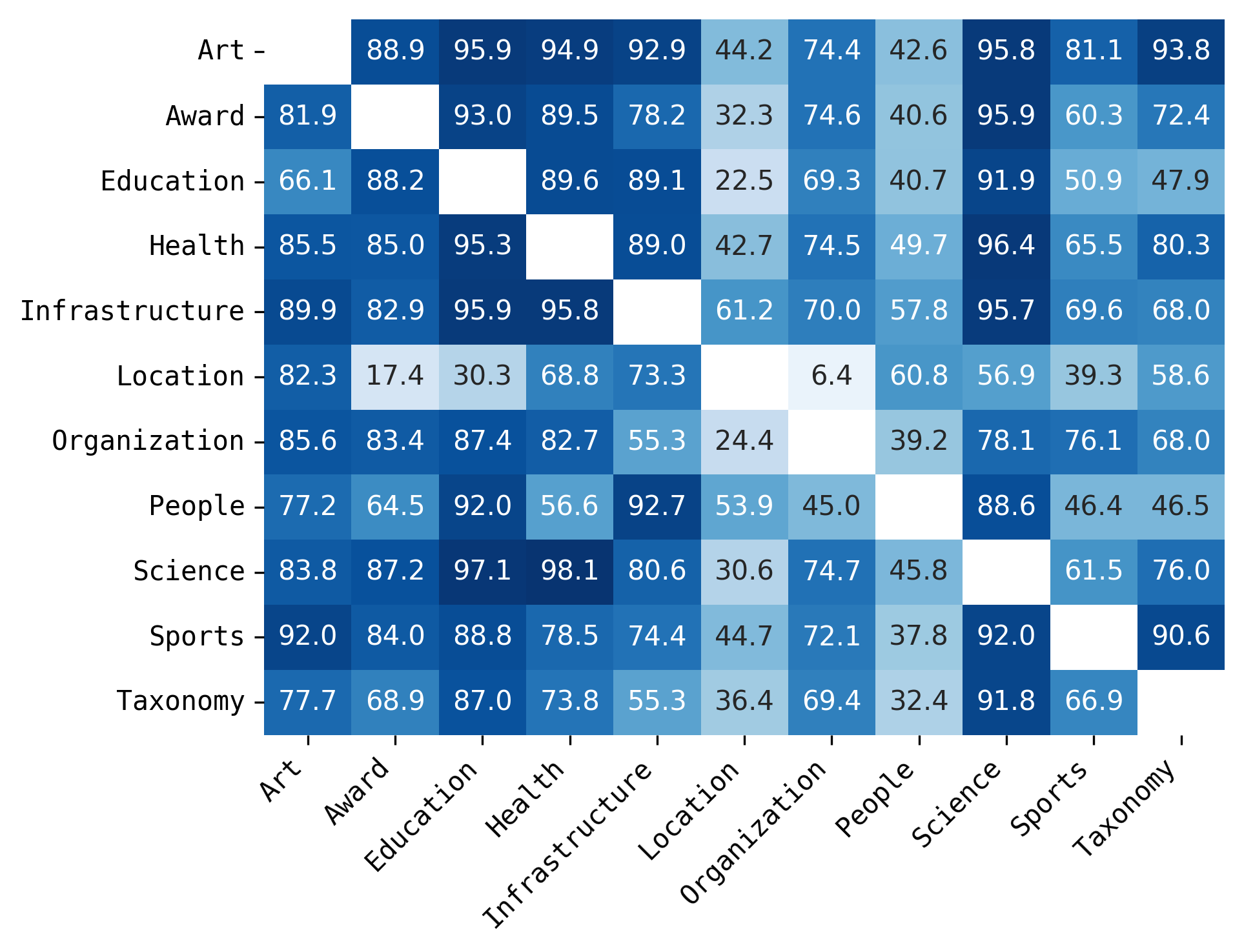}
        %
        \caption{\OurNewModel Hits@10}
    \end{subfigure}
    \hfill
    \begin{subfigure}[b]{0.27\textwidth}
        \centering
        \includegraphics[width=\linewidth,height=\linewidth]{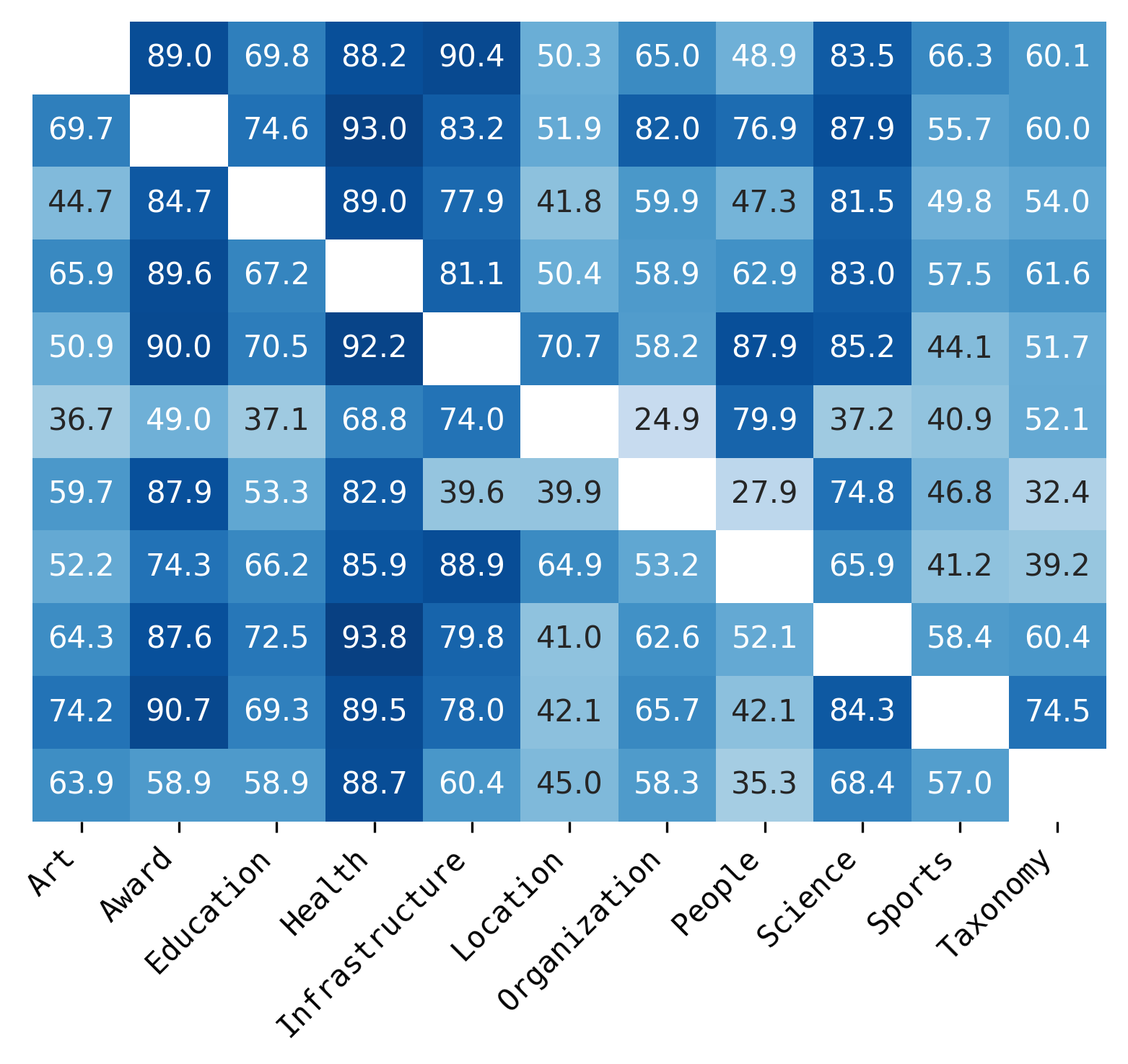}
        \caption{\OurOtherModel Hits@10}
    \end{subfigure}
    \hfill
    \begin{subfigure}[b]{0.31\textwidth}
        \centering
        \includegraphics[width=\linewidth,height=0.88\linewidth]{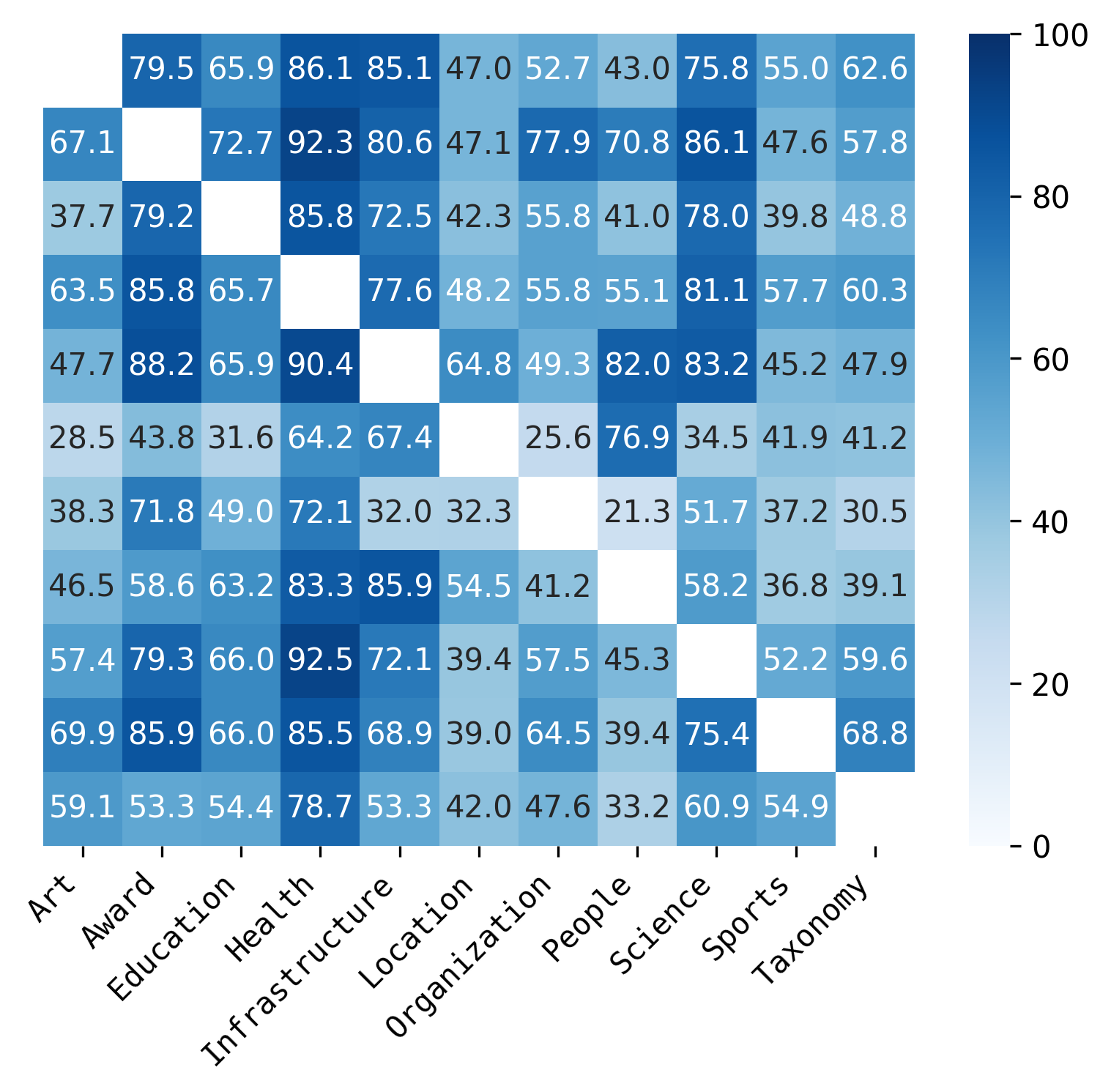}
        \caption{Original InGram Hits@10}
    \end{subfigure}
    \begin{subfigure}[b]{0.32\textwidth}
        \centering        
        \includegraphics[width=\linewidth,height=0.84\linewidth]{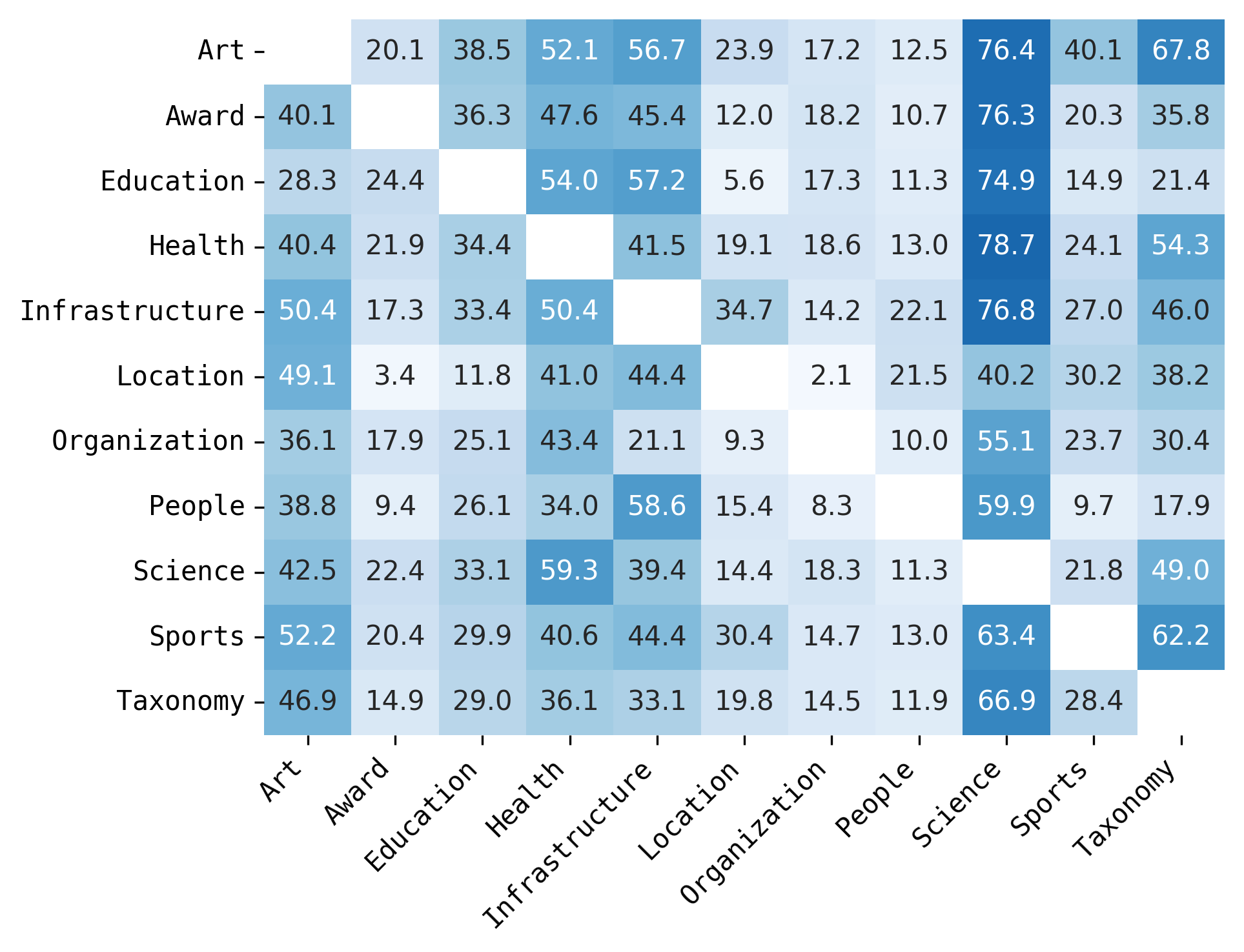}
        %
        \caption{\OurNewModel Hits@1}
    \end{subfigure}
    \hfill
    \begin{subfigure}[b]{0.27\textwidth}
        \centering
        \includegraphics[width=\linewidth,height=\linewidth]{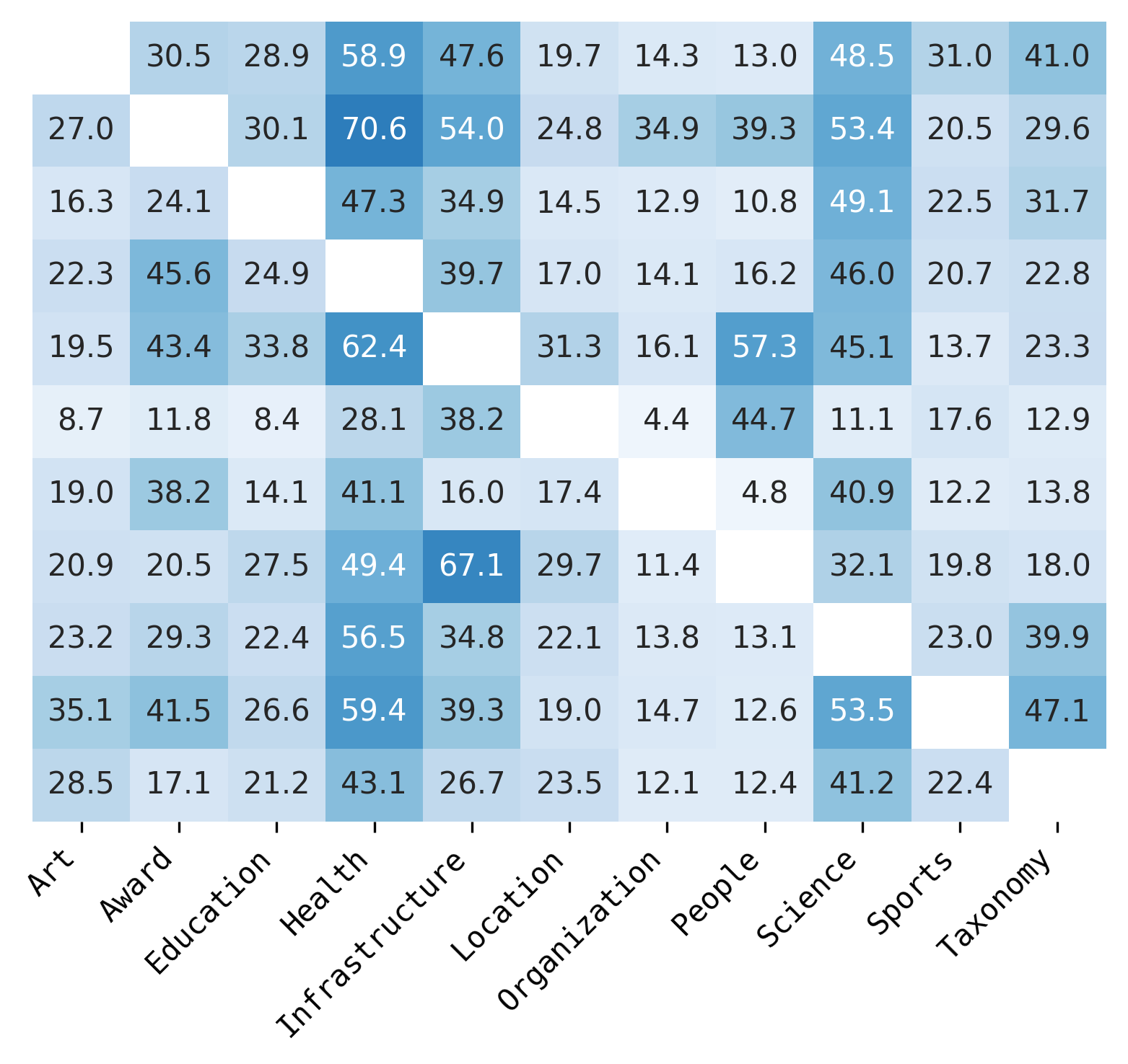}
        \caption{\OurOtherModel Hits@1}
    \end{subfigure}
    \hfill
    \begin{subfigure}[b]{0.31\textwidth}
        \centering
        \includegraphics[width=\linewidth,height=0.88\linewidth]{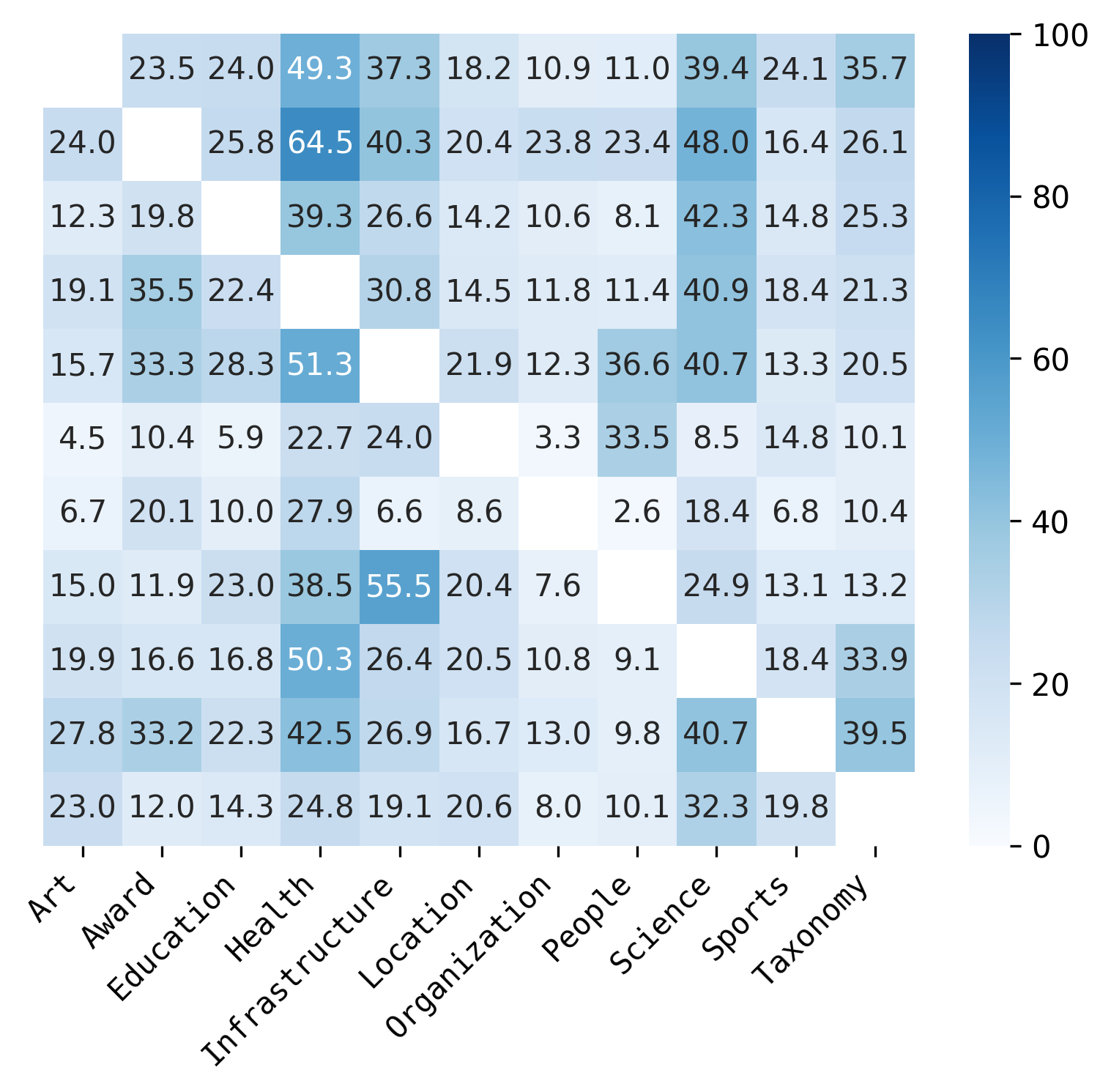}
        \caption{Original InGram Hits@1}
    \end{subfigure}
    \begin{subfigure}[b]{0.32\textwidth}
        \centering        
        \includegraphics[width=\linewidth,height=0.84\linewidth]{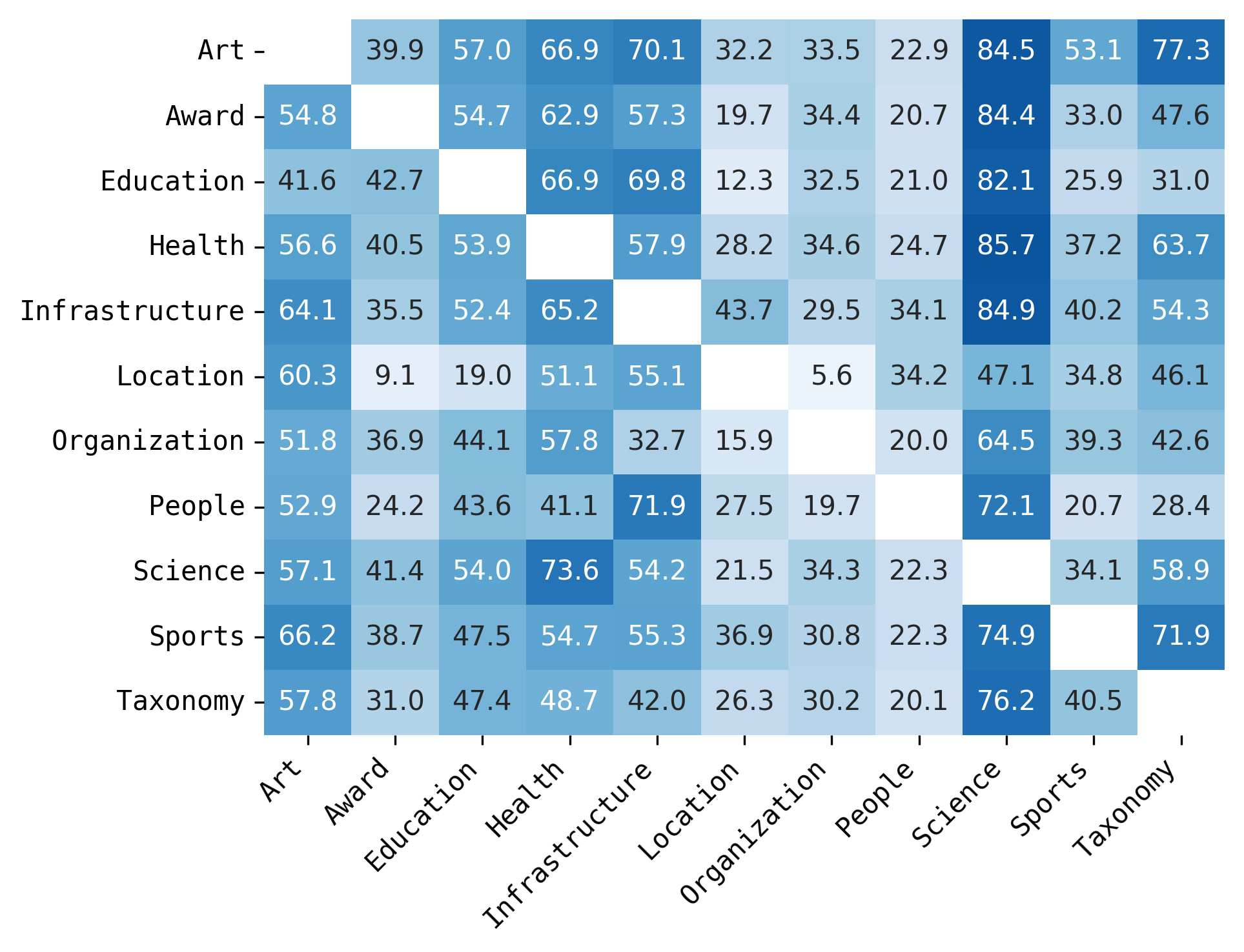}
        %
        \caption{\OurNewModel MRR}
    \end{subfigure}
    \hfill
    \begin{subfigure}[b]{0.27\textwidth}
        \centering
        \includegraphics[width=\linewidth,height=\linewidth]{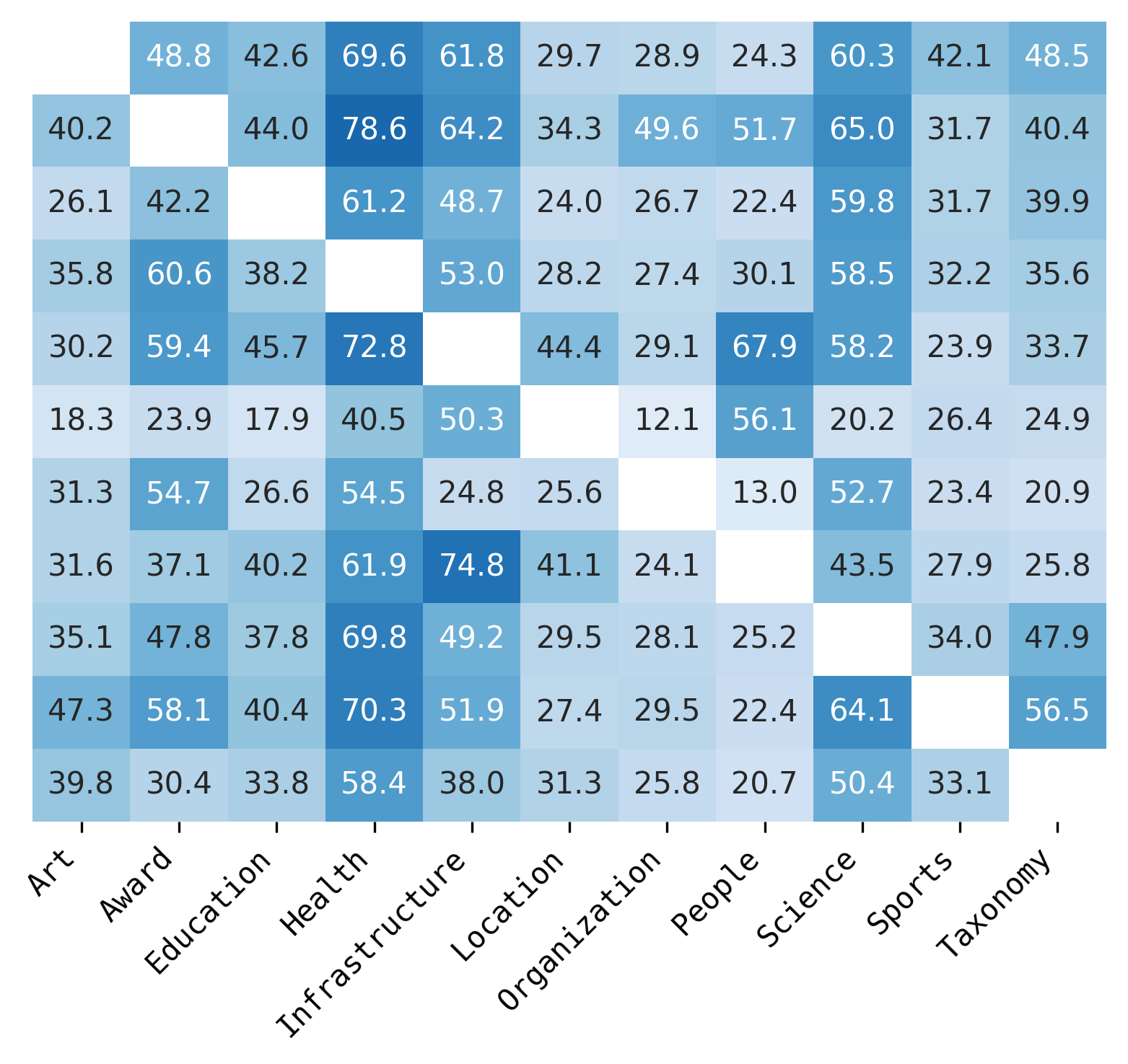}
        \caption{\OurOtherModel MRR}
    \end{subfigure}
    \hfill
    \begin{subfigure}[b]{0.31\textwidth}
        \centering
        \includegraphics[width=\linewidth,height=0.88\linewidth]{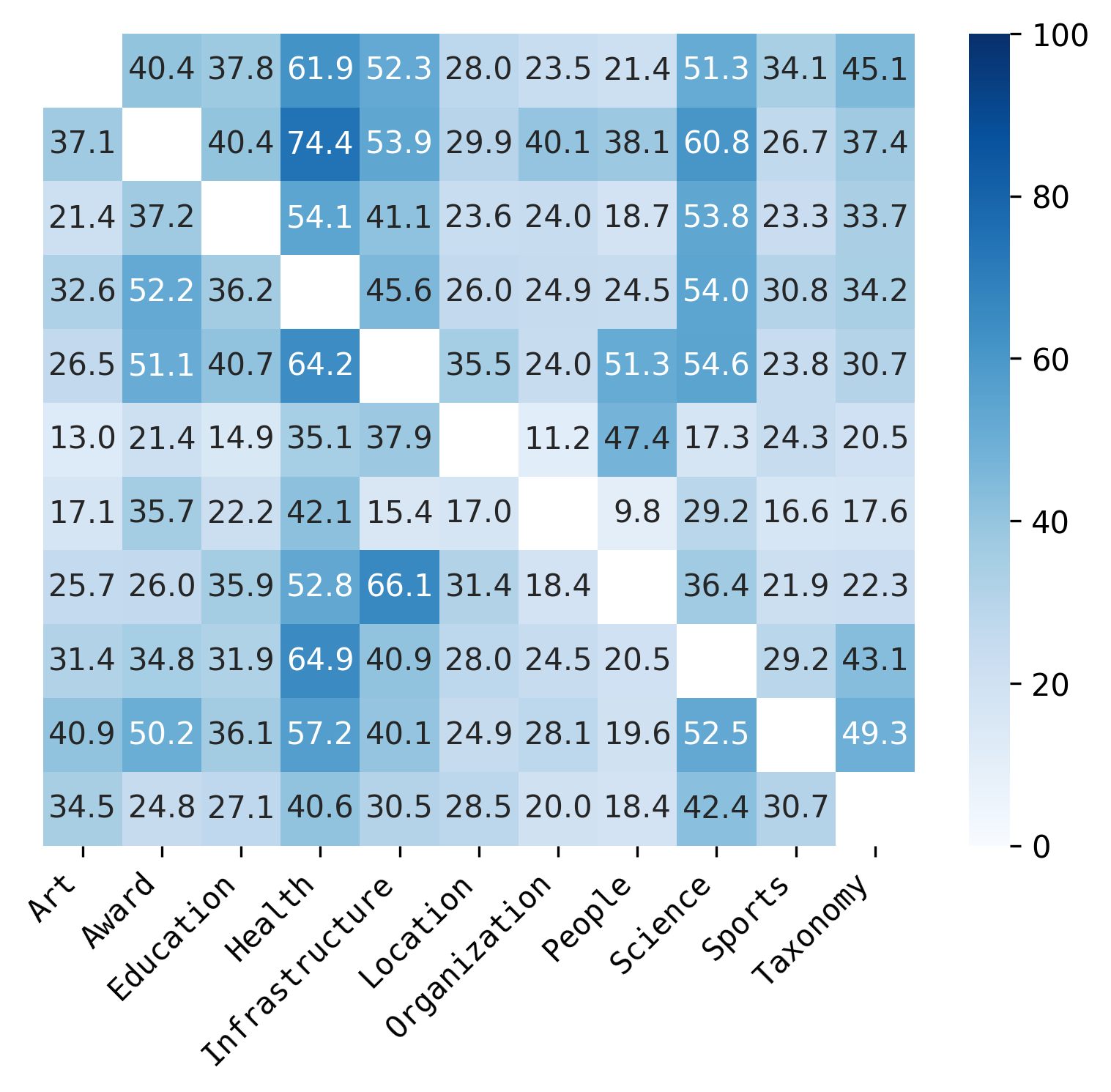}
        \caption{Original InGram MRR}
    \end{subfigure}
    %
    \caption{{\bf Node prediction $(i, k, ?)$ performance over WikiTopics} for \OurNewModel, \OurOtherModel, and InGram~\citep{ingram}. Each row within each heatmap corresponds to a training graph, and each column within each heatmap corresponds to a test graph. A darker color means better performance. 
    \textbf{\OurNewModel, \OurOtherModel, and InGram showcase comparable performance in general, and \OurNewModel exhibits the best performance on Hits@1 in particular.}}
    \label{fig:wikitopics-node-full}
\end{figure}

\end{document}